%% file: main.tex
\documentclass[11pt,a4paper]{article}

\usepackage[margin=1in]{geometry}

\usepackage[numbers,sort&compress]{natbib}

\usepackage[T1]{fontenc}
\usepackage[utf8]{inputenc}
\usepackage{lmodern,comment}
\usepackage{amsmath,amssymb,amsfonts,mathtools,bm,mathrsfs}
\usepackage{graphicx,booktabs,array}
\usepackage{microtype}
\usepackage{xcolor}
\usepackage{amsthm}
\newtheorem{theorem}{Theorem}
\newtheorem{lemma}{Lemma}
\newtheorem{proposition}{Proposition}

\theoremstyle{definition}
\newtheorem{assumption}{Assumption}

\theoremstyle{remark}

\usepackage{xurl}
\usepackage[hidelinks]{hyperref}
\usepackage[nameinlink,noabbrev]{cleveref}
\crefname{assumption}{Assumption}{Assumptions}
\Crefname{assumption}{Assumption}{Assumptions}
\input{macros}

\allowdisplaybreaks
\hypersetup{pdftitle={Stationary Bias and Extrapolation in Nonlinear Two-Timescale Stochastic Approximation},pdfauthor={A.Ch. Madhusudanarao and Rahul Singh}}

\begin{document}

\title{Stationary Bias and Extrapolation in Nonlinear Two-Timescale Stochastic Approximation}
\author{A.Ch. Madhusudanarao\\
\small Department of Computer Science and Automation\\
\small Indian Institute of Science, Bengaluru, India\\
\small \texttt{madhusudanar@iisc.ac.in}
\and
Rahul Singh\\
\small Laboratoire de Recherche de l'EPITA, Paris, France\\
\small \texttt{rahulsingh0188@gmail.com}}
\date{}
\maketitle

\begin{abstract}
Constant-step stochastic approximation generally has a nonzero stationary mean error that persists under time averaging. This paper studies that error for nonlinear two-timescale recursions driven by an exogenous finite-state Markov chain. Under stated smoothness assumptions and conditions on the stationary distribution, we derive a first-order bias expansion whose error bound remains uniform as the slow step size becomes much smaller than the fast step size. Fast-manifold coordinates keep the associated covariance equation regular in this limit. For fast step $\eta$ and slow step $\varepsilon$, the expansion reveals a mixed contribution $\varepsilon^2/\eta$ alongside terms linear in each step size. This dependence matters for bias reduction: along power-law step-size paths, the bias exponents need not be integers, so Richardson--Romberg extrapolation requires weights matched to the path. 
An exactly solvable nonlinear Markov example verifies the coefficients.
We verify localization for temporal-difference learning and compare
finite-run extrapolation at equal update budgets.
For finite runs, we bound the initialization error of tail averages on both timescales under an additional coupling assumption. In the special case of additive independent noise, signed third-moment cancellation yields a sharper remainder.
\end{abstract}

\section{Introduction}\label{sec:intro}
Stochastic approximation (SA) studies iterative algorithms that update their variables using noisy observations. Constant-step SA generally approaches an invariant distribution whose mean differs from the equilibrium of the averaged dynamics. This difference is the \emph{stationary bias}. Averaging the iterates $z_k$ to give $\bar z_N=N^{-1}\sum_{k=1}^N z_k$
reduces fluctuations but retains this stationary bias since
under stationarity $\mathbb E\bar z_N=\mathbb E z_1$.

In two-timescale SA, two coupled sets of variables are updated with different step sizes. Write the fast step as $\eta$, the slow step as $\varepsilon$, and their ratio as $\rho=\varepsilon/\eta$. Our goal is to determine how these two rates jointly affect the stationary bias. The connection to covariance comes from nonlinearity: a Taylor expansion of the averaged update around equilibrium contains quadratic terms whose expectations involve products of coordinate fluctuations. Consequently, the leading curvature contribution to the mean bias depends on the stationary covariance. Correlations between the current iterate and the observation can produce an additional bias contribution, which our analysis also accounts for.

To determine this covariance, we expand the stationary balance of second moments near equilibrium. At leading order, contraction under the linearized dynamics balances the fluctuations supplied by the noise. This balance takes the form of a \emph{Lyapunov equation}, a linear matrix equation for the leading covariance coefficient. For temporally correlated observations, its noise covariance includes correlations across time. Solving and controlling this equation is therefore an intermediate step in computing the nonlinear mean bias.

The difficulty is obtaining covariance estimates that remain useful as $\rho\downarrow0$. On the \emph{fast clock}, where $k$ iterations correspond to time $t=k\eta$, the slow drift is scaled by $\rho$. A scalar slow mode $\dot v=-a\rho v$, $a>0$, for example, decays as $e^{-a\rho t}$. Thus the slow \emph{spectral gap}, the smallest exponential decay rate among the slow linearized modes, tends to zero with $\rho$. Bounds for the Lyapunov equation accumulate perturbations over this decay time and can introduce a factor $1/\rho$, as illustrated by $\int_0^\infty e^{-2a\rho t}\,dt=(2a\rho)^{-1}$. However, the noise and remainder terms in the slow covariance equation also carry small powers of $\rho$. Retaining these factors prevents an artificial loss of uniformity and allows the covariance estimate to support a uniform bias expansion.

We study nonlinear recursions driven by an exogenous finite-state Markov chain: the next input depends on the input history only through its current state, and its transition law is unaffected by the algorithm. We obtain a first-order stationary bias expansion \emph{uniform} in small positive $\rho$, meaning that the remainder constant does not diverge as $\rho\downarrow0$. Coordinates relative to the \emph{fast equilibrium manifold}, the local family of equilibria of the fast averaged dynamics as the slow variable varies, regularize the covariance equation by exposing the scales of its fast and slow blocks. Writing $m_{\eta,\rho}$ for the stationary mean displacement from equilibrium, we obtain
\[
 m_{\eta,\rho}=\eta b_0+\varepsilon b_1+
 \frac{\varepsilon^2}{\eta}b_2+
 O\!\left(\eta^{3/2}+\frac{\varepsilon^3}{\eta^2}\right).
\]
Here $b_0,b_1,b_2$ are coefficient vectors determined by the update rule and noise law, independent of the step sizes.

The mixed term matters for \emph{Richardson--Romberg extrapolation}, which forms a weighted combination of estimates at different step sizes to cancel leading bias terms. Along the power-law path $\varepsilon=\eta^p$, $p>1$, its exponent is $2p-1$, alongside $1$ and $p$. We derive weights adapted to these powers and verify nonzero mixed bias in an exactly solvable nonlinear Markov model. We also bound the initialization error of \emph{tail averages}, formed after discarding an initial segment of the trajectory, under an additional coupling assumption. This assumption controls the difference between a trajectory from a prescribed initial state and a stationary trajectory constructed on the same probability space. Stationary existence and \emph{localization}, expressed through moment bounds around the equilibrium, are explicit hypotheses; local stability alone does not establish them.

We study fixed-policy value estimation using temporal-difference
learning with gradient correction (TDC) \citep{sutton2009fast}.
We verify stationary localization through uniform block contraction
and compare finite-run bias, variance, and mean-square error for
ordinary averaging and extrapolation at equal update budgets.

\paragraph{Related work.}
The ordinary differential equation (ODE) framework for SA is developed in \cite{kushner2003stochastic,borkar2008dynamical}. Decreasing-step two-timescale theory includes \cite{borkar1997twotime,mokkadem2006convergence}; finite-time results treat linear recursions \cite{dalal2018finite}, including Markov inputs \cite{kaledin2020finite} and nonlinear systems \cite{doan2021nonlinear}, with Markov central limit theory and applications to TDC and related gradient-TD methods in \cite{hu2024clt}. Invariant-distribution and weak-error analyses, which study errors in expectations of observables, appear for constant-step stochastic gradient methods in \cite{dieuleveut2020bridging,feng2020uniform}. Bias expansions and Richardson--Romberg correction for Markov SA are studied by \citet{huo2023bias}, while \citet{allmeier2024computing} compute first-order bias using Lyapunov equations.
For nonlinear SA, \citet{huo2024collusion} distinguish bias
contributions from memory, nonlinearity, and their interaction.
The closest two-timescale work, \citet{kwon2025constant},
studies stationary laws, bias, variance, averaging, and
extrapolation for linear recursions. Our work addresses
nonlinear curvature and derives a covariance expansion
uniform as the step-size ratio vanishes.
Extrapolation methods are also developed in \cite{frikha2015multistep,sheshukova2024rr,mangold2025fedrr}.

\paragraph{Notation.}
Norms are Euclidean or induced matrix norms; $\top$ denotes transpose. A matrix is Hurwitz if all its eigenvalues have negative real parts. Constants in uniform bounds are independent of $0<\eta\le\eta_0$ and $0<\rho\le\rho_0$, with sufficiently small fixed upper bounds. Unless indicated otherwise, expectations are stationary.

\section{Model and Assumptions}\label{sec:model}
\subsection{Recursion and Coordinates}\label{sec:model-recursion}
Consider
\begin{align}
 x_{k+1}&=x_k+\eta H(Y_k,x_k,\theta_k),\label{eq:fast}\\
 \theta_{k+1}&=\theta_k+\varepsilon F(Y_k,x_k,\theta_k),\label{eq:slow}
\end{align}
where $x_k\in\R^{d_x}$, $\theta_k\in\R^{d_\theta}$, $\eta>0$, and
$\varepsilon=\eta\rho$ with $0<\rho<1$. Thus $x_k$ uses the larger
step size and $\theta_k$ the smaller one. The uniform results below
concern sufficiently small positive $\eta$ and $\rho$.

The exogenous input $(Y_k)$ is a Markov chain on a finite state space
$\mathsf Y$, with transition matrix $P$, i.e., 
$\Pp(Y_{k+1}=y'\mid\mathcal F_k)=P(Y_k,y')$, where $\mathcal F_k$ is the sigma-field generated by the inputs
and iterates through time $k$.
Its invariant distribution is $\mu$.

For a function $a$ on $\mathsf Y$, write
\[
 Pa(y)=\sum_{y'}P(y,y')a(y'),\qquad
 \mu a=\sum_y\mu(y)a(y).
\]

Set $z=(x^\top,\theta^\top)^\top$ and
$G(y,z)=(H(y,x,\theta)^\top,F(y,x,\theta)^\top)^\top$.
The averaged update field is
\[
 \bar G(z)=\mu G(\cdot,z)
 =\binom{\bar H(x,\theta)}{\bar F(x,\theta)},
\]
where the average is taken over the input while holding $z$ fixed.
The joint recursion is therefore
\begin{equation}
 \begin{gathered}
 D_\rho=\diag(I_{d_x},\rho I_{d_\theta}),\\
 z_{k+1}=z_k+\eta D_\rho G(Y_k,z_k).
 \end{gathered}\label{eq:rec}
\end{equation}

Fix an equilibrium $z^\star=((x^\star)^\top,(\theta^\star)^\top)^\top$
of the averaged field, so $\bar G(z^\star)=0$.
For each $\theta$ near $\theta^\star$, let $\lambda(\theta)$ denote the
corresponding local equilibrium of the fast averaged dynamics:
\[
 \bar H(\lambda(\theta),\theta)=0,\qquad
 \lambda(\theta^\star)=x^\star.
\]
The local uniqueness and $C^2$ regularity of this branch are specified
in Assumption~\ref{as:smooth}. The reduced slow drift is
\[
 g(\theta)=\bar F(\lambda(\theta),\theta),\qquad
 g(\theta^\star)=0.
\]
It describes the slow update when the fast variable is at the
equilibrium associated with the current slow variable.

Define the Jacobian blocks, all evaluated at $z^\star$, by
\[
 \begin{aligned}
 A&=\partial_x\bar H(z^\star),&
 B&=\partial_\theta\bar H(z^\star),\\
 C&=\partial_x\bar F(z^\star),&
 D&=\partial_\theta\bar F(z^\star).
 \end{aligned}
\]
The local stability assumption (Assumption~\ref{as:hurwitz}) makes $A$
invertible. Differentiating the fast equilibrium identity $\bar H(\lambda(\theta),\theta)=0$ gives
$A D\lambda(\theta^\star)+B=0$, and the chain rule then gives
\begin{equation}
 \begin{gathered}
 J=D\bar G(z^\star)=\begin{pmatrix}A&B\\C&D\end{pmatrix},\\
 \Lambda=D\lambda(\theta^\star)=-A^{-1}B,\\
 S=Dg(\theta^\star)=D+C\Lambda=D-CA^{-1}B.
 \end{gathered}\label{eq:J}
\end{equation}
Thus $\Lambda$ describes the first-order movement of the fast
equilibrium as $\theta$ changes, while $S$ is the Jacobian of the
reduced slow drift.

To separate fast tracking from slow displacement, define locally
\begin{equation}
 \widehat u=x-\lambda(\theta),\qquad
 v=\theta-\theta^\star.\label{eq:coords}
\end{equation}
Here $\widehat u$ is the \emph{exact tracking error}: it measures how
far the fast variable is from the equilibrium associated with the
current value of $\theta$. Replacing that equilibrium by its tangent
approximation $x^\star+\Lambda v$ gives the \emph{tangent tracking
error} $u=x-x^\star-\Lambda v$. Equivalently,
\begin{equation}
 \binom{u}{v}=T(z-z^\star),\qquad
 T=\begin{pmatrix}I_{d_x}&-\Lambda\\0&I_{d_\theta}\end{pmatrix}.
 \label{eq:tangent-coordinates}
\end{equation}
The two tracking errors satisfy the exact identity
\[
 \begin{aligned}
 \widehat u&=u-r_\lambda(v),\\
 r_\lambda(v)&=\lambda(\theta^\star+v)-x^\star-\Lambda v.
 \end{aligned}
\]
Since $\lambda$ is $C^2$ near $\theta^\star$, its second derivative is
bounded on a sufficiently small closed ball around $\theta^\star$.
Taylor's theorem therefore gives, for some $r>0$ and $M<\infty$,
\[
 \|r_\lambda(v)\|\le \frac{M}{2}\|v\|^2
 \qquad (\|v\|\le r),
\]
where $M$ bounds $\|D^2\lambda\|$ on that ball, so that 
the remainder is locally quadratic.

The tangent coordinates also simplify the coupling in the equilibrium
linearization. Using $A\Lambda+B=0$, the linearized averaged recursion
becomes
\[
 \begin{aligned}
 u_{k+1}-u_k
 &=\eta\bigl[(A-\rho\Lambda C)u_k-\rho\Lambda S v_k\bigr],\\
 v_{k+1}-v_k
 &=\eta\rho(Cu_k+Sv_k).
 \end{aligned}
\]
The identity $A\Lambda+B=0$ cancels the order-one effect of slow
displacement on the fast linearized drift. After factoring out the
common step size $\eta$, the remaining cross coefficients are
$-\rho\Lambda S$ and $\rho C$, both of order $\rho$.

\subsection{Standing Assumptions}\label{sec:assumptions}
Write $e_k=z_k-z^\star$ for the error process. For a stationary draw
$(Y,z)\sim\pi_{\eta,\rho}$ from a joint invariant law, define
\begin{equation}
 e=z-z^\star,\quad m_{\eta,\rho}=\E e,\quad
 \Gamma_{\eta,\rho}=\Cov(e).\label{eq:targets}
\end{equation}
Thus $e$ is a stationary random vector, while $m_{\eta,\rho}$ and
$\Gamma_{\eta,\rho}$ are its deterministic mean and covariance.
All moments below are taken under $\pi_{\eta,\rho}$.

\begin{assumption}[Local smoothness]
\label{as:smooth}
The averaged field $\bar G$ is $C^3$ near $z^\star$.
Near $\theta^\star$, the fast equilibrium $\lambda(\theta)$
is locally unique and $C^2$, with
$\lambda(\theta^\star)=x^\star$.
\end{assumption}

\begin{assumption}[Local stability]\label{as:hurwitz}
The matrices $A$ and $S$ in \eqref{eq:J} are Hurwitz.
\end{assumption}

\begin{assumption}[Mixing Markov input]\label{as:markov}
For some $n_0\ge1$, $\beta>0$, and probability measure $\nu$ on
$\mathsf Y$,
$P^{n_0}(y,E)\ge\beta\nu(E)$ for every $y\in\mathsf Y$ and
every subset $E\subseteq\mathsf Y$.
\end{assumption}

\begin{assumption}[Stationary fourth-moment localization]
\label{as:localization}
There exist $\eta_0,\rho_0,C>0$ such that, for every
$0<\eta\le\eta_0$ and $0<\rho\le\rho_0$, the joint process
$(Y_k,z_k)$ admits an invariant law $\pi_{\eta,\rho}$ satisfying
\begin{equation}
 \mathbb E_{\pi_{\eta,\rho}}\|z-z^\star\|^4\le C\eta^2.
 \label{eq:cov-raw-fourth}
\end{equation}
The constant $C$ is independent of $\eta$ and $\rho$.
\end{assumption}
All stationary results refer to this selected family of invariant laws.
Stationary existence and this moment bound do not follow from local
Hurwitz stability alone. Proposition~\ref{prop:contractive-localization}
gives sufficient conditions, and
Appendix~\ref{app:example-verification} verifies localization for the
nonlinear Markov example.

\begin{assumption}[Update regularity]\label{as:taylor}
Each $G(y,\cdot)$ is globally $C^2$, with
\begin{equation}
 \sup_{y,z}\bigl(\|D_zG(y,z)\|+\|D_z^2G(y,z)\|\bigr)\le K.
 \label{eq:cov-regularity}
\end{equation}
\end{assumption}
These conditions allow updates growing linearly in $z$.
Together with local $C^3$ smoothness and
\eqref{eq:cov-raw-fourth}, they imply the expected Taylor estimates
in Lemma~\ref{lem:derived-taylor-estimates}.
The stationary covariance and mean-bias expansions require no
separate slow-coordinate moment or cross-covariance assumption.
The stronger block and coordinate conditions in
Assumption~\ref{as:block-localization} are used only when invoking
the supplementary blockwise fluctuation statements.
The stationary iterates need not have bounded support.

\section{Geometry, Noise, and Covariance}\label{sec:geometry}

\subsection{The Two Relaxation Clocks}\label{sec:clocks}
On the fast timescale, where one iteration corresponds to time $\eta$,
the averaged ODE is $\dot z=D_\rho\bar G(z)$. Its linearization at
$z^\star$, expressed in the tangent coordinates $w=T(z-z^\star)$,
is $\dot w=L_\rho w$, where
\begin{equation}
 L_\rho=TD_\rho JT^{-1}
 =\begin{pmatrix}A-\rho\Lambda C&-\rho\Lambda S\\
 \rho C&\rho S\end{pmatrix}.\label{eq:L}
\end{equation}
Thus $e^{tL_\rho}$ propagates an initial linearized perturbation over
time $t$; its subscripts below indicate the output and input coordinate
groups. The following bounds separate the fast and slow decay rates
and control their interaction, providing the linear estimates needed
for uniform covariance analysis.

\begin{proposition}[Two-clock semigroup bounds]\label{prop:semigroup}
Under Assumptions~\ref{as:smooth}--\ref{as:hurwitz}, there exist
$C_*,c,\eta_0,\rho_0>0$ such that for $0<\rho\le\rho_0$ and $t\ge0$,
\begin{equation}
 \begin{aligned}
 \|[e^{tL_\rho}]_{uu}\|&\le C_*(e^{-ct}+\rho e^{-c\rho t}),\\
 \|[e^{tL_\rho}]_{uv}\|+\|[e^{tL_\rho}]_{vu}\|
 &\le C_*\rho(e^{-ct}+e^{-c\rho t}),\\
 \|[e^{tL_\rho}]_{vv}\|&\le C_*e^{-c\rho t}.
 \end{aligned}\label{eq:semigroup}
\end{equation}
For $0<\eta\le\eta_0$ and integers $k\ge0$, the same inequalities hold
with $e^{tL_\rho}$ replaced by $(I+\eta L_\rho)^k$, and with
$e^{-ct}$ and $e^{-c\rho t}$ replaced by $e^{-c\eta k}$ and
$e^{-c\varepsilon k}$, respectively, where $\varepsilon=\eta\rho$.
\end{proposition}
Here $(I+\eta L_\rho)^k$ propagates the linearized averaged recursion
$w_{k+1}=(I+\eta L_\rho)w_k$ through $k$ iterations, corresponding to
fast time $t=\eta k$ and slow time $\rho t=\varepsilon k$.
The proof of both bounds is given in
Appendix~\ref{sec:sup:report-ode-geometry}.

\subsection{Poisson Representation}\label{sec:poisson}
We identify the effective covariance of the Markov noise,
accounting for correlations across time. This supplies the noise
term in the Lyapunov equation~\eqref{eq:lyap}, whose solution
determines the leading-order stationary covariance. The latter
is used below to quantify how nonlinear drift converts fluctuations
into stationary bias. To obtain the noise covariance, we use
the Poisson equation introduced next.

For fixed $z$, let $\widetilde G(y,z)=G(y,z)-\bar G(z)$.
We consider the \emph{Poisson equation} for the Markov
transition operator $P$, with a zero-mean normalization:
\begin{equation}
 \mathcal U-P\mathcal U=\widetilde G,\qquad
 \mu\mathcal U=0.
 \label{eq:poisson}
\end{equation}
Under Assumption~\ref{as:markov}, geometric mixing of the
Markov input $Y_k$ gives the solution
$\mathcal U(\cdot,z)=\sum_{j\ge0}P^j\widetilde G(\cdot,z)$
\citep[Section~5]{glynnInfanger2024}.

Set $\xi_k=G(Y_k,z^\star)$ and
$\zeta_{k+1}(z)=\mathcal U(Y_{k+1},z)-(P\mathcal U)(Y_k,z)$.
\begin{lemma}[Long-run covariance]\label{lem:Q}
Under Assumption~\ref{as:markov}, $\zeta_{k+1}(z)$ is a martingale
difference for fixed $z$, meaning
$\E[\zeta_{k+1}(z)\mid\mathcal F_k]=0$.
For the stationary input chain $Y_0\sim\mu$, write
$K_h=\E[\xi_0\xi_h^\top]$. Then
\begin{equation}
 \begin{aligned}
 Q&=\lim_{m\to\infty}m^{-1}\E\left[
 \Bigl(\sum_{k=0}^{m-1}\xi_k\Bigr)
 \Bigl(\sum_{k=0}^{m-1}\xi_k\Bigr)^\top\right]\\
 &=K_0+\sum_{h\ge1}(K_h+K_h^\top)
 =\E[\zeta_1(z^\star)\zeta_1(z^\star)^\top].
 \end{aligned}\label{eq:Q}
\end{equation}
\end{lemma}
The sum in \eqref{eq:Q} accumulates the equilibrium noise over
$m$ steps. To compute its covariance per step, we use the
Poisson equation to write
\[
\begin{aligned}
 \sum_{k=0}^{m-1}\xi_k
 &=\sum_{k=0}^{m-1}\zeta_{k+1}(z^\star)+\mathcal U(Y_0,z^\star)-\mathcal U(Y_m,z^\star).
\end{aligned}
\]
The first term is a martingale sum. The remaining difference
depends only on the initial and final noise states and is bounded
uniformly in $m$, since $\mathsf Y$ is finite.
Its covariance and cross terms vanish after division by $m$,
yielding the last equality in \eqref{eq:Q}; see
Appendix~\ref{subsec:sup:poisson}.

\subsection{Covariance of Stationary Iterates}\label{sec:cov-coeff}
We now approximate $\Gamma_{\eta,\rho}=\Cov(z-z^\star)$ under
the stationary law $\pi_{\eta,\rho}$. This covariance measures
fluctuations around the stationary mean and enters the bias
expansion in Section~\ref{sec:bias}.

Write $J_\rho=D_\rho J$, which is Hurwitz for sufficiently small
$\rho$ by Proposition~\ref{prop:semigroup}. Define $\Sigma_\rho$
as the unique solution of the Lyapunov equation
\begin{equation}
 J_\rho\Sigma_\rho+\Sigma_\rho J_\rho^\top+D_\rho QD_\rho=0.
 \label{eq:lyap}
\end{equation}
Appendix Lemma~\ref{lem:sup:linear-covariance-limit} derives this
equation from the linearized recursion.
The next proposition shows that $\eta\Sigma_\rho$ approximates
$\Gamma_{\eta,\rho}$, the covariance of the joint iterates
$z_k=(x_k^\top,\theta_k^\top)^\top$ generated by
\eqref{eq:rec} under the stationary law $\pi_{\eta,\rho}$.

\begin{proposition}[Uniform covariance approximation]
\label{prop:covariance}
Under Assumptions~\ref{as:smooth}--\ref{as:taylor}, for all
sufficiently small positive $\eta,\rho$,
\begin{equation}
 \Gamma_{\eta,\rho}=\eta\Sigma_\rho+R^\Gamma_{\eta,\rho},
 \qquad
 \|R^\Gamma_{\eta,\rho}\|\le C\eta^{3/2}.
 \label{eq:cov-exp}
\end{equation}
The constant $C$ is independent of $\eta$ and $\rho$, so the
bound remains valid as $\rho\downarrow0$.
\end{proposition}

\begin{proof}[Proof idea]
The complete proof of Proposition~\ref{prop:covariance} is given
in Appendix~\ref{app:covariance-proof}.
Apply the Poisson correction~\eqref{eq:sup:report-corrected-error}
and then the tangent transformation~\eqref{eq:tangent-coordinates};
the leading noise then has zero conditional mean.
Stationarity then gives
a covariance balance, with remainder terms controlled by the
moment and regularity assumptions. Crucially, the remainder in
the slow covariance equation carries a factor $\rho$, offsetting
its $O(1/\rho)$ relaxation time. Solving this balance and removing
the correction gives the stated bound; see
Appendix~\ref{app:covariance-proof}.
\end{proof}

\section{Stationary Mean Bias}\label{sec:bias}
We derive the stationary mean bias from the balance
$\E G(Y_k,z_k)=0$.
Theorem~\ref{thm:fixed-ratio} combines the covariance approximation
of Proposition~\ref{prop:covariance}, a second-order Taylor expansion
of $\bar G(z)$ about $z^\star$, and a Poisson-corrector identity to obtain
$m_{\eta,\rho}=\eta b(\rho)+O(\eta^{3/2})$.
The coefficient separates curvature and direct Markov contributions.
Lemma~\ref{lem:mean-O-eta} supplies the auxiliary bound
$\|m_{\eta,\rho}\|=O(\eta)$ used in the theorem's proof.
Both bounds are uniform as $\rho\downarrow0$.

Define the curvature map and direct Markov response by
\[
 \mathcal C(M)_i=\tfrac12\Tr(D^2\bar G_i(z^\star)M),
\]
\begin{equation}
 r_\rho=\sum_{y,y'}\mu(y)P(y,y')
 D_z\mathcal U(y',z^\star)D_\rho G(y,z^\star).
 \label{eq:report-markov-response}
\end{equation}
\begin{theorem}[Fixed-ratio bias]\label{thm:fixed-ratio}
If Assumptions~\ref{as:smooth}--\ref{as:taylor} hold,
\begin{equation}
 \begin{gathered}
 m_{\eta,\rho}=\eta b(\rho)+O(\eta^{3/2}),\\
 b(\rho)=-J^{-1}\{\mathcal C(\Sigma_\rho)+r_\rho\},
 \end{gathered}\label{eq:report-fixed-ratio-bias}
\end{equation}
with a remainder uniform for $0<\rho\le\rho_0$.
\end{theorem}
\begin{lemma}[A priori mean bound]\label{lem:mean-O-eta}
Under the same assumptions, $\|m_{\eta,\rho}\|\le C\eta$ uniformly.
\end{lemma}
The proof is given in Appendix~\ref{subsec:fixed-ratio-proof};
Appendix~\ref{sec:sup:bias-expansion} supplies the detailed
identities and remainder bounds.

\section{Bias Expansion as the Step-Size Ratio Vanishes.}\label{sec:singular}
Section~\ref{sec:bias} leaves the ratio dependence implicit in
$b(\rho)$. As $\rho\downarrow0$, the slow relaxation rate vanishes
on the fast clock, making the covariance equation singular at
$\rho=0$. We rescale the covariance blocks to obtain regular limiting
equations and expand $b(\rho)$.
Theorem~\ref{thm:singular} identifies contributions proportional to
$\eta$, $\varepsilon$, and $\varepsilon^2/\eta$, with an
$O(\eta^{3/2}+\varepsilon^3/\eta^2)$ remainder whose implied constant
is independent of $\eta$ and $\rho$.
This expansion also guides bias reduction using runs at different
step sizes. For example, at fixed ratio, halving both steps halves
the leading bias, so twice the finer-step average minus the
coarser-step average cancels that leading stationary bias.
Section~\ref{sec:rr} adapts this cancellation to runs in which
the step-size ratio also changes.

\begin{theorem}[Singular expansion]\label{thm:singular}
Under Assumptions~\ref{as:smooth}--\ref{as:taylor}, uniformly for small positive $\eta,\rho$,
\begin{equation}
 \begin{aligned}
 m_{\eta,\rho}&=\eta b_0+\varepsilon b_1+\frac{\varepsilon^2}{\eta}b_2+O\!\left(\eta^{3/2}+\frac{\varepsilon^3}{\eta^2}\right).
 \end{aligned}\label{eq:report-mixed-expansion}
\end{equation}
\end{theorem}
\begin{proof}[Proof sketch]
In tangent coordinates, write
\[
 T\Sigma_\rho T^\top=
 \begin{pmatrix}
 U(\rho)&\rho R(\rho)\\
 \rho R(\rho)^\top&\rho V(\rho)
 \end{pmatrix}.
\]
Rescaling the cross and slow covariance equations gives a
linear system with polynomial coefficients in $\rho$.
At $\rho=0$, it is triangular and invertible by Hurwitz
stability of $A$ and $S$. Its solution is therefore analytic
near zero. Expanding $\Sigma_\rho$ through order $\rho^2$,
using the affine dependence of $r_\rho$ on $\rho$, and
substituting into Theorem~\ref{thm:fixed-ratio} gives the result.
The detailed proof and coefficient formulas are in
Appendix~\ref{subsec:singular-proof}.
\end{proof}
In tangent coordinates, the leading fast covariance, fast--slow
cross-covariance, and slow covariance are
$\eta U(\rho)=O(\eta)$,
$\eta\rho R(\rho)=O(\varepsilon)$, and
$\eta\rho V(\rho)=O(\varepsilon)$, respectively.
The absolute $O(\eta^{3/2})$ covariance remainder can obscure the
$O(\varepsilon)$ terms when $\rho\lesssim\sqrt{\eta}$.
The sharper $O(\varepsilon)$ bounds for the actual slow covariance
and fast--slow cross-covariance are additional conditions collected
in Assumption~\ref{as:block-localization}.
They are not needed for Proposition~\ref{prop:covariance} or
Theorems~\ref{thm:fixed-ratio}--\ref{thm:singular}.
Proposition~\ref{prop:contractive-localization} gives sufficient
conditions under which these stronger bounds hold.

\section{Extrapolation Along Step-Size Paths}\label{sec:rr}

Time averaging reduces random fluctuations but does not remove
stationary bias. We reduce this bias by combining averages from
several runs at different pairs of fast and slow step sizes.
For each run, we choose a fast step $\eta$ and set the slow step
to $\varepsilon=\eta^p$, with the same fixed $p>1$ across runs.
Thus, reducing $\eta$ also determines how $\varepsilon$ decreases.
This relation between the two step sizes defines the step-size path.
The expansion in Section~\ref{sec:singular} describes how the bias
changes along this path, allowing us to choose weights so that
the leading stationary bias terms cancel.
Proposition~\ref{prop:rr} gives a three-run combination that cancels
the contributions proportional to the fast and slow steps.
We also describe how to cancel further terms and bound the
remaining bias.

Before choosing extrapolation weights, we identify the bias powers
along the chosen step-size path and compare them with the remainder.
For $\varepsilon=\eta^p$, $p>1$, Theorem~\ref{thm:singular} gives
\begin{equation}
 m(\eta)=\eta b_0+\eta^p b_1+\eta^{2p-1}b_2
 +O(\eta^{3/2}+\eta^{3p-2}).\label{eq:report-path-expansion}
\end{equation}
A nonzero bias term can be distinguished from the remainder when
it decays more slowly as $\eta\downarrow0$.
Since $3p-2>2p-1>p>1$, comparison with the
$O(\eta^{3/2})$ remainder determines which displayed terms satisfy
this condition: $\eta^p b_1$ does so for $1<p<3/2$, and
$\eta^{2p-1}b_2$ for $1<p<5/4$.
Under the additional additive-i.i.d.\ conditions of
Theorem~\ref{thm:iid}, the sharper $O(\eta^2)$ remainder extends
the latter range to $1<p<3/2$.

At fixed $\rho$, let $M_\rho(h)$ denote the stationary mean at
steps $(\eta,\varepsilon)=(h,\rho h)$.
Halving both steps halves the leading bias, so twice the finer-step
mean minus the coarser-step mean cancels this contribution:
\begin{equation}
 2M_\rho(h/2)-M_\rho(h)
 =z^\star+O(h^{3/2}).\label{eq:report-fixed-ratio-RR}
\end{equation}
Along $\varepsilon=\eta^p$, reducing the fast step must preserve
this relation: the corresponding pairs are
$(\eta_j,\varepsilon_j)=(2^{-j}h,(2^{-j}h)^p)$.
The bias terms proportional to $h$ and $h^p$ now decrease by
different factors between successive runs.
Proposition~\ref{prop:rr} gives weights for combining the three
stationary means at $j=0,1,2$ so that both terms cancel, and
bounds the remaining bias.
\begin{proposition}[Three-level path extrapolation]\label{prop:rr}
Let $M(h)$ be the stationary mean at $(h,h^p)$ under Theorem~\ref{thm:singular}. Set $q_1=2^{-1}$, $q_p=2^{-p}$, and
\begin{equation}
 (w_0,w_1,w_2)=
 \frac{(q_1q_p,-q_1-q_p,1)}{(1-q_1)(1-q_p)}.
 \label{eq:report-three-weights}
\end{equation}
Then
\[
 \sum_{j=0}^2w_jM(2^{-j}h)=z^\star+
 O\!\left(h^{\min\{2p-1,3/2\}}\right).
\]
If the stationary weak remainder is $O(h^2)$, replace $3/2$ by $2$.
\end{proposition}
\begin{proof}
The weights satisfy $\sum_jw_j=1$, $\sum_jw_j2^{-j}=0$, and $\sum_jw_j2^{-jp}=0$. Apply these identities to \eqref{eq:report-path-expansion}.
\end{proof}
Appendix~\ref{sec:sup:singular-path-rr} gives weights for canceling
additional bias terms and explains the limits imposed by the
available remainder bounds.
Extrapolation can increase sampling variance, so reducing bias does
not necessarily improve accuracy at fixed computational cost.

An exactly solvable nonlinear Markov example in
Appendix~\ref{sec:sup:exact-markov-example} verifies that the mixed
bias term can be nonzero and illustrates these extrapolation rules
using exact stationary means.

\section{Finite-Time Consequences}\label{sec:finite}
A stationary coupling decomposes the error into initialization,
centered fluctuation, and stationary mean. Under
Assumptions~\ref{as:smooth}--\ref{as:taylor} and the additional
block localization in Assumption~\ref{as:block-localization},
\begin{equation}
 \E\|u\|^2=O(\eta),\qquad
 \E\|v\|^2=O(\varepsilon+\eta^2).\label{eq:floors}
\end{equation}
The $\eta^2$ term is squared bias and can dominate the slow variance on thin paths.
\begin{assumption}[Two-clock coupling]\label{as:coupling}
A coupling with a stationary version satisfies
\begin{equation}
 \E\|z_k-z_k^{\rm st}\|
 \le C\{1+V(z_0)\}(e^{-c\eta k}+e^{-c\varepsilon k}),\label{eq:coupling}
\end{equation}
where $V\ge0$ has uniformly bounded invariant expectation. Any meeting-time contribution from different input initializations is included in this bound.
\end{assumption}
This additional condition does not follow from local linear stability.
\begin{theorem}[Tail-average bias and initialization]\label{thm:tail}
For $\bar z_{K,N}=L^{-1}\sum_{k=K}^{N-1}z_k$, $L=N-K>0$, under Assumption~\ref{as:coupling},
\begin{align}
 \|\mathsf T_{K,N}\|\le C\{1+V(z_0)\}\bigl[
 &e^{-c\eta K}\min\{1,(\eta L)^{-1}\}
+e^{-c\varepsilon K}\min\{1,(\varepsilon L)^{-1}\}\bigr],\label{eq:tail-bound}
\end{align}
where $\mathsf T_{K,N}=\E\bar z_{K,N}-\E z^{\rm st}$. If Assumptions~\ref{as:smooth}--\ref{as:taylor} also hold, then
\begin{equation}
 \E\bar z_{K,N}-z^\star=\eta b(\rho)+O(\eta^{3/2})
 +\mathsf T_{K,N}.\label{eq:tail-bias}
\end{equation}
\end{theorem}
\begin{proof}
Average \eqref{eq:coupling} and use, for $h=\eta,\varepsilon$,
\[
 \begin{aligned}
 L^{-1}\sum_{k=K}^{N-1}e^{-chk}
 &=e^{-chK}\frac{1-e^{-chL}}{L(1-e^{-ch})}\\
 &\le Ce^{-chK}\min\{1,(hL)^{-1}\}.
 \end{aligned}
\]
Then apply Theorem~\ref{thm:fixed-ratio}.
\end{proof}
Stationary initialization makes $\mathsf T_{K,N}=0$. For finite-run extrapolation, weighted initialization errors must lie below the retained remainder.
Appendix~\ref{sec:sup:finite-time} derives conditional blockwise
mean-square bounds under additional stationary moment and
coordinate $L^2$ coupling assumptions.

\section{Application: temporal-difference learning}
\label{sec:experiments}

We study fixed-policy value estimation using temporal-difference
learning with gradient correction (TDC)~\citep{sutton2009fast,mahadevan2026convergence}. We compare stationary-bias
predictions with simulations and evaluate extrapolation at equal
update budgets.

\subsection{Setup and Recursion}

Let $S_k$ be the state at time $k$ under a fixed policy,
$R_{k+1}$ the reward on the transition to $S_{k+1}$, and
$\gamma\in(0,1)$ the discount factor.
For a fixed scalar feature $\phi(s)$, we approximate the discounted
value function by $V_\theta(s)=\phi(s)\theta$, with parameter
$\theta\in\mathbb R$. Write $\phi_k=\phi(S_k)$ and define
\[
 \delta_k=R_{k+1}+\gamma\phi_{k+1}\theta_k-\phi_k\theta_k.
\]
TDC uses a fast auxiliary variable $w_k$ and a slow value
parameter $\theta_k$ \citep{sutton2009fast}:
\begin{align*}
 w_{k+1}&=w_k+\eta(\delta_k-\phi_k w_k)\phi_k,\\
 \theta_{k+1}&=\theta_k+\varepsilon
 \bigl[\delta_k\phi_k-\gamma\phi_{k+1}\phi_k w_k\bigr].
\end{align*}
Both updates use the current parameters and the same observed
transition.

We use a Markov reward process with state space
$\mathsf S=\{0,1,2,3,4\}$, representing positions along a
five-site chain.
From an interior state $s\in\{1,2,3\}$, the process stays at $s$
with probability $1/2$ and moves to each neighboring state
with probability $1/4$.
At either boundary, an attempted outward move becomes a
self-loop: the process stays with probability $3/4$ and moves
to its only neighbor with probability $1/4$.
The reward is
$R_{k+1}=\mathbf 1\{S_{k+1}=4\}$, so a transition earns reward
one exactly when its destination is the rightmost state.
The process continues after receiving this reward; state $4$
is not terminal.
We set $\phi(s)=1+0.05s$ and $\gamma=0.2$.
Our estimation target is the TD equilibrium
$\theta^\star\approx0.2443164284$, derived in
Appendix~\ref{subsec:tdc-application}.
The updates are affine, so their curvature contribution vanishes;
this example tests the direct Markov contribution to stationary bias.

\begin{figure*}[t]
 \centering
\includegraphics[
  width=.90\textwidth,
  trim=0bp 24bp 0bp 0bp,
  clip
]{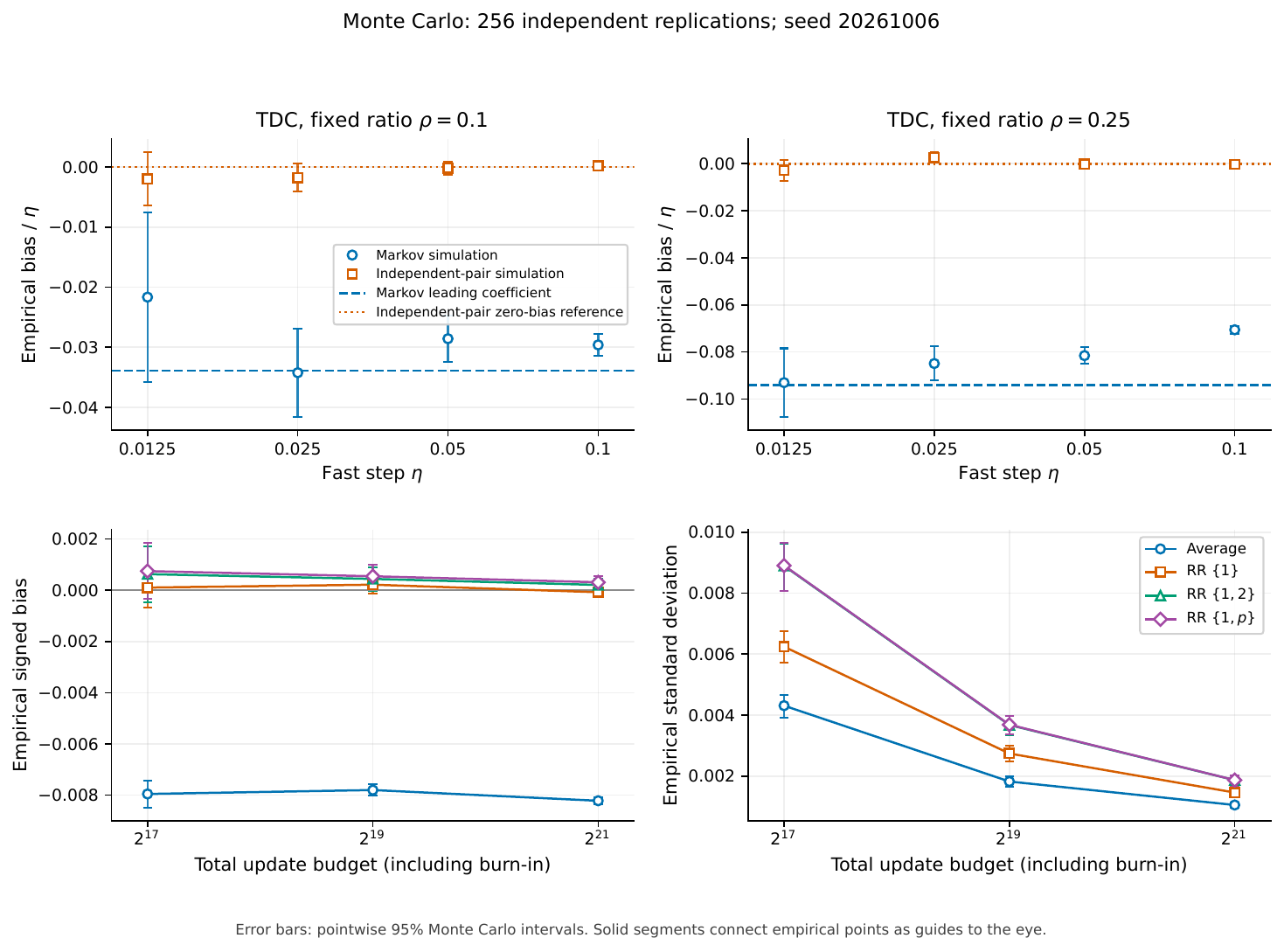}
 \caption{Empirical bias and variability of TDC estimates of
 $\theta^\star$.
 Top: empirical bias divided by $\eta$ at fixed ratios
 $\rho=0.1$ and $\rho=0.25$, for Markov sampling and independent
 transition-pair sampling. Dashed lines show the theoretical
 leading coefficient $b_0^\theta+\rho b_1^\theta$;
 the dotted line marks the zero stationary bias of the
 independent-pair control.
 Bottom: empirical signed bias and standard deviation at equal
 total update budgets, including burn-in, along
 $\varepsilon=\eta^p$ with $p=6/5$ and base step $h=0.05$.
 All estimates use 256 independent replications.
 Error bars are pointwise 95\% intervals: Student-$t$ intervals
 for bias and replication-bootstrap intervals for standard
 deviation. Solid segments connect empirical points.
 RR labels specify the cancelled powers.}
 \label{fig:tdc-diagnostics}
\end{figure*}

\subsection{Empirical Bias Estimation and Extrapolation}

Our experiments have two goals: to compare the bias estimated
from simulated trajectories with the theoretical bias expansion,
and to assess whether extrapolation improves estimation accuracy
at equal total update budgets.

\paragraph{Estimating bias from simulated trajectories.}
Fix the step sizes $\eta$ and $\varepsilon=\eta\rho$.
We generate $R=256$ independent replications of the TDC recursion.
In replication $r$, we discard the first $K$ updates and average
the next $L$ slow iterates:
\[
 \widehat\theta^{(r)}
 =\frac1L\sum_{\ell=1}^{L}\theta^{(r)}_{K+\ell}.
\]
The empirical bias is
\begin{equation}
 \widehat B^\theta_{\eta,\rho}
 =\frac1R\sum_{r=1}^{R}
   \bigl(\widehat\theta^{(r)}-\theta^\star\bigr).
 \label{eq:tdc-empirical-bias}
\end{equation}
The target $\theta^\star$ is the equilibrium of the averaged
TD problem, given in~\eqref{eq:tdc-equilibrium}.
Thus, the empirical bias is calculated directly from simulated
iterates and the known target.
For finite $K$ and $L$, this estimates the bias of the
post-burn average. We interpret it as an estimate of stationary
bias when the remaining initialization effect is negligible.

\paragraph{Comparison with the theoretical prediction.}
Theorem~\ref{thm:fixed-ratio}, specialized to this affine
TDC recursion in Appendix~\ref{subsec:tdc-application},
gives, for sufficiently small positive $\eta$ and $\rho$,
\[
 B^\theta_{\eta,\rho}
 :=\mathbb E_{\pi_{\eta,\rho}}[\theta]-\theta^\star
 =\eta\bigl(b_0^\theta+\rho b_1^\theta\bigr)
   +O(\eta^{3/2}).
\]
Consequently, at a fixed ratio $\rho$, the normalized empirical
bias $\widehat B^\theta_{\eta,\rho}/\eta$ can be compared with
the leading coefficient $b_0^\theta+\rho b_1^\theta$.
The top panels of Figure~\ref{fig:tdc-diagnostics} show this
comparison at $\rho\in\{0.1,0.25\}$.
The points and their confidence intervals are obtained from
simulations; the dashed horizontal lines mark
$b_0^\theta+\rho b_1^\theta$, the predicted limit of
the normalized stationary bias $B^\theta_{\eta,\rho}/\eta$
as $\eta\to0$ at fixed $\rho$.

As a control, we independently sample transition pairs from
the same stationary transition distribution.
This preserves the averaged TD problem and its target
$\theta^\star$ while removing serial dependence.
The stationary bias of this independent-pair recursion is zero.
Comparing its empirical bias with that of the Markov recursion
therefore examines the effect of temporal dependence.

\paragraph{Extrapolation at equal update budgets.}
We compare ordinary averaging with extrapolated estimates
along the step-size path
\[
 \eta_j=h2^{-j},\qquad
 \varepsilon_j=\eta_j^p,\qquad
 h=0.05,\quad p=6/5.
\]
Within each replication, let $\widehat\theta_j$ denote the
post-burn average at level $j$.
Ordinary averaging uses level $0$.
The two-level rule RR $\{1\}$ uses
\[
 \widehat\theta_{\mathrm{RR}\{1\}}
 =2\widehat\theta_1-\widehat\theta_0.
\]
The three-level rules RR $\{1,2\}$ and RR $\{1,p\}$
use weights that sum to one and cancel the indicated
powers of $h$.

We use total update budgets
$B\in\{2^{17},2^{19},2^{21}\}$.
A method with $m$ levels allocates
$N=\lfloor B/m\rfloor$ updates to each level, including burn-in.
Within a replication, the levels share the same input sequence.
Each replication produces one estimate for each method.
We compute empirical bias, standard deviation, and RMSE across
these independent replication-level estimates.
The bottom panels of Figure~\ref{fig:tdc-diagnostics} show
empirical bias and standard deviation.
RMSE measures their combined contribution to estimation error.

Appendix~\ref{subsec:tdc-application} verifies the assumptions
and derives the theoretical bias coefficients.
Appendix~\ref{subsec:mc-protocol} specifies the simulation
lengths, burn-in, and uncertainty calculations.
The nonlinear mixed bias term is illustrated separately in
Appendix~\ref{sec:sup:exact-markov-example}.

\section{Conclusion}\label{sec:conclusion}
For nonlinear two-timescale stochastic approximation, we
derive a stationary-bias expansion with terms
$\eta b_0+\varepsilon b_1+(\varepsilon^2/\eta)b_2$
and uniform remainder control as the step-size ratio vanishes.
Along paths $\varepsilon=\eta^p$, these contributions have
powers $1$, $p$, and $2p-1$, motivating Richardson--Romberg
extrapolation rules adapted to the chosen path.
The TDC experiments illustrate the practical trade-off:
cancelling additional bias terms can increase sampling
variability, so higher-order cancellation need not improve
accuracy at equal update budgets.
Sharper Markov remainder bounds and optimal allocation
across extrapolation levels remain open questions.

\subsection*{AI Use Statement}
ChatGPT/Codex assisted with paper review, and simulation code.
\bibliographystyle{plainnat}
\bibliography{references}
\clearpage

\appendix
\crefalias{section}{appendix}
\crefname{appendix}{Appendix}{Appendices}
\Crefname{appendix}{Appendix}{Appendices}
\onecolumn

\makeatletter
\newcommand{\appendixlabelalias}[2]{%
  \@ifundefined{r@#2}{}{%
    \expandafter\let\csname r@#1\expandafter\endcsname\csname r@#2\endcsname
    \expandafter\let\csname r@#1@cref\expandafter\endcsname\csname r@#2@cref\endcsname}}
\makeatother
\appendixlabelalias{as:sup:smooth}{as:smooth}
\appendixlabelalias{as:sup:hurwitz}{as:hurwitz}
\appendixlabelalias{as:sup:markov}{as:markov}
\appendixlabelalias{as:sup:localization}{as:localization}
\appendixlabelalias{as:sup:taylor}{as:taylor}
\appendixlabelalias{as:sup:coupling}{as:coupling}
\appendixlabelalias{prop:sup:semigroup}{prop:semigroup}
\appendixlabelalias{thm:sup:fixed-ratio}{thm:fixed-ratio}
\appendixlabelalias{thm:sup:singular}{thm:singular}
\appendixlabelalias{thm:sup:tail-average}{thm:tail}

\newtheorem*{appendixsemigrouprestatement}{Proposition~\ref*{prop:semigroup}}
\newtheorem*{appendixbiasrestatement}{Theorem~\ref*{thm:fixed-ratio}}
\newtheorem*{appendixsingularrestatement}{Theorem~\ref*{thm:singular}}
\newtheorem*{appendixtailrestatement}{Theorem~\ref*{thm:tail}}

\section*{Supplementary Derivations for Stationary Bias and Extrapolation in Nonlinear Two-Timescale Stochastic Approximation}
\tableofcontents
\clearpage
\section{Notation, assumptions, and proof roadmap}
\label{sec:sup:two-timescale-setup}

We first recall the notation and the assumptions from the main paper. The general
proofs come next, followed by the additive-noise refinement, a sufficient
localization condition, the worked example, finite-time consequences, and
the temporal-difference learning application and additional experiments.

\begin{center}
\begin{tabular}{@{}p{0.23\textwidth}p{0.71\textwidth}@{}}
\toprule
Location & Purpose \\
\midrule
Appendix~\ref{sec:sup:report-ode-geometry}
 & Fast-manifold geometry and the two-clock semigroup bounds. \\
Appendix~\ref{app:poisson-tools}
 & Poisson representation, long-run noise covariance, and Taylor estimates. \\
Appendix~\ref{sec:sup:report-markov-covariance}
 & Proofs of the uniform covariance approximation and stationary mean bias. \\
Appendix~\ref{sec:sup:singular-path-rr}
 & Singular covariance coefficients and extrapolation along step-size paths. \\
Appendix~\ref{app:iid}
 & The sharper remainder under additive independent noise. \\
Appendix~\ref{app:contractive-localization}
 & A sufficient condition for stationary localization. \\
Appendix~\ref{sec:sup:exact-markov-example}
 & Exact formulas and localization and coupling verification for the Markov example. \\
Appendix~\ref{sec:sup:finite-time}
 & Additional block localization and conditional mean-square bounds. \\
Appendix~\ref{app:experiments}
 & Temporal-difference learning model,
   assumption verification, and Monte Carlo protocol for TDC and the exact nonlinear example. \\
\bottomrule
\end{tabular}
\end{center}

Appendix~\ref{sec:sup:literature} gives additional related work.
Constants in uniform estimates are independent
of sufficiently small positive $\eta$ and $\rho$, as in the main paper.

\subsection{The Algorithm and Its Two Clocks}

The fast variable is $x_k\in\mathbb R^{d_x}$ and the slow variable is
$\theta_k\in\mathbb R^{d_\theta}$.  They evolve according to
\begin{align}
 x_{k+1}
 &=x_k+\eta H(Y_k,x_k,\theta_k),
 \label{eq:sup:report-fast-recursion}\\
 \theta_{k+1}
 &=\theta_k+\varepsilon F(Y_k,x_k,\theta_k).
 \label{eq:sup:report-slow-recursion}
\end{align}
Here $\eta>0$ is the fast step size, $\varepsilon>0$ is the slow step size,
$H$ and $F$ are the two update fields, and $Y_k$ is the random input.  The
state space is Euclidean; no projection is applied after either update.
A stationary analysis therefore needs a global stability condition
(\Cref{as:sup:localization}).

The random input is an exogenous Markov chain on a finite state space
$\mathsf Y$ with $P$ as its transition matrix. We assume that this chain has a unique
stationary distribution $\mu$ and mixes geometrically. 

The ratio
\begin{equation}
 \rho=\frac{\varepsilon}{\eta}
 \label{eq:sup:report-ratio}
\end{equation}
measures how much the slow variable moves during one unit of fast algorithmic
time.  We work with $0<\rho\leq\rho_0$, where $\rho_0$ is a sufficiently
small fixed constant.  Equivalently, $\varepsilon=\eta\rho$.

The limit $\eta\to0$ controls the size of the stationary fluctuations; the
limit $\rho\to0$ separates the two clocks:

\begin{equation}
 \boxed{
 \text{fast time }t_k=k\eta,
 \qquad
 \text{slow time }s_k=k\varepsilon=\rho t_k.}
 \label{eq:sup:report-two-clocks}
\end{equation}

We distinguish two limiting regimes. In the \emph{fixed-ratio
regime}, $\eta\to0$ while $\rho=\varepsilon/\eta$ remains fixed
and positive, so $\varepsilon=\rho\eta\to0$. In the
\emph{singular-ratio regime}, both $\eta\to0$ and $\rho\to0$,
so $\varepsilon=o(\eta)$ and the two clocks become increasingly
separated. A result established separately for each fixed
$\rho$ need not hold uniformly in the singular-ratio regime,
because its bounds may deteriorate as $\rho\to0$. The uniform
remainder bounds proved below ensure that our stationary
expansion also applies in the singular-ratio regime.

\subsection{Averaged Dynamics and the Fast Equilibrium Manifold}
This subsection introduces the averaged dynamics and the fast
equilibrium manifold, and gives a high-level interpretation
of their fast--slow structure. The qualitative discussion of
relaxation and tracking motivates the coordinates introduced
in the next subsection. The local ODE geometry and quantitative
two-clock semigroup bounds are developed in
Appendix~\ref{sec:sup:report-ode-geometry}.

The stationary distribution $\mu$ of the exogenous chain defines the
averaged update fields
\begin{align}
 \bar H(x,\theta)
 &=\sum_{y\in\mathsf Y}\mu(y)H(y,x,\theta),
 \label{eq:sup:report-averaged-H}\\
 \bar F(x,\theta)
 &=\sum_{y\in\mathsf Y}\mu(y)F(y,x,\theta).
 \label{eq:sup:report-averaged-F}
\end{align}
These fields describe the deterministic motion obtained after averaging the
rapid random fluctuations of $Y_k$.

On the fast time scale $t=k\eta$, the formal averaged ODE is
\begin{align}
 \frac{dx}{dt}&=\bar H(x,\theta),
 \label{eq:sup:report-fast-time-x-ode}\\
 \frac{d\theta}{dt}&=\rho\,\bar F(x,\theta).
 \label{eq:sup:report-fast-time-theta-ode}
\end{align}
When $\rho$ is small, $\theta$ barely changes while the fast ODE relaxes.
This suggests freezing $\theta$ and studying
\begin{equation}
 \frac{dx}{dt}=\bar H(x,\theta).
 \label{eq:sup:report-frozen-fast-ode}
\end{equation}
Assume that, for every $\theta$ in the region of interest, this frozen ODE
has a locally unique asymptotically stable equilibrium.  Denote it by
$\lambda(\theta)$, so
\begin{equation}
 \bar H(\lambda(\theta),\theta)=0.
 \label{eq:sup:report-fast-manifold}
\end{equation}
The graph $x=\lambda(\theta)$ is the fast equilibrium manifold.  After the
fast transient has died away, the pair $(x,\theta)$ should remain close to
this graph while $\theta$ evolves.

Substituting $x=\lambda(\theta)$ into the averaged slow update gives the
reduced slow field
\begin{equation}
 g(\theta)=\bar F(\lambda(\theta),\theta).
 \label{eq:sup:report-reduced-field}
\end{equation}
On slow time $s=k\varepsilon$, the reduced ODE is
\begin{equation}
 \frac{d\theta}{ds}=g(\theta).
 \label{eq:sup:report-reduced-slow-ode}
\end{equation}
Let $\theta^\star$ be its stable equilibrium, so that 
$g(\theta^\star)=0$.  The corresponding equilibrium of the full averaged
system is
\begin{equation}
 x^\star=\lambda(\theta^\star),
 \qquad
 z^\star=\begin{pmatrix}x^\star\\ \theta^\star\end{pmatrix}.
 \label{eq:sup:report-full-equilibrium}
\end{equation}
For vector-valued variables, local stability is expressed through the
Jacobians
\begin{equation}
 A=D_x\bar H(x^\star,\theta^\star),
 \qquad
 S=D_\theta g(\theta^\star).
 \label{eq:sup:report-fast-slow-jacobians}
\end{equation}

We assume that both $A$ and $S$ are Hurwitz: every eigenvalue
has strictly negative real part. For initial conditions
sufficiently close to the relevant equilibrium, this gives
bounds of the form $C e^{-ct}$ times the initial displacement
for the frozen fast ODE, and $C e^{-cs}$ times the initial
displacement for the reduced slow ODE, with suitable constants
$C,c>0$. Thus both deterministic systems approach their
equilibria at an exponential rate, measured on their respective
time scales $t=k\eta$ and $s=k\varepsilon$. We call this
exponential approach toward equilibrium \emph{exponential
relaxation}, and use \emph{relaxation} as shorthand below.
Local Hurwitz stability does not ensure an invariant law
for the unprojected stochastic recursion; global localization
is imposed in \Cref{as:sup:localization}.

\subsection{Coordinates Relative to the Fast Equilibrium Manifold}
The raw error $x-x^\star$ combines two effects: displacement away from the
fast manifold and movement of the manifold caused by $\theta-\theta^\star$.
It is more informative to introduce
\begin{equation}
 \widehat u=x-\lambda(\theta),
 \qquad
 v=\theta-\theta^\star.
 \label{eq:sup:report-nonlinear-normal-coordinate}
\end{equation}
These coordinates are initially defined where the local map $\lambda$
is available. The variable $\widehat u$ measures displacement from the
moving fast equilibrium, while $v$ measures slow displacement.
Their use throughout the stationary law requires the extension
specified in Assumption~\ref{as:block-localization}.

One update shows why this coordinate is useful. Let
$x^+=x+\eta H(Y,x,\theta)$ and
$\theta^+=\theta+\varepsilon F(Y,x,\theta)$.
When the segment $[\theta,\theta^+]$ lies in the domain of the
chosen $C^2$ map $\lambda$, we have
\begin{align}
 \widehat u^+
 &=x^+-\lambda(\theta^+) \notag\\
 &=\widehat u+\eta H(Y,x,\theta)
   -\{\lambda(\theta+\varepsilon F(Y,x,\theta))-\lambda(\theta)\}.
 \label{eq:sup:report-normal-coordinate-exact-step}
\end{align}
For $\lambda:\mathbb R^{d_\theta}\to\mathbb R^{d_x}$,
the Jacobian with respect to $\theta$ has dimensions
\[
 D_\theta\lambda(\theta)\in\mathbb R^{d_x\times d_\theta},
 \qquad D_\theta\lambda(\theta)F(Y,x,\theta)\in\mathbb R^{d_x}.
\]
If $\lambda$ has a locally Lipschitz derivative (in particular, if
$\lambda\in C^2$), Taylor expansion gives
\[
 \lambda(\theta+\varepsilon F)-\lambda(\theta)
 =\varepsilon D_\theta\lambda(\theta)F
   +O(\varepsilon^2\|F\|^2).
\]
Using $\varepsilon=\eta\rho$ yields
\begin{align}
 \widehat u^+
 &=\widehat u+\eta\bigl[H(Y,x,\theta)
     -\rho D_\theta\lambda(\theta)F(Y,x,\theta)\bigr]
   +O(\varepsilon^2\|F(Y,x,\theta)\|^2).
 \label{eq:sup:report-normal-coordinate-expanded-step}
\end{align}
The factor $\rho=\varepsilon/\eta$ is now visible in front of the slow
forcing of the fast-manifold error.   This factor must be retained. A generic bound for the stacked recursion can replace the two relaxation rates by the slow spectral gap, producing a spurious factor $1/\rho$. In the coordinate $(\widehat u,v)$, slow forcing has size $\rho$ and acts over a relaxation time of order $1/\rho$ on the fast clock; the two factors compensate.

Near $z^\star$, reserve $u$ for the linear approximation of the nonlinear
tracking error:
\begin{equation}
 u=(x-x^\star)-D_\theta\lambda(\theta^\star)(\theta-\theta^\star)
 \qquad\text{and hence}\qquad
 \widehat u=u+O(\|\theta-\theta^\star\|^2).
 \label{eq:sup:report-linear-normal-coordinate}
\end{equation}
To show that the difference between $\widehat u$ and its linear
approximation $u$ is $O(\|\theta-\theta^\star\|^2)$, let
$\delta x=x-x^\star$, $v=\theta-\theta^\star$, and
$\Lambda=D_\theta\lambda(\theta^\star)$.
Since $x^\star=\lambda(\theta^\star)$, the $C^2$ Taylor formula is
\[
 \lambda(\theta^\star+v)=x^\star+\Lambda v+r_\lambda(v),\qquad
 r_\lambda(v)=\int_0^1(1-s)
 D_\theta^2\lambda(\theta^\star+sv)[v,v]\,ds.
\]
A bound $\|D_\theta^2\lambda\|\le M$ along the segment gives
$\|r_\lambda(v)\|\le(M/2)\|v\|^2$. Therefore
\[
 \widehat u=x-\lambda(\theta)=\delta x-\Lambda v-r_\lambda(v)
             =u-r_\lambda(v).
\]
This proves \cref{eq:sup:report-linear-normal-coordinate} in vector norm; the
minus sign is absorbed by the big-$O$ notation. Differentiability alone
would give only $\widehat u=u+o(\|v\|)$.
The nonlinear version is useful for localization and transient estimates;
the linearized version is convenient when computing covariance and bias
coefficients.

\subsection{Compact Notation}

Define the stacked state and update
\begin{equation}
 z=\begin{pmatrix}x\\ \theta\end{pmatrix},
 \qquad
 G(y,z)=\begin{pmatrix}H(y,x,\theta)\\ F(y,x,\theta)\end{pmatrix},
 \label{eq:sup:report-stacked-state-field}
\end{equation}
and define the block scaling matrix
\begin{equation}
 D_\rho=
 \begin{pmatrix}
  I_{d_x}&0\\
  0&\rho I_{d_\theta}
 \end{pmatrix}.
 \label{eq:sup:report-block-scaling}
\end{equation}
Here $I_{d_x}$ and $I_{d_\theta}$ are identity matrices of dimensions
$d_x$ and $d_\theta$.  The recursions in
\cref{eq:sup:report-fast-recursion,eq:sup:report-slow-recursion} are then exactly
equivalent to
\begin{equation}
 z_{k+1}=z_k+\eta D_\rho G(Y_k,z_k).
 \label{eq:sup:report-stacked-recursion}
\end{equation}
The lower block has the smaller step $\eta\rho=\varepsilon$.

\subsection{Standing Assumptions}
\label{subsec:sup:standing-assumptions}

The core stationary assumptions are stated in
Section~\ref{sec:assumptions}; the additional coupling and block
conditions are stated separately.

\begin{description}
\item[Assumptions~\ref{as:smooth} and \ref{as:hurwitz}.]
Local smoothness and Hurwitz stability specify the deterministic
geometry and the fixed tangent transformation.
\item[Assumption~\ref{as:markov}.]
The exogenous finite-state chain satisfies the stated Doeblin condition.
\item[Assumption~\ref{as:localization}.]
A selected family of joint invariant laws satisfies
$\mathbb E\|z-z^\star\|^4\le C\eta^2$, uniformly in the step sizes.
\item[Assumption~\ref{as:taylor}.]
The update fields have globally bounded first and second derivatives.
Lemma~\ref{lem:derived-taylor-estimates} derives the expected Taylor
and Poisson-corrector bounds from these conditions.
\item[Assumption~\ref{as:coupling}.]
The $L^1$ coupling controls initialization of expected tail averages.
The blockwise mean-square statements additionally require the
coordinate $L^2$ coupling bounds and the uniform stationary $V_2$
moment displayed in Appendix~\ref{sec:sup:finite-time}.
\item[Assumption~\ref{as:block-localization}.]
The stronger raw and centered block moments, tangent cross-covariance
bound, and stationary extension and remainder conditions for $\lambda$
support supplementary blockwise fluctuation statements.
They are not assumptions of the core covariance and bias expansions.
\end{description}

The local deterministic and Markov assumptions determine the coefficients;
fourth-moment localization and update regularity justify the stationary
expansions. Additional block conditions and coupling are invoked only
for the conclusions that use them.

\subsection{Quantities of Interest and Step-Size Regimes}
Suppose that for sufficiently small $(\eta,\rho)$ the joint chain
$(Y_k,z_k)$ admits an invariant law, denoted by $\pi_{\eta,\rho}$.  The
stationary mean displacement from the averaged equilibrium is
\begin{equation}
 m_{\eta,\rho}
 =\mathbb E_{\pi_{\eta,\rho}}[z]-z^\star.
 \label{eq:sup:report-stationary-displacement}
\end{equation}
The centered stationary covariance is
\begin{equation}
 \Gamma_{\eta,\rho}
 =\mathbb E_{\pi_{\eta,\rho}}
 \left[(z-\mathbb E_{\pi_{\eta,\rho}}z)
       (z-\mathbb E_{\pi_{\eta,\rho}}z)^\top\right].
 \label{eq:sup:report-stationary-covariance}
\end{equation}
The superscript $\top$ denotes matrix transpose.  The corresponding
mean-square distance decomposes as
\begin{equation}
 \mathbb E_{\pi_{\eta,\rho}}
   \|z-z^\star\|^2
 =\operatorname{tr}(\Gamma_{\eta,\rho})
  +\|m_{\eta,\rho}\|^2,
 \label{eq:sup:report-vector-mse-decomposition}
\end{equation}
where $\operatorname{tr}$ is the matrix trace and $\|\cdot\|$ is the
Euclidean norm.   Finite-time concentration, stationary covariance, and stationary bias expansions thus address distinct properties of the recursion.

For a trajectory started from an arbitrary initial condition,
let $K\ge0$ be the burn-in length and let $N>K$. We average
the iterates $z_K,\ldots,z_{N-1}$, so the averaging window
contains $L=N-K$ iterates.
The tail
average is
\begin{equation}
 \bar z_{K,N}=\frac{1}{L}\sum_{k=K}^{N-1}z_k.
 \label{eq:sup:report-tail-average}
\end{equation}
Let $\nu$ denote the initial law of the joint chain. Adding and subtracting
the invariant mean gives the exact decomposition
\[
 \begin{aligned}
 \mathbb E_\nu[\bar z_{K,N}]-z^\star
   &=m_{\eta,\rho}+ \mathsf{T}_{K,N}(\nu),\\
\mathsf{T}_{K,N}(\nu)
   &:=\frac1L\sum_{k=K}^{N-1}
      \bigl(\mathbb E_\nu[z_k]-\mathbb E_{\pi_{\eta,\rho}}[z]\bigr).
 \end{aligned}
\]
The first term is the persistent stationary displacement. Under stationary
initialization, $\mathsf{T}_{K,N}=0$ for every $K,N$; averaging reduces sampling
variability but leaves the mean displacement unchanged. The second term is
the initialization transient.
To bound the initialization transient $\mathsf{T}_{K,N}(\nu)$, we need
an estimate of how quickly the expected iterates approach
the stationary mean. Theorem~\ref{thm:tail} bounds this contribution under the
two-clock coupling assumption.

The ratio $\rho$ determines the form of the persistent displacement.
Fixed-ratio analysis holds $\rho\in(0,\rho_0]$ fixed and lets $\eta\downarrow0$
with $\varepsilon=\rho\eta$, giving a representation
\[
 m_{\eta,\rho}=\eta b(\rho)+R_{\eta,\rho}.
\]
Singular-ratio analysis also lets $\rho\downarrow0$. At $\rho=0$ the slow
block of $D_\rho$ vanishes, the slow equation freezes, and its spectral gap
collapses.  A remainder bound at each fixed $\rho$ need not be uniform in this limit.

Appendix~\ref{sec:sup:singular-path-rr} develops the ratio expansion
and explains how its powers determine the extrapolation weights.

\section{ODE geometry of two timescales}
\label{sec:sup:report-ode-geometry}

This section explains how coordinates relative to the fast
equilibrium manifold separate fast tracking from slow
evolution. As $\rho\to0$, slow relaxation takes increasingly
long on the fast clock. We therefore keep the two relaxation
rates and the factors of $\rho$ in the coupling terms explicit.

We first derive the exact coordinate dynamics for both the
averaged ODE and the discrete recursion
(Subsection~\ref{subsec:sup:ode-exact-dynamics}).
We then linearize at the equilibrium to identify the fast
and reduced slow Jacobian blocks
(Subsection~\ref{subsec:sup:ode-linearization}).
This linear structure leads to bounds that retain both
relaxation rates, with constants independent of sufficiently
small positive $\rho$
(Subsection~\ref{subsec:sup:ode-semigroup}).
These deterministic calculations provide tools for the
stationary covariance and bias proofs developed later;
stationary existence and localization require additional
assumptions.

\subsection{Exact Fast-Manifold Dynamics}
\label{subsec:sup:ode-exact-dynamics}

We use the clocks \eqref{eq:sup:report-two-clocks}, the fast equilibrium graph
\eqref{eq:sup:report-fast-manifold}, the reduced field
\eqref{eq:sup:report-reduced-field}, and the fast-manifold coordinates
\((\widehat u,v)\) of \eqref{eq:sup:report-nonlinear-normal-coordinate}, and
derive the dynamics in these coordinates.

To distinguish the exact coordinate $\widehat u$ from its tangent
approximation $u$, use a neutral dummy variable $w\in\mathbb R^{d_x}$ and
define
\begin{align}
 h(w,v)
 &=\bar H\bigl(\lambda(\theta^\star+v)+w,\theta^\star+v\bigr),
 \label{eq:sup:report-h-definition}\\
 f(w,v)
 &=\bar F\bigl(\lambda(\theta^\star+v)+w,\theta^\star+v\bigr),
 \label{eq:sup:report-f-definition}\\
 \ell(v)&=D_\theta\lambda(\theta^\star+v)
          \in\mathbb R^{d_x\times d_\theta}.
 \label{eq:sup:report-ell-definition}
\end{align}
Here $\ell(v)$ is the Jacobian of $\lambda$ with respect to $\theta$
(equivalently, with respect to $v$); it is not a time derivative. At the
actual state $w=\widehat u$, since
$x=\lambda(\theta^\star+v)+\widehat u$.
The chain rule applied to the fast-clock ODE
\cref{eq:sup:report-fast-time-x-ode,eq:sup:report-fast-time-theta-ode} gives the exact
fast-clock equations
\begin{equation}
 \dot{\widehat u}=h(\widehat u,v)-\rho\ell(v)f(\widehat u,v),
 \qquad
 \dot v=\rho f(\widehat u,v).
 \label{eq:sup:report-exact-normal-ode}
\end{equation}
The fast-manifold identity \eqref{eq:sup:report-fast-manifold} becomes
\begin{equation}
 h(0,v)=0
 \quad\text{for every admissible }v.
 \label{eq:sup:report-h-zero-on-manifold}
\end{equation}
Thus a tangential displacement \(v\) cannot create an order-one forcing in
the normal equation while \(\widehat u=0\).  The only such forcing visible in
\eqref{eq:sup:report-exact-normal-ode} carries the explicit factor \(\rho\).
 A one-timescale norm bound for the stacked system obscures these factors.

There is an exact discrete counterpart.  Abbreviate
\begin{equation}
 H_k=H(Y_k,x_k,\theta_k),
 \qquad
 F_k=F(Y_k,x_k,\theta_k).
\end{equation}
Set $\Delta_k=\theta_{k+1}-\theta_k=\eta\rho F_k$. Whenever the
segment $[\theta_k,\theta_{k+1}]$ lies in the domain of the chosen
$C^2$ extension of $\lambda$, applying the fundamental theorem of calculus twice gives
\[
 \begin{aligned}
 \lambda(\theta_k+\Delta_k)-\lambda(\theta_k)
 &=\int_0^1D_\theta\lambda(\theta_k+s\Delta_k)\Delta_k\,ds\\
 &=D_\theta\lambda(\theta_k)\Delta_k
   +\int_0^1\int_0^s D_\theta^2\lambda(\theta_k+t\Delta_k)
        [\Delta_k,\Delta_k]\,dt\,ds\\
 &=D_\theta\lambda(\theta_k)\Delta_k
   +\int_0^1(1-s)D_\theta^2\lambda(\theta_k+s\Delta_k)
        [\Delta_k,\Delta_k]\,ds.
 \end{aligned}
\]
The last equality interchanges the order of integration. Subtracting this
identity from $x_{k+1}-x_k=\eta H_k$ gives
\begin{align}
 \widehat u_{k+1}-\widehat u_k
 ={}&\eta H_k-\eta\rho D_\theta\lambda(\theta_k)F_k
 \nonumber\\
 &-\eta^2\rho^2\int_0^1(1-s)
 D_\theta^2\lambda(\theta_k+s\eta\rho F_k)[F_k,F_k]\,ds,
 \label{eq:sup:report-exact-discrete-normal}\\
 v_{k+1}-v_k={}&\eta\rho F_k.
 \label{eq:sup:report-exact-discrete-tangent}
\end{align}
Here $D_\theta^2\lambda(\theta)[a,b]\in\mathbb R^{d_x}$ is a
vector-valued bilinear second derivative. The factor $1/2$ is already
contained in $\int_0^1(1-s)\,ds=1/2$. No derivative of $F_k$ enters: $F_k$ is held fixed as the direction of the one-step increment. This identity involves no asymptotic approximation. 

Because $\widehat u_k=x_k-\lambda(\theta_k)$, updating
$\theta_k$ also changes the reference fast equilibrium
$\lambda(\theta_k)$. The first-order contribution of this
change to the tracking error is
$-\eta\rho D_\theta\lambda(\theta_k)F_k$,
which has prefactor $\eta\rho$. The curvature remainder
in \eqref{eq:sup:report-exact-discrete-normal} accounts for
the nonlinear part of this change and has prefactor
$\eta^2\rho^2$. After division by $\eta$ to express increments
per unit of fast time, these prefactors become $\rho$ and
$\eta\rho^2$, respectively. The curvature remainder therefore
contains an additional step-size factor as well as two powers
of $\rho$. If the remaining coefficient factors are uniformly
controlled, integration against a slow kernel gives the
scaling estimates
\[
 \int_0^\infty e^{-c\rho t}\rho\,dt=\frac1c,
 \qquad
 \int_0^\infty e^{-c\rho t}\eta\rho^2\,dt
 =\frac{\eta\rho}{c}.
\]

The preceding integrals illustrate how a factor $\rho$ in
a forcing term offsets the factor $1/\rho$ arising from slow
relaxation. The stationary covariance proof verifies this
cancellation directly in the fixed tangent coordinates
$(u,v)$. When the covariance equation is partitioned according
to these coordinates, its $(v,v)$ block corresponds to
covariances between components of the slow variable
$v=\theta-\theta^\star$. We call this the
\emph{slow--slow block}.

The remainder in this block is $O(\rho\eta^{3/2})$, as shown
in \eqref{eq:cov-residual-blocks}.
The inverse estimate in
Lemma~\ref{lem:cov-lyapunov-inverse} weights this remainder
by $1/\rho$. The factor $\rho$ in the remainder therefore
cancels this weight, giving an $O(\eta^{3/2})$ contribution
to the covariance error, uniformly in sufficiently small
positive $\rho$.

\subsection{Linearization at the Equilibrium}
\label{subsec:sup:ode-linearization}

Recall the four Jacobian blocks at $z^\star$ introduced in
Subsection~\ref{sec:model-recursion}:
\begin{equation}
 A=D_x\bar H(z^\star),\quad
 B=D_\theta\bar H(z^\star),\quad
 C=D_x\bar F(z^\star),\quad
 D=D_\theta\bar F(z^\star).
 \label{eq:sup:report-jacobian-blocks}
\end{equation}

Differentiating \eqref{eq:sup:report-fast-manifold} at
\(\theta^\star\) gives
\begin{equation}
 A\Lambda+B=0,
 \qquad
 \Lambda=D\lambda(\theta^\star)=-A^{-1}B.
 \label{eq:sup:report-manifold-tangent}
\end{equation}
The inverse \(A^{-1}\) is available when the frozen fast equilibrium is
locally exponentially stable.  The reduced field was defined in \cref{eq:sup:report-reduced-field} by
$g(\theta)=\bar F(\lambda(\theta),\theta)$. Its chain rule is
\[
 D_\theta g(\theta^\star)
 =D_x\bar F(x^\star,\theta^\star)D_\theta\lambda(\theta^\star)
  +D_\theta\bar F(x^\star,\theta^\star).
\]
Substituting the Jacobian blocks gives
\begin{equation}
 S=Dg(\theta^\star)=C\Lambda+D=D-CA^{-1}B.
 \label{eq:sup:report-reduced-jacobian}
\end{equation}
For the local calculation, use the tangent coordinates \((u,v)\) of
\eqref{eq:sup:report-linear-normal-coordinate}.  Write $\delta x=x-x^\star$ and $\delta\theta=\theta-\theta^\star$.
Retaining the first-order terms of the fast-clock ODE at $z^\star$ gives
the first-variation system
\begin{equation}
 \frac{d}{dt}
 \begin{pmatrix}\delta x\\ \delta\theta\end{pmatrix}
 =
 \begin{pmatrix}A&B\\ \rho C&\rho D\end{pmatrix}
 \begin{pmatrix}\delta x\\ \delta\theta\end{pmatrix},
 \label{eq:sup:report-raw-linearization}
\end{equation}
For finite deviations with $r=\|(\delta x,\delta\theta)\|$, the omitted
remainders are $O(r^2)$ in the first row and $\rho O(r^2)$ in the second.
Now substitute $\delta x=u+\Lambda v$. The lower row gives
\[
 \dot v=\rho C(u+\Lambda v)+\rho Dv
       =\rho Cu+\rho(C\Lambda+D)v=\rho Cu+\rho Sv.
\]
Because $\Lambda$ is fixed in this tangent-coordinate calculation,
$\dot u=\dot{\delta x}-\Lambda\dot v$. Thus
\[
 \begin{aligned}
 \dot u
 &=A(u+\Lambda v)+Bv-\Lambda\{\rho C(u+\Lambda v)+\rho Dv\}\\
 &=(A-\rho\Lambda C)u+
   \{A\Lambda+B-\rho\Lambda(C\Lambda+D)\}v\\
 &=(A-\rho\Lambda C)u-\rho\Lambda Sv.
 \end{aligned}
\]
The identities $A\Lambda+B=0$ and $C\Lambda+D=S$ therefore give
\begin{equation}
 \frac{d}{dt}
 \begin{pmatrix}u\\v\end{pmatrix}
 =L_\rho
 \begin{pmatrix}u\\v\end{pmatrix},
 \qquad
 L_\rho=
 \begin{pmatrix}
 A-\rho\Lambda C&-\rho\Lambda S\\
 \rho C&\rho S
 \end{pmatrix}.
 \label{eq:sup:report-normal-linearization}
\end{equation}
Every off-diagonal interaction in these coordinates has a factor \(\rho\).
At \(\rho=0\), the \(u\)-block is the stable matrix \(A\) and the \(v\)-block
is frozen. For small positive $\rho$, the diagonal slow block
$\rho S$ is Hurwitz, and the complete slow equation is
$\dot v=\rho Sv+\rho Cu$. The term $\rho Cu$ describes
the coupling from the fast deviation to the slow variable.

\subsection{Two-Clock Semigroup Bounds}
\label{subsec:sup:ode-semigroup}

We now quantify relaxation in the linearized system:
$e^{tL_\rho}$ maps the initial deviation $(u(0),v(0))$
to $(u(t),v(t))$. The proof separates fast and slow invariant
subspaces to obtain bounds that retain both relaxation rates
and the factors $\rho$ in their interactions, uniformly for
small positive $\rho$, for both the ODE and its Euler
discretization.

\begin{appendixsemigrouprestatement}[Two-clock semigroup bounds; restatement]
Under \Cref{as:sup:smooth,as:sup:hurwitz}, there exist constants $C_*,c,\rho_0>0$,
independent of $\rho$, such that for $0<\rho\le\rho_0$ and $t\ge0$,
\begin{equation}
 \begin{aligned}
 \|[e^{tL_\rho}]_{uu}\|
   &\le C_*\bigl(e^{-ct}+\rho e^{-c\rho t}\bigr),\\
 \|[e^{tL_\rho}]_{uv}\|+\|[e^{tL_\rho}]_{vu}\|
   &\le C_*\rho\bigl(e^{-ct}+e^{-c\rho t}\bigr),\\
 \|[e^{tL_\rho}]_{vv}\|&\le C_*e^{-c\rho t}.
 \end{aligned}
 \label{eq:sup:report-semigroup-orders}
\end{equation}
The same bounds hold for $(I+\eta L_\rho)^k$, with $t=\eta k$,
uniformly for sufficiently small $\eta$ and integers $k\ge0$.
\end{appendixsemigrouprestatement}

The block subscripts specify the output and input coordinates,
in that order. Thus $[e^{tL_\rho}]_{uv}$ gives the contribution
to $u(t)$ from the initial slow deviation $v(0)$, while
$[e^{tL_\rho}]_{vu}$ gives the contribution to $v(t)$ from
the initial fast deviation $u(0)$. The factors $e^{-ct}$ and
$e^{-c\rho t}$ describe fast and slow relaxation, respectively.

\begin{proof}
Write the blocks of $L_\rho$ as
\[
 L_\rho=\begin{pmatrix}A_\rho& \mathcal B_\rho\\C_\rho&D_\rho'\end{pmatrix},
 \quad A_\rho=A-\rho\Lambda C,\quad \mathcal B_\rho=-\rho\Lambda S,
 \quad C_\rho=\rho C,\quad D_\rho'=\rho S.
\]
We seek invariant subspaces of the form $u=X_\rho v$ and
$v=R_\rho u$: a solution starting in either subspace remains
there. These describe the perturbed slow and fast subspaces,
respectively.
Their invariance equations are
\begin{align*}
 A_\rho X_\rho+ \mathcal B_\rho
   -X_\rho(C_\rho X_\rho+D_\rho')&=0,\\
 C_\rho+D_\rho'R_\rho
   -R_\rho(A_\rho+ \mathcal B_\rho R_\rho)&=0.
\end{align*}
At $\rho=0$, the derivatives in $X$ and $R$ are respectively
$X\mapsto AX$ and $R\mapsto-RA$, both invertible.
The implicit-function theorem therefore gives
$X_\rho,R_\rho=O(\rho)$.  The matrix
\[
 Q_\rho=\begin{pmatrix}I&X_\rho\\R_\rho&I\end{pmatrix}
\]
is uniformly invertible for small $\rho$, and
\[
 Q_\rho^{-1}L_\rho Q_\rho
 =\operatorname{diag}(A_f(\rho),\rho S_s(\rho)),
 \quad A_f=A_\rho+ \mathcal B_\rho R_\rho=A+O(\rho),
 \quad S_s=S+CX_\rho=S+O(\rho).
\]
Fixed quadratic Lyapunov functions for the Hurwitz matrices $A$ and $S$
remain strict Lyapunov functions for these perturbations. Consequently,
\[
 \|e^{tA_f}\|\le C e^{-ct},\qquad
 \|e^{\rho tS_s}\|\le C e^{-c\rho t},
\]
with constants independent of $\rho$.
Finally, writing $Q_\rho=I+E_\rho$, where $E_\rho$ has only off-diagonal blocks of size $O(\rho)$, the Neumann series $Q_\rho^{-1}=I-E_\rho+E_\rho^2-\cdots$ shows that its diagonal corrections are $O(\rho^2)$ and its off-diagonal blocks are $O(\rho)$. Multiplication in $e^{tL_\rho}=Q_\rho\operatorname{diag}(e^{tA_f},e^{\rho tS_s})Q_\rho^{-1}$ then preserves the two clocks and the small cross factors. Multiplying the block matrices proves
\eqref{eq:sup:report-semigroup-orders}; in fact, the upper-left slow contribution and
lower-right fast contribution are both $O(\rho^2)$.

For completeness, the same Lyapunov functions give, uniformly
for small $\eta$, contraction factors $1-c\eta$ and
$1-c\eta\rho$ for the squared Lyapunov norms of
$I+\eta A_f$ and $I+\eta\rho S_s$, respectively.
Indeed, if $P_f$ is the fixed fast
Lyapunov matrix, then
\[
 (I+\eta A_f)^\top P_f(I+\eta A_f)-P_f
 =\eta(A_f^\top P_f+P_fA_f)+\eta^2A_f^\top P_fA_f
 \preceq-c\eta P_f,
\]
after decreasing $\eta_0$; the slow argument is identical with
$\eta$ replaced by $\eta\rho$. Thus the same block bounds hold for
$(I+\eta L_\rho)^k$, with $t=\eta k$.
\end{proof}

The upper equation of \cref{eq:sup:report-normal-linearization} is
\[
 \dot u=(A-\rho\Lambda C)u-\rho\Lambda S v.
\]
To isolate the upper-right channel, consider a one-way scalar model with
coupling $\beta$ and stable rates $-a,-s$.  Prescribe the
uncoupled slow response $v(t)=e^{-\rho s t}v_0$ and solve
\begin{equation}
 \dot u=-a u+\rho\beta v,
 \qquad a,s>0.
\end{equation}
Then
\begin{align}
 u(t)
 &=e^{-at}u_0+\rho\beta v_0
   \int_0^t e^{-a(t-r)}e^{-\rho s r}\,dr
 \nonumber\\
 &=e^{-at}u_0+\rho\beta v_0
   \frac{e^{-\rho s t}-e^{-at}}{a-\rho s}.
 \label{eq:sup:report-scalar-slow-to-fast}
\end{align}
For small $\rho$, the denominator is bounded away from zero and the slow
tail in $u$ has amplitude $O(\rho)$.
The reverse channel is illustrated by a scalar coupling $\gamma$ in
$\dot v=-\rho s v+\rho\gamma u$. Prescribing the uncoupled fast response
$u(t)=e^{-at}u_0$ gives
\[
 \dot v+\rho s v=\rho\gamma u_0e^{-at},\qquad
 \frac{d}{dt}\{e^{\rho st}v(t)\}=\rho\gamma u_0e^{(\rho s-a)t}.
\]
Integrating from zero to $t$, for $a\ne\rho s$,
\[
 e^{\rho st}v(t)-v_0
 =\rho\gamma u_0\frac{e^{(\rho s-a)t}-1}{\rho s-a}.
\]
Multiplication by $e^{-\rho st}$ yields
\begin{equation}
 v(t)=e^{-\rho s t}v_0+\rho\gamma u_0
 \frac{e^{-\rho s t}-e^{-at}}{a-\rho s},
 \label{eq:sup:report-scalar-fast-to-slow}
\end{equation}
which again has cross amplitude $O(\rho)$ for sufficiently small $\rho$.
At the exceptional resonant value $a=\rho s$, the continuous limiting
formula is
\[
 v(t)=e^{-at}\{v_0+\rho\gamma u_0t\}.
\]
The same denominator caveat applies to
\cref{eq:sup:report-scalar-slow-to-fast}; at resonance its convolution is
$t e^{-at}$. One may choose $\rho_0<a/s$ to avoid resonance throughout the
small-ratio range. These one-way formulas illustrate the two cross bounds
in \cref{eq:sup:report-semigroup-orders}.

The slow exponential integrates to \(O(1/\rho)\) on the fast clock:
\begin{equation}
 \int_0^\infty e^{-c\rho t}\,dt=\frac{1}{c\rho}.
 \label{eq:sup:report-slow-integral}
\end{equation}
If a forcing term has size $O(\rho)$ with a constant
independent of $\rho$, integration against this slow
exponential gives an $O(1)$ bound. Discarding the factor
$\rho$ before integration can therefore introduce an
unnecessary factor $1/\rho$ in the estimate.

Returning to iteration number \(k\), \(t=k\eta\) changes the two exponentials
into
\begin{equation}
 e^{-c\eta k}
 \quad\text{and}\quad
 e^{-c\rho\eta k}=e^{-c\varepsilon k}.
 \label{eq:sup:report-discrete-two-exponentials}
\end{equation}
These linear bounds identify the fast and slow relaxation
rates relevant to initialization effects. Quantitative
initialization bounds for the nonlinear stochastic recursion
are developed in Appendix~\ref{sec:sup:finite-time} under the
two-clock coupling conditions. Stationary bias calculations
start from an invariant law and do not require an
initialization estimate.

\section{Poisson tools and Taylor estimates}
\label{app:poisson-tools}
We first compute the long-run covariance $Q$ of the noise
$\xi_k=G(Y_k,z^\star)$, obtained by holding the algorithmic
state at the equilibrium $z^\star$ while the Markov chain
continues to evolve
(Subsection~\ref{subsec:sup:Q}).
Because successive Markov inputs are correlated, the
accumulated noise fluctuations depend on correlations across
time, which $Q$ records. This matrix supplies the noise term
in the Lyapunov equation for the leading stationary covariance
of the iterates. That covariance enters the quadratic Taylor
terms of the averaged drift and therefore determines the
curvature contribution to the stationary mean bias.

Next, for a fixed algorithmic state $z$, we use the Poisson
equation to decompose the centered Markov update into a
martingale difference and a telescoping correction, whose
sum reduces to initial and final terms
(Subsection~\ref{subsec:sup:poisson}).
We then derive Taylor estimates for the nonlinear remainder
terms and for changes in the Poisson solution as the
algorithmic state evolves
(Subsection~\ref{app:taylor-estimates}).
Appendix~\ref{sec:sup:report-markov-covariance} uses these
tools to establish the stationary covariance and mean-bias
expansions.

\subsection{Long-Run Covariance of the Markov Noise}
\label{subsec:sup:Q}

Recall the compact recursion
\[
 z_{k+1}=z_k+\eta D_\rho G(Y_k,z_k),
\]
where \(z=(x,\theta)\), \(D_\rho=\operatorname{diag}(I,\rho I)\), and
\(\rho=\varepsilon/\eta\).  Average the update over the stationary law
\(\mu\) of the exogenous Markov chain:
\begin{equation}
 \bar G(z)=\sum_y\mu(y)G(y,z),
 \qquad 
 \mbox{ define }
 \qquad
 \widetilde G(y,z)=G(y,z)-\bar G(z).
 \label{eq:sup:report-centered-field}
\end{equation}
At the averaged
equilibrium \(z^\star\), we have \(\bar G(z^\star)=0\).
The update noise
frozen at that point is therefore
\begin{equation}
 \xi_k=G(Y_k,z^\star)=\widetilde G(Y_k,z^\star).
 \label{eq:sup:report-frozen-increment}
\end{equation}
For each fixed $z$, $\widetilde G(\cdot,z)$ is centered
under $\mu$. 

In \Cref{subsec:sup:Q,subsec:sup:poisson}, expectations refer to the stationary Markov path:
\(Y_0\sim\mu\), followed by transitions with kernel \(P\).  In particular,
\(\mathbb E\xi_k=0\).   This convention specifies the noise covariance without requiring stationary initialization of the algorithm.

With Markov data the variance of a long sum contains all cross-time
covariances, so the relevant matrix is
\begin{equation}
 Q=\lim_{m\to\infty}\frac{1}{m}
 \mathbb E\left[
 \left(\sum_{k=0}^{m-1}\xi_k\right)
 \left(\sum_{k=0}^{m-1}\xi_k\right)^{\!\top}
 \right].
 \label{eq:sup:report-long-run-covariance}
\end{equation}
Thus \(Q\) is the covariance accumulated per step.  Write \(S_m=\sum_{k=0}^{m-1}\xi_k\) and
\(K_h=\mathbb E[\xi_0\xi_h^\top]\).  Since the sum is centered,
\[
 \operatorname{Cov}(S_m)=\mathbb E[S_mS_m^\top]
 =\sum_{i=0}^{m-1}\sum_{j=0}^{m-1}\mathbb E[\xi_i\xi_j^\top].
\]
Stationarity makes each diagonal term equal to \(K_0\).  For a positive
lag \(h\), the \(m-h\) pairs with \(j=i+h\) each contribute \(K_h\),
whereas the \(m-h\) reversed pairs each contribute \(K_h^\top\).  Thus
the exact finite-sum identity is
\[
 \frac1m\mathbb E[S_mS_m^\top]
 =K_0+\sum_{h=1}^{m-1}
       \left(1-\frac hm\right)(K_h+K_h^\top).
\]

Suppose \(\sum_{h\ge1}\|K_h\|<\infty\).  For every fixed \(h\), the
factor \(\mathbf1_{\{h<m\}}(1-h/m)\) tends to one and lies in \([0,1]\).
The norm of the corresponding summand is bounded by \(2\|K_h\|\), a
summable sequence.
The dominated convergence theorem for series, applied
entrywise with summable dominating sequence $2\|K_h\|$,
therefore gives
\begin{equation}
 Q=K_0+\sum_{h=1}^\infty
 \bigl(K_h+K_h^\top\bigr).
 \label{eq:sup:report-autocovariance-sum}
\end{equation}
The summability condition follows from
Assumption~\ref{as:sup:markov}.
Indeed, on the finite state space the centered function
\(\xi(y)=G(y,z^\star)\) is bounded.  Geometric mixing gives constants
\(C<\infty\) and \(r\in(0,1)\) such that
\[
 \sup_y\|(P^h\xi)(y)\|\le Cr^h\max_y\|\xi(y)\|.
\]
Conditioning on \(Y_0\) then gives
\[
 K_h=\sum_y\mu(y)\xi(y)(P^h\xi)(y)^\top,
 \qquad
 \|K_h\|\le Cr^h\bigl(\max_y\|\xi(y)\|\bigr)^2.
\]
Consequently, \eqref{eq:sup:report-autocovariance-sum} holds here.
Stationarity alone would not suffice: for the stationary alternating chain
\(Y_{k+1}=-Y_k\), with equally likely initial signs and \(\xi_k=Y_k\),
the partial sums are bounded and \(Q=0\), but \(K_h=(-1)^h\) makes the
displayed infinite series divergent.  That periodic chain is excluded by
geometric mixing.

By contrast, \(\mathbb E[m^{-1}\sum_{k<m}\xi_k\xi_k^\top]=K_0\) omits every
off-diagonal term, so \(Q\) generally differs from the one-step
covariance.

\subsection{Martingale Decomposition via the Poisson Equation}
\label{subsec:sup:poisson}

For a function \(a\) of the Markov state, define
\[
 (Pa)(y)=\sum_{y'}P(y,y')a(y').
\]
This is the conditional mean of \(a(Y_{k+1})\) given \(Y_k=y\).

For each fixed value of \(z\) consider the following equation
\begin{equation}
 \mathcal U(y,z)-(P\mathcal U)(y,z)=\widetilde G(y,z),
 \qquad
 \sum_y\mu(y)\mathcal U(y,z)=0.
 \label{eq:sup:report-poisson-equation}
\end{equation}
This is the Poisson equation. The second condition chooses one solution by
forcing its stationary mean to be zero.  For a finite-state geometrically
ergodic chain, the solution is
\begin{equation}
 \mathcal U(\,\cdot\,,z)
 =\sum_{j=0}^\infty P^j\widetilde G(\,\cdot\,,z).
 \label{eq:sup:report-poisson-series}
\end{equation}

Define the one-step innovation
\begin{equation}
 \zeta_{k+1}(z)
 =\mathcal U(Y_{k+1},z)-(P\mathcal U)(Y_k,z).
 \label{eq:sup:report-martingale-innovation}
\end{equation}
Let \(\mathcal F_k\) contain the iterate and the complete Markov history up
to time \(k\).  Exogeneity means that, given \(Y_k\), the next Markov state
does not depend on the rest of \(\mathcal F_k\).  Therefore
\begin{equation}
 \mathbb E[\zeta_{k+1}(z)\mid\mathcal F_k]=0.
 \label{eq:sup:report-martingale-difference}
\end{equation}
Indeed, conditional on \(Y_k\), the expected value of
\(\mathcal U(Y_{k+1},z)\) is exactly \((P\mathcal U)(Y_k,z)\).

The Poisson equation can now be rearranged as
\begin{equation}
 \widetilde G(Y_k,z)
 =\mathcal U(Y_k,z)-\mathcal U(Y_{k+1},z)+\zeta_{k+1}(z).
 \label{eq:sup:report-coboundary-decomposition}
\end{equation}
Summing this identity makes the first two terms cancel across time:
\begin{equation}
 \sum_{k=0}^{m-1}\widetilde G(Y_k,z)
 =\mathcal U(Y_0,z)-\mathcal U(Y_m,z)
  +\sum_{k=0}^{m-1}\zeta_{k+1}(z).
 \label{eq:sup:report-telescoped-poisson}
\end{equation}
The endpoint difference remains bounded in mean square. 
The martingale sum has root-mean-square size $O(\sqrt m)$.
Fix \(z=z^\star\) and write
\[
 \begin{gathered}
 M_m=\sum_{k=0}^{m-1}\zeta_{k+1}(z^\star),\qquad
 B_m=\mathcal U(Y_0,z^\star)-\mathcal U(Y_m,z^\star),\\
 Q_\zeta=\mathbb E[\zeta_1(z^\star)\zeta_1(z^\star)^\top].
 \end{gathered}
\]
Thus \(S_m=M_m+B_m\). For \(i<j\), the vector \(\zeta_i(z^\star)\)
is measurable with respect to \(\mathcal F_{j-1}\), so
\[
 \mathbb E[\zeta_i(z^\star)\zeta_j(z^\star)^\top]
 =\mathbb E\!\left[
   \zeta_i(z^\star)
   \mathbb E[\zeta_j(z^\star)^\top\mid\mathcal F_{j-1}]
   \right]=0.
\]
The expectation with the indices reversed is also zero:
\[
 \mathbb E[\zeta_j(z^\star)\zeta_i(z^\star)^\top]
 =\left(
   \mathbb E[\zeta_i(z^\star)\zeta_j(z^\star)^\top]
   \right)^\top
 =0.
\]
Since the Markov input chain is stationary ($Y_0\sim\mu$),
the pair $(Y_k,Y_{k+1})$ has the same distribution for every
$k$. At the fixed state $z^\star$, each innovation is the
same function of this pair. Consequently,
\[
 \mathbb E[
 \zeta_{k+1}(z^\star)\zeta_{k+1}(z^\star)^\top]
 =Q_\zeta
 \qquad\text{for every }k\ge0.
\]
Expanding the square therefore yields
\[
 \begin{gathered}
 \mathbb E[M_mM_m^\top]=mQ_\zeta,\qquad
 \mathbb E\|M_m\|^2=m\operatorname{tr}Q_\zeta,\\
 \|M_m\|_{L^2}=\sqrt{m\operatorname{tr}Q_\zeta},
 \end{gathered}
\]
where \(\|X\|_{L^2}=(\mathbb E\|X\|^2)^{1/2}\). This proves the
\(O(\sqrt m)\) assertion in \(L^2\). Since
\(\mathbb E(a^\top M_m)^2=m\,a^\top Q_\zeta a\), the scale is nondegenerate
in a direction \(a\) if and only if \(a^\top Q_\zeta a>0\).

The correction $B_m$ consists of the values of the Poisson
solution at the initial and final Markov states, $Y_0$ and
$Y_m$. We now show that its contribution to the normalized
covariance $m^{-1}\mathbb E[S_mS_m^\top]$, including its
cross terms with $M_m$, tends to zero as $m\to\infty$.
Put
\(C_U=2\|\mathcal U(Y_0,z^\star)\|_{L^2}<\infty\).  Stationarity and the
triangle inequality give \(\|B_m\|_{L^2}\le C_U\), uniformly in \(m\).
Keeping every cross term gives
\[
 \frac1m\mathbb E[S_mS_m^\top]-Q_\zeta
 =\frac1m\mathbb E[M_mB_m^\top+B_mM_m^\top+B_mB_m^\top].
\]
The inequality \(\|xy^\top\|\le\|x\|\|y\|\) and Cauchy--Schwarz give
\[
 \left\|\frac1m\mathbb E[S_mS_m^\top]-Q_\zeta\right\|
 \le \frac{2C_U\sqrt{\operatorname{tr}Q_\zeta}}{\sqrt m}
       +\frac{C_U^2}{m}\ \longrightarrow\ 0.
\]
We have proved the following.

\begin{lemma}[Poisson representation]\label{lem:sup:poisson-Q}
Under \Cref{as:sup:markov}, the innovations \eqref{eq:sup:report-martingale-innovation}
are martingale differences, and the long-run covariance
\eqref{eq:sup:report-long-run-covariance} satisfies
\begin{equation}
 Q=\mathbb E\bigl[\zeta_1(z^\star)\zeta_1(z^\star)^\top\bigr].
 \label{eq:sup:report-Q-martingale-covariance}
\end{equation}
\end{lemma}

\paragraph{Example: A Symmetric Two-State Chain.}
Let \(Y_k\in\{-1,1\}\) be a stationary symmetric two-state
Markov chain, with
\begin{equation}
 \mathbb E[Y_{k+1}\mid Y_k]=qY_k,
 \qquad -1<q<1.
 \label{eq:sup:report-two-state-correlation}
\end{equation}
For an integer $h\ge0$, the Markov property and the tower
property of conditional expectation give
\[
 \mathbb E[Y_{k+h+1}\mid Y_k]
 =\mathbb E\!\left[
     \mathbb E[Y_{k+h+1}\mid\mathcal F_{k+h}]
     \mid Y_k\right]
 =q\mathbb E[Y_{k+h}\mid Y_k].
\]
Starting at \(h=0\), induction yields
\(\mathbb E[Y_{k+h}\mid Y_k]=q^hY_k\), hence
\[
 \mathbb E[Y_kY_{k+h}]
 =\mathbb E\{Y_k\mathbb E[Y_{k+h}\mid Y_k]\}
 =q^h\mathbb E Y_k^2=q^h.
\]
Since \(\mathbb EY_k=0\), this is also the autocovariance.  For negative
integer lags, stationarity gives \(q^{|h|}\).

Fix \(\omega\in\mathbb R^d\), \(d=d_x+d_\theta\), and take
\(\widetilde G(y,z^\star)=\omega y\).  The frozen noise is \(\xi_k=\omega Y_k\), and the Poisson solution is
\begin{equation}
 \mathcal U(y,z^\star)=\frac{\omega y}{1-q},
 \label{eq:sup:report-two-state-poisson}
\end{equation}
because \((P\mathcal U)(y,z^\star)=q\omega y/(1-q)\).  Hence
\begin{equation}
 \zeta_{k+1}(z^\star)=\frac{\omega(Y_{k+1}-qY_k)}{1-q}.
 \label{eq:sup:report-two-state-martingale}
\end{equation}
Since \(Y_k^2=1\) and \(\mathbb E[Y_kY_{k+1}]=q\), its covariance is
\begin{align}
 \mathbb E[\zeta_{k+1}(z^\star)\zeta_{k+1}(z^\star)^\top]
 &=\frac{1-2q^2+q^2}{(1-q)^2}\omega\omega^\top
 =\frac{1+q}{1-q}\omega\omega^\top.
 \label{eq:sup:report-two-state-Q-martingale}
\end{align}
The autocovariance sum gives the same answer:
\begin{equation}
 Q=\left(1+2\sum_{h=1}^\infty q^h\right)\omega\omega^\top
 =\frac{1+q}{1-q}\omega\omega^\top.
 \label{eq:sup:report-two-state-Q-series}
\end{equation}
Here \(K_0=\omega\omega^\top\) for every \(q\), whereas \(Q\) depends
strongly on \(q\).

\subsection{Taylor Estimates for the Averaged Drift and Poisson Corrector}
\label{app:taylor-estimates}

The stationary covariance and mean calculations require control of
two expansions: the averaged drift around the equilibrium, and the
Poisson corrector across one algorithmic update. We also need to
control the error from evaluating the corrector's response
coefficient at the equilibrium.

Write
\[
 e_k=z_k-z^\star,\qquad
 \Delta z_k=z_{k+1}-z_k=\eta D_\rho G(Y_k,z_k),
\]
and let $\mathcal U$ be the normalized solution of
\eqref{eq:poisson}.

For the averaged drift, $R^{(2)}_{\bar G,k}$ is the error left
after retaining the linear Taylor term, while
$R^{(3)}_{\bar G,k}$ is the error left after also retaining the
quadratic term. Since $\bar G(z^\star)=0$ and
$J=D\bar G(z^\star)$, define
\begin{align*}
 R^{(2)}_{\bar G,k}
 &=\bar G(z_k)-Je_k,\\
 R^{(3)}_{\bar G,k}
 &=\bar G(z_k)-Je_k
   -\tfrac12D^2\bar G(z^\star)[e_k,e_k].
\end{align*}
Here $D^2\bar G(z^\star)[e_k,e_k]$ denotes the vector obtained
by applying the Hessian of each component of $\bar G$ to
$(e_k,e_k)$. Near the equilibrium, these remainders satisfy
\[
 \|R^{(2)}_{\bar G,k}\|=O(\|e_k\|^2),\qquad
 \|R^{(3)}_{\bar G,k}\|=O(\|e_k\|^3).
\]
Thus the superscripts indicate the local remainder orders.

For the Poisson corrector, we expand
$\mathcal U(Y_{k+1},\cdot)$ from $z_k$ to $z_{k+1}$, holding
the Markov argument $Y_{k+1}$ fixed. The error in its linear
Taylor approximation is
\begin{align*}
 R_{\mathcal U,k}
 &=
 \mathcal U(Y_{k+1},z_{k+1})
 -\mathcal U(Y_{k+1},z_k)\\
 &\quad
 -D_z\mathcal U(Y_{k+1},z_k)\Delta z_k.
\end{align*}
This remainder is quadratic in the one-step displacement
$\Delta z_k$ under the bounded second-derivative condition.

The linear corrector increment contains the response coefficient
\[
 D_z\mathcal U(Y_{k+1},z_k)D_\rho G(Y_k,z_k).
\]
To compute its leading contribution, we evaluate both factors
at $z^\star$. Define the resulting coefficient error by
\begin{align*}
 \Delta_{\mathcal U,k}
 &=
 D_z\mathcal U(Y_{k+1},z_k)D_\rho G(Y_k,z_k)\\
 &\quad
 -D_z\mathcal U(Y_{k+1},z^\star)
      D_\rho G(Y_k,z^\star).
\end{align*}
The Markov arguments retain their original values; only the
iterate argument is replaced by the equilibrium. Consequently,
the corrector increment has the exact decomposition
\[
\begin{aligned}
 &\mathcal U(Y_{k+1},z_{k+1})
  -\mathcal U(Y_{k+1},z_k)\\
 &\qquad=
 \eta D_z\mathcal U(Y_{k+1},z^\star)
       D_\rho G(Y_k,z^\star)
 +\eta\Delta_{\mathcal U,k}
 +R_{\mathcal U,k}.
\end{aligned}
\]
This identity separates the equilibrium response, the error from
freezing its coefficient, and the Taylor remainder over one step.

The following lemma derives the expectation bounds below from
the regularity and stationary localization assumptions.
Expectations are taken along a stationary trajectory of the
joint process $(Y_k,z_k)$, and the constant $C$ is independent
of sufficiently small positive $\eta$ and $\rho$:
\begin{align}
 \E\|R^{(2)}_{\bar G,k}\|&\le C\eta,
 &\E\|R^{(3)}_{\bar G,k}\|&\le C\eta^{3/2},\nonumber\\
 \E\|R_{\mathcal U,k}\|&\le C\eta^2,
 &\E\|\Delta_{\mathcal U,k}\|&\le C\sqrt\eta,\nonumber\\
 \E\int_0^1&\|D_z\mathcal U(Y_{k+1},z_k+s\Delta z_k)\|\,
 \|D_\rho G(Y_k,z_k)\|\,ds\le C.\label{eq:taylor-remainder-bounds}
\end{align}
\begin{equation}
 \E\|e\|^j\le C\eta^{j/2},\qquad j=2,3.\label{eq:raw-error-moments}
\end{equation}

\begin{lemma}[Taylor estimates from fourth-moment localization]
\label{lem:derived-taylor-estimates}
Suppose that the state space of the exogenous chain is finite and its
transition kernel satisfies Assumption~\ref{as:markov}. Suppose that
$G(y,\cdot)\in C^2(\mathbb R^{d_x+d_\theta})$ has uniformly bounded
first and second derivatives, that $\bar G$ is $C^3$ on a neighborhood
of $z^\star$, and that a stationary law satisfies
$\E\|e_k\|^4\le C\eta^2$ uniformly for $0<\rho\le\rho_0\le1$.
Then every estimate in \eqref{eq:taylor-remainder-bounds} and
\eqref{eq:raw-error-moments} follows with a constant independent of
$\eta$ and $\rho$. The coordinate-change estimates
\eqref{eq:lambda-remainder-bounds} also follow if the chosen extension
of $\lambda$ has bounded second derivative on the segments joining
$\theta^\star$ to the stationary values of $\theta$.
\end{lemma}

\begin{proof}
All norms below are Euclidean norms or the induced multilinear norms.
For $1\le p\le4$, H\"older's inequality gives
\[
 \E\|e_k\|^p\le (\E\|e_k\|^4)^{p/4}\le C\eta^{p/2}.
\]
The operator $I-P$ is not invertible on the full space of functions. It is invertible on the subspace of $\mu$-centered functions under Assumption~\ref{as:markov}. The normalization $\mu\mathcal U=0$ in \eqref{eq:poisson} selects precisely this inverse. Here are the details.

\textbf{Centering and the mixing bound.}
For a function $a$ on $\mathsf Y$, define
\[
 (\Pi a)(y):=\mu a,\qquad
 \|a\|_\infty:=\max_{y\in\mathsf Y}\|a(y)\|.
\]
Thus $\Pi$ projects onto the constant functions, and $\ker\Pi$ is the centered subspace. Invariance of $\mu$ gives
\[
 \Pi^2=\Pi,\qquad P\Pi=\Pi P=\Pi.
\]
We may decrease the minorization constant, if necessary, so that $0<\beta<1$. The Doeblin condition then gives the decomposition
\[
 P^{n_0}(y,\cdot)
 =\beta\nu(\cdot)+(1-\beta)R(y,\cdot)
\]
for a Markov kernel $R$. 
For probability measures $\alpha,\alpha'$ on $\mathsf Y$, we use
the total variation distance
\[
 \|\alpha-\alpha'\|_{\rm TV}
 :=\sup_{E\subseteq\mathsf Y}|\alpha(E)-\alpha'(E)|
 =\frac12\sum_{y\in\mathsf Y}|\alpha(y)-\alpha'(y)|.
\]
The common $\beta\nu$ term cancels when comparing two probability distributions, and a Markov kernel does not increase total variation distance. Therefore
\[
 \|\alpha P^{n_0}-\alpha'P^{n_0}\|_{\rm TV}
 \le(1-\beta)\|\alpha-\alpha'\|_{\rm TV}.
\]
Iterating this inequality, using $\mu P=\mu$, and treating the remaining fewer than $n_0$ steps by nonexpansiveness yields
\[
 \|P^n(y,\cdot)-\mu\|_{\rm TV}
 \le(1-\beta)^{\lfloor n/n_0\rfloor}.
\]
Consequently, for scalar, vector, or derivative-tensor-valued $a$,
\[
 \|(P^n-\Pi)a\|_\infty
 \le2(1-\beta)^{\lfloor n/n_0\rfloor}\|a\|_\infty.
\]
For vector or tensor values this follows by the triangle inequality applied to the finite signed sum over $\mathsf Y$.

\textbf{The bounded Poisson solution operator.}
Define the fixed linear operator
\[
 \mathscr R_P:=\sum_{n=0}^{\infty}(P^n-\Pi).
\]
The preceding bound gives convergence in operator norm and the explicit estimate
\[
 \|\mathscr R_P\|_{\infty\to\infty}
 \le2\sum_{n=0}^{\infty}(1-\beta)^{\lfloor n/n_0\rfloor}
 =\frac{2n_0}{\beta}
 =:C_P.
\]
For the partial sum through $N$, telescoping gives
\[
 (I-P)\sum_{n=0}^{N}(P^n-\Pi)=I-P^{N+1}.
\]
Taking the limit, and using the centering identities, gives
\[
 (I-P)\mathscr R_P=\mathscr R_P(I-P)=I-\Pi,
 \qquad
 \Pi\mathscr R_P=\mathscr R_P\Pi=0.
\]
Hence, for every centered $h$, $u=\mathscr R_Ph$ satisfies $(I-P)u=h$ and $\mu u=0$. It is unique: if $(I-P)v=0$ and $\mu v=0$, then $v=P^nv$ for every $n$, while the mixing bound forces $\|P^nv\|_\infty\to0$. Equivalently,
\[
 \mathscr R_P=(I-P+\Pi)^{-1}-\Pi.
\]
Indeed, the preceding identities show that $\Pi+\mathscr R_P$ is the two-sided inverse of $I-P+\Pi$. In particular, on a centered right-hand side, $\mathscr R_Ph=(I-P+\Pi)^{-1}h$.

\textbf{Derivatives of the Poisson corrector.}
We now differentiate the Poisson corrector $\mathcal U(y,z)$
with respect to $z$, keeping $y$ fixed. Our goal is to derive
uniform bounds on its first and second derivatives from the
corresponding derivative bounds on the update field $G$.
The equation in \eqref{eq:poisson} is
\[
 (I-P)\mathcal U(\cdot,z)
 =\widetilde G(\cdot,z),\qquad
 \widetilde G(\cdot,z)=(I-\Pi)G(\cdot,z),\qquad
 \mu\mathcal U(\cdot,z)=0.
\]
Thus
\[
 \mathcal U(\cdot,z)=\mathscr R_P\widetilde G(\cdot,z).
\]
Both $P$ and $\mu$ are independent of $z$, so $\mathscr R_P$ is a fixed finite-dimensional matrix. Differentiating its finite matrix-vector product requires no interchange of an infinite series with a derivative. For $j=1,2$,
\[
 D_z^j\mathcal U(\cdot,z)
 =\mathscr R_P\bigl(D_z^jG(\cdot,z)-\mu D_z^jG(\cdot,z)\bigr).
\]
The operator acts on the Markov-state index of these derivative tensors. Assumption~\ref{as:taylor} gives $\sup_{y,z}\|D_z^jG(y,z)\|\le K$, and hence
\[
 \sup_z\|D_z^jG(\cdot,z)-\mu D_z^jG(\cdot,z)\|_\infty
 \le2K,\qquad
 \sup_{y,z}\|D_z^j\mathcal U(y,z)\|\le2C_PK.
\]
This is the precise meaning of commuting with differentiation.

\textbf{Growth bounds for the update field and Poisson corrector.}
We show that the update field $G(y,z)$ and the Poisson
corrector $\mathcal U(y,z)$ grow at most linearly with the
distance $\|z-z^\star\|$, uniformly in the Markov state $y$.
These growth bounds also control the one-step increment
$\Delta z_k=\eta D_\rho G(Y_k,z_k)$.
Set $M_G:=\max_y\|G(y,z^\star)\|<\infty$. Since $\bar G(z^\star)=0$, we have
$\widetilde G(\cdot,z^\star)=G(\cdot,z^\star)$ and therefore
$\max_y\|\mathcal U(y,z^\star)\|\le C_PM_G$.
Writing $e=z-z^\star$, the fundamental theorem of calculus along the segment $z^\star+se$, $0\le s\le1$, gives
\[
 \begin{aligned}
 \|G(y,z)\|&\le M_G+K\|e\|,\\
 \|\mathcal U(y,z)\|&\le C_PM_G+2C_PK\|e\|.
 \end{aligned}
\]
Adding these inequalities proves
\[
 \|G(y,z)\|+\|\mathcal U(y,z)\|
 \le C(1+\|z-z^\star\|).
\]
Finally, $\|D_\rho\|\le1$ for $0<\rho\le1$, so the exact update satisfies
\[
 \|\Delta z_k\|
 =\eta\|D_\rho G(Y_k,z_k)\|
 \le\eta\bigl(M_G+K\|e_k\|\bigr)
 \le C\eta(1+\|e_k\|).
\]
These pointwise bounds hold with constants independent of
$\eta$ and $\rho$. We now combine them with the stationary
fourth-moment bound to derive the expectation estimates
stated above.

Assumption~\ref{as:taylor} bounds the second derivatives of
$G(y,z)$ uniformly over all $y$ and $z$. The averaged field
$\bar G$ inherits this bound, so Taylor's theorem gives
\[
 \|R^{(2)}_{\bar G,k}\|\le C\|e_k\|^2.
\]
Choose $\delta>0$ such that the closed ball of radius $\delta$ about
$z^\star$ lies inside the $C^3$ neighborhood. On this ball,
$\|R^{(3)}_{\bar G,k}\|\le C\|e_k\|^3$; outside it, the global
Hessian bound still gives
$\|R^{(3)}_{\bar G,k}\|\le C\|e_k\|^2$. Consequently,
\[
 \E\|R^{(3)}_{\bar G,k}\|
 \le C\E\|e_k\|^3
       +C\delta^{-2}\E\|e_k\|^4
 \le C\eta^{3/2}.
\]
For each realized value of $Y_{k+1}$, Taylor's formula along the segment from $z_k$ to $z_k+\Delta z_k$ gives the exact remainder
\[
 R_{\mathcal U,k}
 =\int_0^1(1-s)D_z^2\mathcal U
       (Y_{k+1},z_k+s\Delta z_k)
       [\Delta z_k,\Delta z_k]\,ds.
\]
Since the second derivative of $\mathcal U$ is uniformly bounded,
\[
 \|R_{\mathcal U,k}\|\le C\|\Delta z_k\|^2
\]
for every realization.
Consequently, the preceding linear-growth estimate and the fourth-moment assumption imply
\[
 \begin{aligned}
 \E\|\Delta z_k\|^2
 &\le C\eta^2\E(1+\|e_k\|)^2\\
 &\le C\eta^2(1+\E\|e_k\|^2)
 \le C\eta^2(1+\eta)
 \le C\eta^2.
 \end{aligned}
\]
The last inequality uses sufficiently small $\eta$, and all constants are uniform in $\rho$. Because the Taylor bound holds for every realization, dependence between $Y_{k+1}$ and the iterates does not require an independence argument.

For the frozen-coefficient error, add and subtract
$D_z\mathcal U(Y_{k+1},z_k)D_\rho G(Y_k,z^\star)$.
The bounded first derivative of $G$, the bounded first and second
derivatives of $\mathcal U$, and the bounded values
$G(y,z^\star)$ imply
\[
 \|\Delta_{\mathcal U,k}\|\le C\|e_k\|,
 \qquad
 \E\|\Delta_{\mathcal U,k}\|\le C\sqrt\eta.
\]
For every $s\in[0,1]$, the bounded first derivative of
$\mathcal U$, the linear growth of $G$, and
$\|D_\rho\|\le1$ give
\[
 \begin{aligned}
 &\|D_z\mathcal U(Y_{k+1},z_k+s\Delta z_k)\|\,
   \|D_\rho G(Y_k,z_k)\|\\
 &\qquad\le C(1+\|e_k\|).
 \end{aligned}
\]
The derivative bound applies throughout the update segment, so no bounded-support assumption on that segment is needed. Integrating over an interval of length one and taking expectations yields
\[
 \begin{aligned}
 &\E\int_0^1
 \|D_z\mathcal U(Y_{k+1},z_k+s\Delta z_k)\|
 \|D_\rho G(Y_k,z_k)\|\,ds\\
 &\qquad\le C(1+\E\|e_k\|)\le C.
 \end{aligned}
\]
By H\"older's inequality and the stationary fourth-moment
bound \eqref{eq:cov-raw-fourth},
\[
 \E\|e_k\|
 \le \bigl(\E\|e_k\|^4\bigr)^{1/4}
 \le C\sqrt\eta.
\]
Finally, suppose that the second derivative of the chosen
extension of $\lambda$ is uniformly bounded along the segments
joining $\theta^\star$ to stationary values of $\theta$, as
specified in the last part of the lemma. Taylor's theorem gives
\[
 \|r_\lambda(v_k)\|\le C\|v_k\|^2.
\]
Since $v_k$ is a block of $e_k$, H\"older's inequality and
the fourth-moment bound give
\[
 \begin{aligned}
 \E\|r_\lambda(v_k)\|
 &\le C\E\|v_k\|^2
 \le C\bigl(\E\|e_k\|^4\bigr)^{1/2}
 \le C\eta,\\
 \E\|r_\lambda(v_k)\|^2
 &\le C\E\|v_k\|^4
 \le C\E\|e_k\|^4
 \le C\eta^2.
 \end{aligned}
\]
\end{proof}

\section{Uniform stationary covariance and mean bias}
\label{sec:sup:report-markov-covariance}

This section derives the stationary covariance and mean-bias
expansions for the nonlinear recursion.
Subsection~\ref{subsec:sup:lyapunov} studies the linearized,
Poisson-corrected recursion and identifies the leading covariance
coefficient through a Lyapunov equation.
Subsection~\ref{app:covariance-proof} then proves that this
coefficient approximates the stationary covariance of the
original nonlinear recursion, with an error bound that remains
uniform as $\rho\downarrow0$, establishing
Proposition~\ref{prop:covariance}.
Finally, Subsection~\ref{sec:sup:bias-expansion} first establishes
the a priori $O(\eta)$ mean bound, then combines the covariance
approximation with the stationary drift balance and
Poisson-corrector identity to obtain the mean-bias expansion
in Theorem~\ref{thm:fixed-ratio}.

\subsection{Linear Covariance Balance and Its Small-Step Limit}
\label{subsec:sup:lyapunov}
Our goal is to identify the leading stationary covariance
coefficient from the linear approximation near equilibrium.
We first apply a Poisson correction to express the Markov
noise as a martingale innovation plus a remainder.
Linearizing the drift and evaluating the innovation at
equilibrium then gives a recursion whose stationary covariance
balances linear contraction against noise input. The innovation
has zero conditional mean, so its cross terms with the current
state vanish in this balance. Rescaling the covariance by
$\eta$ and taking $\eta\downarrow0$ at fixed $\rho>0$ yields
a Lyapunov equation for the leading coefficient.
Subsection~\ref{app:covariance-proof} proves that the stationary
covariance of the nonlinear recursion is
$\eta\Sigma_\rho+O(\eta^{3/2})$, with an error constant
independent of $\rho$ for sufficiently small positive
$\eta$ and $\rho$.

Let
\begin{equation}
 J=D\bar G(z^\star)
 =\begin{pmatrix}A&B\\C&D\end{pmatrix},
 \qquad
 J_\rho=D_\rho J
 =\begin{pmatrix}A&B\\\rho C&\rho D\end{pmatrix}.
 \label{eq:sup:report-Jrho}
\end{equation}
Recall that \(e_k=z_k-z^\star\).
Using the Poisson decomposition, we define the corrected error
\begin{equation}
 \widehat e_k=e_k+\eta D_\rho\mathcal U(Y_k,z_k).
 \label{eq:sup:report-corrected-error}
\end{equation}
Substitute the decomposition
\eqref{eq:sup:report-coboundary-decomposition} into the recursion and use the
definition of \(\widehat e_{k+1}\).  Before making any approximation, this
gives
\begin{align}
 \widehat e_{k+1}
 ={}&\widehat e_k+\eta D_\rho\bar G(z_k)
       +\eta D_\rho\zeta_{k+1}(z_k)\nonumber\\
 &+\eta D_\rho\bigl\{\mathcal U(Y_{k+1},z_{k+1})
                      -\mathcal U(Y_{k+1},z_k)\bigr\}.
 \label{eq:sup:report-exact-corrected-recursion}
\end{align}
By the recursion~\eqref{eq:rec} and the growth bound
in~\eqref{eq:cov-poisson-bounds},
\[
 \|z_{k+1}-z_k\|
 \le C\eta(1+\|e_k\|).
\]
The fourth-moment localization~\eqref{eq:cov-raw-fourth}
therefore gives
$\|z_{k+1}-z_k\|_{L^2}\le C\eta$.
Using the Lipschitz bound for $\mathcal U$
in~\eqref{eq:cov-poisson-bounds}, we obtain
\[
 \left\|
 \eta D_\rho
 \bigl\{\mathcal U(Y_{k+1},z_{k+1})
       -\mathcal U(Y_{k+1},z_k)\bigr\}
 \right\|_{L^2}
 \le C\eta\|z_{k+1}-z_k\|_{L^2}
 \le C\eta^2.
\]
Thus the last line
of~\eqref{eq:sup:report-exact-corrected-recursion}
is $O_{L^2}(\eta^2)$, uniformly in sufficiently small
positive $\eta$ and $\rho$. Linearize
\(\bar G\), replace \(e_k\) by \(\widehat e_k\) at leading order, and freeze
the innovation at \(z^\star\).  The leading recursion is then
\begin{equation}
 \widehat e_{k+1}=(I+\eta J_\rho)\widehat e_k
       +\eta D_\rho\zeta_{k+1}(z^\star).
 \label{eq:sup:report-linear-corrected-recursion}
\end{equation}
This equation defines the auxiliary linear recursion
studied in Lemma~\ref{lem:sup:linear-covariance-limit}.
For fixed sufficiently small $\rho>0$, the lemma shows that
its stationary covariance, divided by $\eta$, converges as
$\eta\downarrow0$ to $\Sigma_\rho$, the unique solution of
the associated Lyapunov equation. This identifies the leading
covariance coefficient for the linear model.

In forming this model, we have omitted the nonlinear drift
remainder, the drift correction arising from replacing $e_k$
by $\widehat e_k$, the change of the Poisson corrector over
one update, and the variation of the innovation away from
equilibrium. Subsection~\ref{app:covariance-proof} controls
these contributions to establish the covariance approximation
for the original nonlinear recursion uniformly as
$\rho\downarrow0$.

\begin{lemma}[Small-step covariance of the linear recursion]
\label{lem:sup:linear-covariance-limit}
Fix a sufficiently small $\rho>0$ such that $J_\rho$ is Hurwitz.
For sufficiently small $\eta>0$, let
$\widehat\Gamma^{\mathrm{lin}}_{\eta,\rho}$ be the stationary
covariance of the linear recursion
\eqref{eq:sup:report-linear-corrected-recursion}. Then
\[
 \lim_{\eta\downarrow0}
 \frac{\widehat\Gamma^{\mathrm{lin}}_{\eta,\rho}}{\eta}
 =\Sigma_\rho,
\]
where $\Sigma_\rho$ is the unique solution of
\begin{equation}
 J_\rho\Sigma_\rho+\Sigma_\rho J_\rho^\top
       +D_\rho QD_\rho=0.
 \label{eq:sup:report-stationary-lyapunov}
\end{equation}
\end{lemma}

\begin{proof}
Write $A_{\eta,\rho}=I+\eta J_\rho$ and
$\mathcal Q_\rho=D_\rho QD_\rho$. Since $J_\rho$ is Hurwitz,
$A_{\eta,\rho}$ is stable for all sufficiently small $\eta$.
The frozen innovations are bounded because the Markov state
space is finite. With the Markov chain in stationarity,
stability therefore gives a stationary solution determined
by past innovations and having finite covariance.
By Lemma~\ref{lem:Q}, the covariance of
$\zeta_{k+1}(z^\star)$ is $Q$.
The current state of the linear recursion is measurable with
respect to the past, while $\zeta_{k+1}(z^\star)$ has conditional
mean zero. Thus their cross-covariance is zero, and stationarity
gives the exact balance
\[
 \widehat\Gamma^{\mathrm{lin}}_{\eta,\rho}
 =A_{\eta,\rho}\widehat\Gamma^{\mathrm{lin}}_{\eta,\rho}
       A_{\eta,\rho}^\top+\eta^2\mathcal Q_\rho.
\]

Set
$\Sigma_{\eta,\rho}
 =\widehat\Gamma^{\mathrm{lin}}_{\eta,\rho}/\eta$.
Expanding the balance and dividing by $\eta^2$ yields
\[
 J_\rho\Sigma_{\eta,\rho}
 +\Sigma_{\eta,\rho}J_\rho^\top
 +\eta J_\rho\Sigma_{\eta,\rho}J_\rho^\top
 +\mathcal Q_\rho=0.
\]

Define the linear map
$\mathcal L_\rho(X)=J_\rho X+XJ_\rho^\top$.
It is invertible: if $\mathcal L_\rho(X)=0$, then
$e^{tJ_\rho}Xe^{tJ_\rho^\top}$ is constant in $t$ and tends to
zero as $t\to\infty$, so $X=0$. Since this is a linear map on a
finite-dimensional space, injectivity implies invertibility.
Consequently, for sufficiently small $\eta$, the map
\[
 X\longmapsto
 \mathcal L_\rho(X)+\eta J_\rho XJ_\rho^\top
\]
is also invertible, and its inverse converges to
$\mathcal L_\rho^{-1}$. Hence $\Sigma_{\eta,\rho}$ converges to
the unique solution of
\eqref{eq:sup:report-stationary-lyapunov}.
\end{proof}

\subsection{Proof of Proposition~\ref{prop:covariance}: Uniform Covariance Approximation}
\label{app:covariance-proof}

Under the assumptions of Proposition~\ref{prop:covariance},
we prove the following estimate for the stationary covariance
$\Gamma_{\eta,\rho}$ of the process generated by the nonlinear
recursion~\eqref{eq:rec}:
\begin{equation}
 \Gamma_{\eta,\rho}
 =\eta\Sigma_\rho+\mathcal R^\Gamma_{\eta,\rho},
 \qquad
 \|\mathcal R^\Gamma_{\eta,\rho}\|\le C\eta^{3/2}.
 \label{eq:sup:report-actual-covariance-expansion}
\end{equation}
All constants below are independent of
$0<\eta\le\eta_0$ and $0<\rho\le\rho_0\le1$.
The sharper slow and cross-moment bounds are not needed for
this absolute covariance approximation.

The covariance error will satisfy a Lyapunov equation whose
right-hand side is the remainder in the stationary covariance
balance. The next lemma bounds the solution in terms of the
blocks of this remainder. The slow block enters the bound
with a factor $1/\rho$. In the proof of
Proposition~\ref{prop:covariance}, this factor is compensated
by the $O(\rho\eta^{3/2})$ bound on the slow remainder,
yielding a covariance error uniform as $\rho\downarrow0$.
\begin{lemma}[Lyapunov inverse with scaled slow forcing]
\label{lem:cov-lyapunov-inverse}
Let
\[
 L_\rho=
 \begin{pmatrix}
 A-\rho\Lambda C&-\rho\Lambda S\\
 \rho C&\rho S
 \end{pmatrix},
\]
where $A$ and $S$ are Hurwitz.
For any symmetric matrix $F^{\mathrm{force}}$, write
\[
 F^{\mathrm{force}}
 =
 \begin{pmatrix}
 F^{\mathrm{force}}_{uu} & F^{\mathrm{force}}_{uv}\\
 (F^{\mathrm{force}}_{uv})^\top & F^{\mathrm{force}}_{vv}
 \end{pmatrix},
\]
using the same $(u,v)$ block partition as $L_\rho$.

After reducing $\rho_0>0$ if necessary, the symmetric solution
$X$ of
\[
 L_\rho X+XL_\rho^\top=-F^{\mathrm{force}}
\]
satisfies
\begin{equation}
 \|X\|\le C_*
 \left(
 \|F^{\mathrm{force}}_{uu}\|
 +\|F^{\mathrm{force}}_{uv}\|
 +\rho^{-1}\|F^{\mathrm{force}}_{vv}\|
 \right),
 \qquad 0<\rho\le\rho_0.
 \label{eq:cov-lyapunov-inverse}
\end{equation}
\end{lemma}

\begin{proof}
Write
\[
 X=
 \begin{pmatrix}
 U&R\\
 R^\top&V
 \end{pmatrix}
\]
and divide the $vv$ block equation by $\rho$.
The resulting linear map acting on $(U,R,V)$ has coefficients
continuous at $\rho=0$. At that point it is
\[
 (U,R,V)\longmapsto
 \left(
 AU+UA^\top,\;
 AR,\;
 CR+R^\top C^\top+SV+VS^\top
 \right).
\]
This map is block triangular. Its diagonal operators are
the Lyapunov map associated with $A$, left multiplication
by $A$, and the Lyapunov map associated with $S$.
Each is invertible because $A,S$ are Hurwitz.

Continuity of inversion in finite dimensions therefore gives
a uniformly bounded inverse for sufficiently small $\rho$.
Its right-hand side is
$-(F^{\mathrm{force}}_{uu},
F^{\mathrm{force}}_{uv},
F^{\mathrm{force}}_{vv}/\rho)$, proving
\eqref{eq:cov-lyapunov-inverse}.
Stability of $L_\rho$ also gives uniqueness for positive
small $\rho$.
\end{proof}

\begin{proof}[Proof of Proposition~\ref{prop:covariance}]
All expectations are taken under the joint invariant law.
Throughout the proof,
$\|X\|_2=(\E\|X\|^2)^{1/2}$ for a random vector $X$,
and $C$ denotes a constant independent of sufficiently small
$\eta$ and $\rho$.

We first obtain an exact recursion whose leading noise is a
martingale innovation. We then derive its stationary covariance
balance and estimate the remainder in that balance. The decisive
estimate is that the slow--slow remainder carries a factor $\rho$:
this compensates for the factor $1/\rho$ in
Lemma~\ref{lem:cov-lyapunov-inverse}. Finally, we transfer the
result back to the original coordinates.

\medskip
\noindent\textbf{Step 1: Express the recursion in corrected tangent coordinates.}
The tangent transformation exposes the factors of $\rho$ in the
slow dynamics, while the Poisson correction isolates noise with
zero conditional mean. Define
\[
 T=\begin{pmatrix}I&-\Lambda\\0&I\end{pmatrix},
 \qquad
 B_\rho=TD_\rho
 =\begin{pmatrix}I&-\rho\Lambda\\0&\rho I\end{pmatrix},
 \qquad
 L_\rho=TD_\rho JT^{-1}.
\]
For the fixed finite-state Doeblin kernel, the Poisson inverse
is bounded on the space of $\mu$-centered functions. Applying
this inverse to the first and second derivatives of
$\widetilde G=G-\bar G$ shows that the Poisson corrector
$\mathcal U$ has bounded first and second derivatives,
uniformly in the Markov state and $z$. Its values at $z^\star$
are bounded because the state space is finite. Together with
the derivative bounds on $G$, this gives
\begin{align}
 \|G(y,z)\|+\|\mathcal U(y,z)\|
 &\le C(1+\|z-z^\star\|),
 \nonumber\\
 \|\mathcal U(y,z')-\mathcal U(y,z)\|
 &\le C\|z'-z\|,
 \nonumber\\
 \|\zeta_{k+1}(z)-\zeta_{k+1}(z^\star)\|
 &\le C\|z-z^\star\|.
 \label{eq:cov-poisson-bounds}
\end{align}
For the adapted iterate, the innovation satisfies
$\E[\zeta_{k+1}(z_k)\mid\mathcal F_k]=0$; this uses no
independence between $Y_k$ and $z_k$.

Set
\[
 w_k=Te_k,
 \qquad
 \widehat w_k=w_k+\eta B_\rho\mathcal U(Y_k,z_k),
 \qquad
 R^{(2)}_{\bar G,k}=\bar G(z_k)-Je_k.
\]
Applying $T$ to the exact Poisson-corrected recursion
\eqref{eq:sup:report-exact-corrected-recursion} gives
\begin{equation}
 \widehat w_{k+1}
 =\widehat w_k
 +\eta\{a_k+n_{k+1}+\mathfrak r_{k+1}\},
 \qquad
 a_k=L_\rho\widehat w_k,
 \qquad
 n_{k+1}=B_\rho\zeta_{k+1}(z_k),
 \label{eq:cov-corrected-recursion}
\end{equation}
where
\begin{align}
 \mathfrak r_{k+1}
 ={}& B_\rho R^{(2)}_{\bar G,k}
 -\eta L_\rho B_\rho\mathcal U(Y_k,z_k)
 \nonumber\\
 &+B_\rho
 \{\mathcal U(Y_{k+1},z_{k+1})
       -\mathcal U(Y_{k+1},z_k)\}.
 \label{eq:cov-exact-remainder}
\end{align}
The three terms account for the nonlinear drift, the correction
to the linear drift, and the change of the Poisson corrector
over one update. In particular, $\mathfrak r_{k+1}$ can depend
on $Y_{k+1}$ and need not be $\mathcal F_k$-measurable.

\medskip
\noindent\textbf{Step 2: Derive the stationary covariance balance.}
The corrected process is stationary because $\widehat w_k$
is a fixed function of the stationary pair $(Y_k,z_k)$.
Let
\[
 q_k=\widehat w_k-\E\widehat w_k,
 \qquad
 \widehat P=\E[q_kq_k^\top].
\]
Stationarity makes the mean increment in
\eqref{eq:cov-corrected-recursion} zero. Moreover, $q_k$ and
$a_k$ are $\mathcal F_k$-measurable, whereas $n_{k+1}$ has
conditional mean zero. Hence
\[
 \E[q_kn_{k+1}^\top]=0,
 \qquad
 \E[a_kn_{k+1}^\top]=0.
\]
To identify the leading noise covariance, recall that the
$Y$-marginal of the joint invariant law is $\mu$. Thus
$(Y_k,Y_{k+1})$ has law $\mu(dy)P(y,dy')$, and
\[
 \E[\zeta_{k+1}(z^\star)\zeta_{k+1}(z^\star)^\top]=Q.
\]
Define the error from evaluating the innovation at $z_k$
instead of $z^\star$ by
\[
 \Delta Q_{\eta,\rho}
 =\E[\zeta_{k+1}(z_k)\zeta_{k+1}(z_k)^\top]-Q.
\]
Expanding $\E[q_{k+1}q_{k+1}^\top-q_kq_k^\top]=0$
and dividing by $\eta$ now yields the exact equation
\begin{equation}
 L_\rho\widehat P+\widehat P L_\rho^\top
 +\eta B_\rho Q B_\rho^\top
 +\mathcal R_{\eta,\rho}=0,
 \label{eq:cov-stationary-balance}
\end{equation}
where, suppressing time indices on $a,n,\mathfrak r$,
\begin{align}
 \mathcal R_{\eta,\rho}
 ={}&
 \E[q_k\mathfrak r^\top+\mathfrak r q_k^\top]
 +\eta B_\rho\Delta Q_{\eta,\rho}B_\rho^\top
 \nonumber\\
 &+\eta\E[
 (a+\mathfrak r)(a+\mathfrak r)^\top
 +n\mathfrak r^\top+\mathfrak r n^\top
 ].
 \label{eq:cov-balance-remainder}
\end{align}
The three groups are the state--remainder cross terms, the
error in the innovation covariance, and the remaining products
of increments. The terms $n\mathfrak r^\top$ and
$\mathfrak r n^\top$ must be retained because the remainder
need not be predictable. It remains to bound these three
groups with their slow-coordinate factors intact.

\medskip
\noindent\textbf{Step 3: Bound the remainder and retain its slow factor.}
The raw fourth-moment assumption gives
\begin{equation}
 \|e_k\|_2\le C\sqrt\eta,
 \qquad
 \bigl\|\|e_k\|^2\bigr\|_2\le C\eta.
 \label{eq:cov-error-moments}
\end{equation}
Bounded second derivatives of $\bar G$ and
$\bar G(z^\star)=0$ imply
$\|R^{(2)}_{\bar G,k}\|\le C\|e_k\|^2$.
Also, the update and \eqref{eq:cov-poisson-bounds} give
\[
 \|\mathcal U(Y_{k+1},z_{k+1})
       -\mathcal U(Y_{k+1},z_k)\|
 \le C\eta(1+\|e_k\|).
\]
The upper block rows of $B_\rho$, $L_\rho$, and
$L_\rho B_\rho$ are uniformly bounded, and their lower
block rows are $O(\rho)$. Applying these facts to
\eqref{eq:cov-corrected-recursion} and
\eqref{eq:cov-exact-remainder} yields
\begin{align}
 \|\widehat w_k\|_2&\le C\sqrt\eta,
 \nonumber\\
 \|\mathfrak r_{k+1}^u\|_2&\le C\eta,
 &
 \|\mathfrak r_{k+1}^v\|_2&\le C\rho\eta,
 \nonumber\\
 \|a_k^u\|_2&\le C\sqrt\eta,
 &
 \|a_k^v\|_2&\le C\rho\sqrt\eta,
 \nonumber\\
 \|n_{k+1}^u\|_2&\le C,
 &
 \|n_{k+1}^v\|_2&\le C\rho.
 \label{eq:cov-block-L2}
\end{align}
In particular,
$\|q_k^u\|_2,\|q_k^v\|_2\le C\sqrt\eta$.
For the innovation covariance, the Lipschitz bound in
\eqref{eq:cov-poisson-bounds} and boundedness of the frozen
innovation give
\begin{equation}
 \|\Delta Q_{\eta,\rho}\|
 \le C\E(\|e_k\|+\|e_k\|^2)
 \le C\sqrt\eta.
 \label{eq:cov-noise-freezing}
\end{equation}

We can now estimate each group in
\eqref{eq:cov-balance-remainder} using
$\|\E[XY^\top]\|\le\|X\|_2\|Y\|_2$.
The first group has $uu$ and $uv$ blocks of order
$\eta^{3/2}$. Its $vv$ block is of order
$\rho\eta^{3/2}$ because each product contains
$\mathfrak r^v$. The second group has blocks of orders
$\eta^{3/2}$, $\rho\eta^{3/2}$, and
$\rho^2\eta^{3/2}$, respectively. The last group has blocks
of orders $\eta^2$, $\rho\eta^2$, and $\rho^2\eta^2$.
Consequently, for $\eta,\rho\le1$,
\begin{equation}
 \begin{aligned}
 \|(\mathcal R_{\eta,\rho})_{uu}\|
 +\|(\mathcal R_{\eta,\rho})_{uv}\|
 &\le C\eta^{3/2},
 \\
 \|(\mathcal R_{\eta,\rho})_{vv}\|
 &\le C\rho\eta^{3/2}.
 \end{aligned}
 \label{eq:cov-residual-blocks}
\end{equation}
The slow factor comes from the update coefficients.
Only the total raw-error moments in
\eqref{eq:cov-error-moments} were used; no sharper
slow-coordinate or cross-moment bound is needed.

\medskip
\noindent\textbf{Step 4: Solve the covariance equation and return to the original variables.}
The leading covariance in tangent coordinates,
$\widetilde\Sigma_\rho=T\Sigma_\rho T^\top$, satisfies
\[
 L_\rho\widetilde\Sigma_\rho
 +\widetilde\Sigma_\rho L_\rho^\top
 +B_\rho Q B_\rho^\top=0.
\]
Subtracting $\eta$ times this equation from
\eqref{eq:cov-stationary-balance} shows that
$\widehat P-\eta\widetilde\Sigma_\rho$ solves a Lyapunov
equation with forcing $\mathcal R_{\eta,\rho}$.
Lemma~\ref{lem:cov-lyapunov-inverse} and
\eqref{eq:cov-residual-blocks} therefore give
\begin{equation}
 \begin{aligned}
 \|\widehat P-\eta\widetilde\Sigma_\rho\|
 &\le C_*\bigl(
 \|(\mathcal R_{\eta,\rho})_{uu}\|
 +\|(\mathcal R_{\eta,\rho})_{uv}\|
 +\rho^{-1}\|(\mathcal R_{\eta,\rho})_{vv}\|
 \bigr)
 \\
 &\le C\eta^{3/2}.
 \end{aligned}
 \label{eq:cov-corrected-approximation}
\end{equation}
This is where the retained factor $\rho$ gives a bound
uniform as $\rho\downarrow0$.

Finally, the Poisson correction changes the covariance by
at most the same order. Indeed, write
\[
 h_k=B_\rho\mathcal U(Y_k,z_k),
 \qquad
 \widehat w_k=w_k+\eta h_k.
\]
The preceding bounds give
$\|w_k\|_2\le C\sqrt\eta$ and $\|h_k\|_2\le C$.
Expanding the centered covariances yields
\[
 \begin{aligned}
 &\|\operatorname{Cov}(\widehat w_k)
       -\operatorname{Cov}(w_k)\|
 \\
 &\qquad\le
 2\eta\|w_k\|_2\|h_k\|_2+\eta^2\|h_k\|_2^2
 \le C\eta^{3/2}.
 \end{aligned}
\]
Since $\operatorname{Cov}(w_k)=T\Gamma_{\eta,\rho}T^\top$
and $T,T^{-1}$ are fixed, combining this estimate with
\eqref{eq:cov-corrected-approximation} gives
\[
 \|\Gamma_{\eta,\rho}-\eta\Sigma_\rho\|
 \le C\eta^{3/2}.
\]
This proves \eqref{eq:sup:report-actual-covariance-expansion}
and Proposition~\ref{prop:covariance}. The argument uses
neither a stationary mean-bias expansion nor an assumed
covariance approximation.
\end{proof}

Theorem~\ref{thm:iid} sharpens the remainder to $O(\eta^2)$
under its additional additive-i.i.d.\ hypotheses.
An absolute $O(\eta^{3/2})$ covariance remainder does not establish
the sharper $O(\varepsilon)$ slow and cross-covariance bounds on all
step-size paths. Those bounds are separate conditions in
Assumption~\ref{as:block-localization} for the supplementary
fluctuation statements.

\subsection{Proof of Lemma~\ref{lem:mean-O-eta} and
Theorem~\ref{thm:fixed-ratio}}
\label{subsec:fixed-ratio-proof}

\begin{proof}[Proof of the lemma and theorem]
Stationarity and invertibility of $D_\rho$ imply $\E G(Y_k,z_k)=0$. Exogeneity and the Poisson equation give the exact moving-corrector identity
\[
 \E\widetilde G(Y_k,z_k)=\E[\mathcal U(Y_{k+1},z_{k+1})
 -\mathcal U(Y_{k+1},z_k)].
\]
The derivative and moment bounds first bound this expression by $C\eta$. The first-order drift expansion then gives $Jm_{\eta,\rho}=O(\eta)$, proving the lemma because $\det J=\det A\det S\ne0$. This inverts the fixed matrix $J$, avoiding the $1/\rho$ loss from $J_\rho^{-1}$.

Taylor expansion of the corrector, followed by freezing its coefficient at $z^\star$, gives
\[
 \E\widetilde G(Y_k,z_k)=\eta r_\rho+O(\eta^{3/2}).
\]
By Proposition~\ref{prop:covariance} and the mean bound,
\begin{equation}
 \E\bar G(z_k)=Jm_{\eta,\rho}+\eta\mathcal C(\Sigma_\rho)
 +O(\eta^{3/2}).\label{eq:report-drift-taylor-final}
\end{equation}
The mean-square contribution is $O(\eta^2)$. Combining the two expressions with the zero expected drift proves the theorem. Appendix~\ref{sec:sup:bias-expansion} supplies the full identities and remainder bounds.
\end{proof}

\subsection{Stationary Mean Bias}
\label{sec:sup:bias-expansion}
We now use the covariance approximation from
Proposition~\ref{prop:covariance} to determine the stationary
mean displacement $m_{\eta,\rho}$, proving
Theorem~\ref{thm:fixed-ratio}.
The proof starts from the fact that the expected update
vanishes under the invariant law. The Poisson equation
relates this stationary balance to the dependence between
the iterate and the Markov input. A first-order Taylor
expansion then gives the preliminary bound
$m_{\eta,\rho}=O(\eta)$.
With this bound in hand, we expand the averaged drift to
second order, so that its expectation is expressed through
the mean displacement and the stationary covariance.
Substituting the covariance approximation and combining
the result with the Poisson identity yields the bias
expansion, including both the curvature contribution and
the contribution from Markov dependence.

Recall from Appendix~\ref{sec:sup:report-markov-covariance} that
\(e_k=z_k-z^\star\). Expectations without a subscript are under the
joint invariant law \(\pi_{\eta,\rho}\) of \((Y_k,z_k)\). Set
\[
 m_{\eta,\rho}=\mathbb E e_k,\qquad
 \Gamma_{\eta,\rho}=\operatorname{Cov}(e_k).
\]
The \(Y\)-marginal of \(\pi_{\eta,\rho}\) is \(\mu\); denote its
\(z\)-marginal, obtained by integrating out \(Y_k\), by
\(\pi^z_{\eta,\rho}\).

Recall that \(\bar G\) and \(\widetilde G\) are defined in
\eqref{eq:sup:report-centered-field}. Although
\(\widetilde G(\cdot,z)\) has zero \(\mu\)-mean for each fixed \(z\),
\(\mathbb E\widetilde G(Y_k,z_k)\) need not vanish under the joint
invariant law because \(Y_k\) and \(z_k\) may be dependent. Equivalently,
\(\bar G(z_k)\) is generally not
\(\mathbb E[G(Y_k,z_k)\mid z_k]\). The Jacobian
\(J=D\bar G(z^\star)\), defined in \eqref{eq:sup:report-Jrho}, is
nonsingular because \(\det J=\det A\,\det S\ne0\).

\begin{appendixbiasrestatement}[Fixed-ratio bias; restatement]
Under \Cref{as:sup:smooth,as:sup:hurwitz,as:sup:markov,as:sup:localization,as:sup:taylor}, for
\(0<\rho\le\rho_0\),
\begin{equation}
 \boxed{
 m_{\eta,\rho}=\eta b(\rho)+O(\eta^{3/2}),
 \qquad
 b(\rho)=-J^{-1}\{\mathcal C(\Sigma_\rho)+r_\rho\}.}
 \label{eq:sup:report-fixed-ratio-bias}
\end{equation}
with a remainder uniform in \(\rho\).
\end{appendixbiasrestatement}

\paragraph{Proof of Theorem~\ref{thm:fixed-ratio}.}

\begin{proof}
For \(M\in\mathbb R^{(d_x+d_\theta)\times(d_x+d_\theta)}\), define
\[
 \mathcal C(M)_i
 =\frac12\operatorname{Tr}\!\left(
   D^2\bar G_i(z^\star)M\right),
 \qquad i=1,\ldots,d_x+d_\theta.
\]
Define the direct Markov response by
\begin{equation}
 r_\rho
 =\sum_{y,y'\in\mathsf Y}\mu(y)P(y,y')
   D_z\mathcal U(y',z^\star)D_\rho G(y,z^\star).
 \label{eq:sup:report-markov-response}
\end{equation}
The weights \(\mu(y)P(y,y')\) give the stationary one-step transition
law, and \(\mathcal U\) is the Poisson corrector from
\eqref{eq:sup:report-poisson-equation}.

Stationarity of \(z_{k+1}=z_k+\eta D_\rho G(Y_k,z_k)\), and invertibility
of \(D_\rho\) for \(\rho>0\), give
\begin{equation}
 \mathbb E G(Y_k,z_k)=0,\qquad
 \mathbb E\bar G(z_k)=-\mathbb E\widetilde G(Y_k,z_k).
 \label{eq:sup:report-stationary-drift-zero}
\end{equation}

The Poisson equation, exogeneity, and stationarity yield
\begin{align}
 \mathbb E\widetilde G(Y_k,z_k)
 &=\mathbb E\{\mathcal U(Y_k,z_k)
                   -\mathcal U(Y_{k+1},z_k)\}\nonumber\\
 &=\mathbb E\{\mathcal U(Y_{k+1},z_{k+1})
                   -\mathcal U(Y_{k+1},z_k)\}.
 \label{eq:sup:report-moving-corrector-identity}
\end{align}
Taylor expansion across the update \(z_{k+1}-z_k\), followed by the
remainder bounds in \eqref{eq:taylor-remainder-bounds}, gives
\begin{equation}
 \mathbb E\widetilde G(Y_k,z_k)
 =\eta r_\rho+O(\eta^{3/2}).
 \label{eq:sup:report-mean-markov-response}
\end{equation}
Consequently, \eqref{eq:sup:report-stationary-drift-zero} implies
\begin{equation}
 \mathbb E\bar G(z_k)
 =-\eta r_\rho+O(\eta^{3/2}).
 \label{eq:sup:report-averaged-drift-response}
\end{equation}

The localization moment bounds give
\(\mathbb E\|e_k\|^2=O(\eta)\). The moving-corrector identity and the
mean-value bound in \eqref{eq:taylor-remainder-bounds} give
\(\mathbb E\widetilde G(Y_k,z_k)=O(\eta)\). Also,
\(\mathbb E\|R^{(2)}_{\bar G,k}\|=O(\eta)\). Taking expectations in
\[
 \bar G(z_k)=Je_k+R^{(2)}_{\bar G,k}
\]
and using \eqref{eq:sup:report-stationary-drift-zero} therefore yields
\[
 Jm_{\eta,\rho}
 =-\mathbb E\widetilde G(Y_k,z_k)
  -\mathbb E R^{(2)}_{\bar G,k}=O(\eta).
\]
Since \(J\) is nonsingular, \(m_{\eta,\rho}=O(\eta)\).

By Proposition~\ref{prop:covariance},
\begin{equation}
 \Gamma_{\eta,\rho}
 =\eta\Sigma_\rho+O(\eta^{3/2}),
 \label{eq:sup:report-covariance-expansion}
\end{equation}
where \(\Sigma_\rho\) solves \eqref{eq:sup:report-stationary-lyapunov}.
Taylor expansion of the averaged drift gives
\[
 \bar G(z_k)
 =Je_k+\tfrac12D^2\bar G(z^\star)[e_k,e_k]
   +R^{(3)}_{\bar G,k}.
\]
Using
\(\mathbb E[e_ke_k^\top]
 =\Gamma_{\eta,\rho}+m_{\eta,\rho}m_{\eta,\rho}^\top\),
the covariance expansion, \(m_{\eta,\rho}=O(\eta)\), and the
third-order remainder bound yields
\begin{equation}
 \mathbb E\bar G(z_k)
 =Jm_{\eta,\rho}
  +\eta\mathcal C(\Sigma_\rho)+O(\eta^{3/2}).
 \label{eq:sup:report-drift-taylor-final}
\end{equation}

Equating \eqref{eq:sup:report-averaged-drift-response} and
\eqref{eq:sup:report-drift-taylor-final}, then multiplying by
\(J^{-1}\), proves the stated expansion.

All remainder bounds above are uniform for
$0<\rho\le\rho_0\le1$: the Taylor estimates and
Proposition~\ref{prop:covariance} are uniform,
$\|D_\rho\|\le1$, and both the Poisson solution operator
and $J^{-1}$ are independent of $\rho$.
The inverse of $D_\rho$ is used only to deduce the exact
identity $\mathbb E G(Y_k,z_k)=0$, so it introduces no
factor $1/\rho$ into the error estimates.
\end{proof}

\section{The singular ratio limit and path-aware extrapolation}
\label{sec:sup:singular-path-rr}
Appendix~\ref{sec:sup:bias-expansion} established the stationary
mean expansion
\[
 m_{\eta,\rho}=\eta b(\rho)+O(\eta^{3/2}),
\]
with a remainder uniform for $0<\rho\le\rho_0$.
We now expand $b(\rho)$ as $\rho\downarrow0$ to obtain
the expansion stated in Theorem~\ref{thm:singular}.
The coefficient vectors $b_0,b_1,b_2$, defined in
\eqref{eq:report-b0}--\eqref{eq:report-b2}, multiply
$\eta$, $\varepsilon$, and $\varepsilon^2/\eta$,
respectively, and are independent of the step sizes.
To derive these coefficients, we expand the covariance
coefficient $\Sigma_\rho$ and the Markov response $r_\rho$.
For the covariance calculation, we use the tangent coordinates
\[
 u=x-x^\star-\Lambda(\theta-\theta^\star),
 \qquad
 v=\theta-\theta^\star,
\]
where $u$ is the fast tracking error and $v$ is the slow
displacement. In these coordinates, the leading covariance
coefficient is $\widetilde\Sigma_\rho=T\Sigma_\rho T^\top$.
We partition it as
\[
 \widetilde\Sigma_\rho=
 \begin{pmatrix}
  \widetilde\Sigma_\rho^{uu}
  &\widetilde\Sigma_\rho^{uv}\\
  (\widetilde\Sigma_\rho^{uv})^\top
  &\widetilde\Sigma_\rho^{vv}
 \end{pmatrix}.
\]
The $uu$ block describes fluctuations within the fast
coordinates, the $vv$ block describes fluctuations within
the slow coordinates, and the $uv$ block describes their
cross-covariance.

For $\rho>0$, define the rescaled matrices
\[
 U(\rho)=\widetilde\Sigma_\rho^{uu},
 \qquad
 R(\rho)=\rho^{-1}\widetilde\Sigma_\rho^{uv},
 \qquad
 V(\rho)=\rho^{-1}\widetilde\Sigma_\rho^{vv}.
\]
Thus we leave the fast block unchanged and divide the
cross and slow blocks by $\rho$, writing
\[
 \widetilde\Sigma_\rho=
 \begin{pmatrix}
 U(\rho)&\rho R(\rho)\\
 \rho R(\rho)^\top&\rho V(\rho)
 \end{pmatrix}.
\]
Substituting this representation into the transformed
Lyapunov equation, we divide its $uv$ equation by $\rho$
and its $vv$ equation by $\rho^2$.
The resulting equations remain informative at $\rho=0$
and determine finite limits for $U(\rho)$, $R(\rho)$,
and $V(\rho)$.
This allows us to compute the covariance expansion
needed for $b_0,b_1,b_2$.
\subsection{Proof of Theorem~\ref{thm:singular}}
\label{subsec:singular-proof}
\begin{proof}
Let $B_\rho=TD_\rho$ and $\widetilde\Sigma_\rho=T\Sigma_\rho T^\top$. The transformed Lyapunov equation is
\begin{equation}
 L_\rho\widetilde\Sigma_\rho+\widetilde\Sigma_\rho L_\rho^\top
 +B_\rho QB_\rho^\top=0.\label{eq:report-normal-lyapunov}
\end{equation}
Partition $Q$ according to $(x,\theta)$ and seek
\begin{equation}
 \widetilde\Sigma_\rho=
 \begin{pmatrix}U(\rho)&\rho R(\rho)\\
 \rho R(\rho)^\top&\rho V(\rho)\end{pmatrix}.
 \label{eq:report-regularized-Sigma}
\end{equation}
After dividing the cross and slow equations by $\rho$ and $\rho^2$, respectively, one obtains the regular system
\begin{align}
0={}&(A-\rho\Lambda C)U+U(A-\rho\Lambda C)^\top\nonumber\\
 &-\rho^2\Lambda SR^\top-\rho^2RS^\top\Lambda^\top\nonumber\\
 &+Q_{xx}-\rho(\Lambda Q_{\theta x}+Q_{x\theta}\Lambda^\top)
 +\rho^2\Lambda Q_{\theta\theta}\Lambda^\top,\label{eq:report-regular-U}\\
0={}&(A-\rho\Lambda C)R-\rho\Lambda SV+UC^\top\nonumber\\
 &+\rho RS^\top+Q_{x\theta}-\rho\Lambda Q_{\theta\theta},\label{eq:report-regular-R}\\
0={}&CR+SV+R^\top C^\top+VS^\top+Q_{\theta\theta}.
 \label{eq:report-regular-V}
\end{align}
At zero the system is triangular:
\begin{align}
 AU_0+U_0A^\top+Q_{xx}&=0,\label{eq:report-U0-equation}\\
 AR_0+U_0C^\top+Q_{x\theta}&=0,\label{eq:report-R0-equation}\\
 SV_0+V_0S^\top+CR_0+R_0^\top C^\top+Q_{\theta\theta}&=0.
 \label{eq:report-V0-equation}
\end{align}
Hurwitz stability makes this finite-dimensional linear system invertible. Its coefficients are polynomial in $\rho$, so its inverse and solution are analytic near zero. Successive coefficient matching therefore gives
\begin{equation}
 \Sigma_\rho=\Sigma_0+\rho\Sigma_1+\rho^2\Sigma_2+O(\rho^3),
 \quad\Sigma_0=\diag(U_0,0).\label{eq:report-Sigma-series}
\end{equation}
The corrector response is affine:
\begin{equation}
 \begin{gathered}
 r_\rho=r_0+\rho r_1,\quad
 r_0=\E[D_x\mathcal U(Y_1,z^\star)\xi_0^x],\\
 r_1=\E[D_\theta\mathcal U(Y_1,z^\star)\xi_0^\theta].
 \end{gathered}\label{eq:report-r-series}
\end{equation}
Consequently,
\begin{align}
 b_0&=-J^{-1}\{\mathcal C(\Sigma_0)+r_0\},\label{eq:report-b0}\\
 b_1&=-J^{-1}\{\mathcal C(\Sigma_1)+r_1\},\label{eq:report-b1}\\
 b_2&=-J^{-1}\mathcal C(\Sigma_2).\label{eq:report-b2}
\end{align}
Insert these expressions into Theorem~\ref{thm:fixed-ratio} and use $\eta\rho=\varepsilon$. Appendix~\ref{sec:sup:singular-path-rr} details the coefficient equations.
\end{proof}

\subsection{Overview of the Coefficient Expansion}

The starting point is the bias coefficient obtained in
Appendix~\ref{sec:sup:bias-expansion}:
\[
 b(\rho)=-J^{-1}\{\mathcal C(\Sigma_\rho)+r_\rho\}.
\]
Here $J=D\bar G(z^\star)$ and the curvature map
$\mathcal C$ are independent of $\rho$.
Thus, to expand $b(\rho)$, we need to determine how the
leading covariance coefficient $\Sigma_\rho$ and the
Markov response $r_\rho$ depend on the step-size ratio.

The dependence of $r_\rho$ is explicit.
Its defining expression
\eqref{eq:sup:report-markov-response} contains
$D_\rho=\operatorname{diag}(I,\rho I)$ linearly, so
\[
 r_\rho=r_0+\rho r_1,
\]
where $r_0$ and $r_1$ are defined in
\eqref{eq:report-r-series}.
This identity is exact; there are no higher-order
powers of $\rho$ in the direct Markov response.

The covariance coefficient requires a separate calculation.
As shown below, the rescaled covariance equations
\eqref{eq:sup:report-regular-U}--%
\eqref{eq:sup:report-regular-V}
form a finite-dimensional linear system that is invertible
at $\rho=0$. Its coefficients depend polynomially on $\rho$,
so its solution admits a Taylor expansion near zero.
Returning from the rescaled tangent coordinates to the
original coordinates gives
\[
 \Sigma_\rho
 =\Sigma_0+\rho\Sigma_1+\rho^2\Sigma_2+O(\rho^3).
\]
The matrices $\Sigma_0,\Sigma_1,\Sigma_2$ are independent
of the step sizes. The following subsections derive the
equations that determine them.

Substituting these expansions into the formula for
$b(\rho)$ and using the linearity of $\mathcal C$ gives
\[
 b(\rho)=b_0+\rho b_1+\rho^2b_2+O(\rho^3),
\]
with
\[
 \begin{aligned}
 b_0&=-J^{-1}\{\mathcal C(\Sigma_0)+r_0\},\\
 b_1&=-J^{-1}\{\mathcal C(\Sigma_1)+r_1\},\\
 b_2&=-J^{-1}\mathcal C(\Sigma_2).
 \end{aligned}
\]
These are the coefficient vectors defined in
\eqref{eq:report-b0}--\eqref{eq:report-b2}.
The first two coefficients contain both covariance and
direct Markov-response terms. The quadratic coefficient
$b_2$ contains only the covariance contribution because
$r_\rho$ has no quadratic term.

To express the result in the original step sizes, multiply
by $\eta$ and use $\rho=\varepsilon/\eta$.
The identities
$\eta\rho=\varepsilon$ and
$\eta\rho^2=\varepsilon^2/\eta$ yield
\begin{equation}
 \eta b(\rho)
 =\eta b_0+\varepsilon b_1
   +\frac{\varepsilon^2}{\eta}b_2
   +O\!\left(\frac{\varepsilon^3}{\eta^2}\right).
 \label{eq:sup:report-mixed-power-preview}
\end{equation}
Theorem~\ref{thm:fixed-ratio} has already established
$m_{\eta,\rho}=\eta b(\rho)+O(\eta^{3/2})$
with a remainder uniform in $\rho$.
Combining the two expansions therefore gives
\[
 \begin{aligned}
 m_{\eta,\rho}
 &=\eta b_0+\varepsilon b_1
   +\frac{\varepsilon^2}{\eta}b_2\\
 &\quad
   +O\!\left(\eta^{3/2}
             +\frac{\varepsilon^3}{\eta^2}\right),
 \end{aligned}
\]
which is the expansion in Theorem~\ref{thm:singular}.
The two remainder terms have different origins:
$O(\eta^{3/2})$ comes from the stationary mean approximation
in Theorem~\ref{thm:fixed-ratio}, whereas
$O(\varepsilon^3/\eta^2)$ comes from truncating the
expansion of $b(\rho)$ after its quadratic term.

\subsection{Rescaled Covariance Equations}

At \(\rho=0\) both slow drift and slow noise vanish.  We therefore pass to
fast-manifold coordinates and divide out the known powers of \(\rho\) before
taking the limit.  With \(J\), \(\Lambda\), and \(S\) as in
\eqref{eq:sup:report-Jrho}, \eqref{eq:sup:report-manifold-tangent}, and
\eqref{eq:sup:report-reduced-jacobian}, define
\begin{equation}
 T=\begin{pmatrix}I&-\Lambda\\0&I\end{pmatrix},
 \qquad
 \begin{pmatrix}u\\v\end{pmatrix}
 =T\begin{pmatrix}x-x^\star\\\theta-\theta^\star\end{pmatrix}.
 \label{eq:sup:report-T-transform}
\end{equation}
so that \((u,v)\) are the tangent coordinates
\eqref{eq:sup:report-linear-normal-coordinate}.

Let \(D_\rho=\operatorname{diag}(I,\rho I)\).  In the new coordinates, the
linear drift and noise scaling are
\begin{align}
 L_\rho=T D_\rho J T^{-1}
 &=\begin{pmatrix}
 A-\rho\Lambda C&-\rho\Lambda S\\
 \rho C&\rho S
 \end{pmatrix},
 \label{eq:sup:report-Lrho-again}\\
 B_\rho=T D_\rho
 &=\begin{pmatrix}I&-\rho\Lambda\\0&\rho I\end{pmatrix}.
 \label{eq:sup:report-Brho}
\end{align}
The leading covariance coefficient from Appendix~\ref{sec:sup:report-markov-covariance}
becomes \(\widetilde\Sigma_\rho=T\Sigma_\rho T^\top\) and solves
\begin{equation}
 L_\rho\widetilde\Sigma_\rho
 +\widetilde\Sigma_\rho L_\rho^\top
 +B_\rho Q B_\rho^\top=0.
 \label{eq:sup:report-normal-lyapunov}
\end{equation}
Partition the long-run noise covariance as
\[
 Q=\begin{pmatrix}Q_{xx}&Q_{x\theta}\\
 Q_{\theta x}&Q_{\theta\theta}\end{pmatrix}.
\]

The $uu$, $uv$, and $vv$ blocks of $B_\rho Q B_\rho^\top$
are respectively $O(1)$, $O(\rho)$, and $O(\rho^2)$.
Together with the different fast and slow relaxation rates,
these noise scales motivate the following rescaling of
the covariance coefficient.
For each $\rho>0$, write
\begin{equation}
 \widetilde\Sigma_\rho=
 \begin{pmatrix}
 U(\rho)&\rho R(\rho)\\
 \rho R(\rho)^\top&\rho V(\rho)
 \end{pmatrix}.
 \label{eq:sup:report-regularized-Sigma}
\end{equation}
For positive $\rho$, this is an exact change of variables.
The assertion that $U(\rho)$, $R(\rho)$, and $V(\rho)$
remain bounded as $\rho\downarrow0$ is justified by the
invertibility argument in
Appendix~\ref{app:singular-coefficients}.

Substitute this representation into
\eqref{eq:sup:report-normal-lyapunov}.
Leave the $uu$ equation unchanged, divide the $uv$ equation
by $\rho$, and divide the $vv$ equation by $\rho^2$.
For $\rho>0$, these operations give the equivalent system:
\begin{align}
0={}&(A-\rho\Lambda C)U+U(A-\rho\Lambda C)^\top
 -\rho^2\Lambda SR^\top-\rho^2RS^\top\Lambda^\top\nonumber\\*
&+Q_{xx}-\rho(\Lambda Q_{\theta x}+Q_{x\theta}\Lambda^\top)
 +\rho^2\Lambda Q_{\theta\theta}\Lambda^\top,
 \label{eq:sup:report-regular-U}\\
0={}&(A-\rho\Lambda C)R-\rho\Lambda SV+UC^\top
 +\rho RS^\top+Q_{x\theta}-\rho\Lambda Q_{\theta\theta},
 \label{eq:sup:report-regular-R}\\
0={}&CR+SV+R^\top C^\top+VS^\top+Q_{\theta\theta}.
 \label{eq:sup:report-regular-V}
\end{align}
We can now set \(\rho=0\).  The resulting equations form a triangular chain:
solve first for the fast covariance \(U_0\), then the cross-covariance
\(R_0\), and finally the slow covariance \(V_0\):
\begin{align}
 AU_0+U_0A^\top+Q_{xx}&=0,
 \label{eq:sup:report-U0-equation}\\
 AR_0+U_0C^\top+Q_{x\theta}&=0,
 \label{eq:sup:report-R0-equation}\\
 SV_0+V_0S^\top+CR_0+R_0^\top C^\top+Q_{\theta\theta}&=0.
 \label{eq:sup:report-V0-equation}
\end{align}
Stability of \(A\) and \(S\) gives uniqueness at every step.

Eliminating \(R_0\) makes the slow-noise interpretation explicit.  The middle
equation gives
\(R_0=-A^{-1}(U_0C^\top+Q_{x\theta})\).  Substituting this into the last
equation gives
\[
 SV_0+V_0S^\top+Q_{\mathrm{red}}=0,
\]
with
\begin{equation}
 Q_{\mathrm{red}}
 =\begin{pmatrix}-CA^{-1}&I\end{pmatrix}
 Q
 \begin{pmatrix}-CA^{-1}&I\end{pmatrix}^{\!\top}.
 \label{eq:sup:report-Qred}
\end{equation}
 The slow system receives both direct noise and fast noise transmitted by the equilibrium response \(-CA^{-1}\).

Multiplication by $\eta$ gives the leading covariance
contributions:
\begin{equation}
 [\eta\widetilde\Sigma_\rho]_{uu}=O(\eta),\qquad
 [\eta\widetilde\Sigma_\rho]_{uv}=O(\varepsilon),\qquad
 [\eta\widetilde\Sigma_\rho]_{vv}=O(\varepsilon).
 \label{eq:sup:report-actual-block-orders}
\end{equation}
These are orders for the leading covariance approximation.
Proposition~\ref{prop:covariance} bounds the difference
between the actual covariance and this approximation by
$O(\eta^{3/2})$. Consequently, under
Assumptions~\ref{as:smooth}--\ref{as:taylor}, it gives
\[
 \|\operatorname{Cov}(v)\|
 +\|\operatorname{Cov}(u,v)\|
 =O(\varepsilon+\eta^{3/2}).
\]
When $\rho\lesssim\sqrt{\eta}$, the approximation error
can be comparable to or larger than the leading
$O(\varepsilon)$ terms.
The stronger bounds
\[
 \|\operatorname{Cov}(v)\|=O(\varepsilon),
 \qquad
 \|\operatorname{Cov}(u,v)\|=O(\varepsilon),
\]
uniformly over all sufficiently small positive $\eta$
and $\rho$, are additional conditions collected in
Assumption~\ref{as:block-localization}.

Extrapolation needs terms through \(\rho^2\).  \[
 U=U_0+\rho U_1+\rho^2U_2+O(\rho^3),\qquad
 R=R_0+\rho R_1+O(\rho^2),\qquad
 V=V_0+\rho V_1+O(\rho^2).
\]
Differentiate \eqref{eq:sup:report-regular-U}--\eqref{eq:sup:report-regular-V} to
obtain the coefficients by successive Lyapunov or Sylvester solves.
Transforming back gives
\begin{equation}
 \Sigma_\rho=\Sigma_0+\rho\Sigma_1+\rho^2\Sigma_2+O(\rho^3).
 \label{eq:sup:report-Sigma-series}
\end{equation}
The coefficients follow from
\(\Sigma_\rho=T^{-1}\widetilde\Sigma_\rho T^{-\top}\); in particular,
\(\Sigma_0=\operatorname{diag}(U_0,0)\).

\subsection{Derivation of the Covariance Expansion}
\label{app:singular-coefficients}

We now justify the expansion
$\Sigma_\rho=\Sigma_0+\rho\Sigma_1+\rho^2\Sigma_2+O(\rho^3)$
and derive equations for its coefficients.
The argument first establishes that the rescaled covariance
equations have a solution analytic in $\rho$ near zero.
We then compute its coefficients and transform back to the
original coordinates.

\paragraph{Existence of the Expansion.}
Let $\mathbb S^m$ denote the space of real symmetric
$m\times m$ matrices, and set
\[
 \mathcal X
 =\mathbb S^{d_x}\times\mathbb R^{d_x\times d_\theta}
  \times\mathbb S^{d_\theta}.
\]
The rescaled equations
\eqref{eq:sup:report-regular-U}--%
\eqref{eq:sup:report-regular-V} can be written as
\[
 \mathscr A(\rho)(U,R,V)+\mathscr F(\rho)=0.
\]
Here $\mathscr A(\rho):\mathcal X\to\mathcal X$ contains
all terms linear in $U,R,V$, while $\mathscr F(\rho)$
consists of the terms involving the blocks of $Q$.
In fixed bases, both have polynomial entries in $\rho$.

To prove invertibility at zero, consider the homogeneous
system $\mathscr A(0)(U,R,V)=0$:
\[
 \begin{aligned}
 AU+UA^\top&=0,\\
 AR+UC^\top&=0,\\
 SV+VS^\top+CR+R^\top C^\top&=0.
 \end{aligned}
\]
Since $A$ is Hurwitz, its Lyapunov equation in the first
line has only the solution $U=0$. The second line then
gives $R=0$ because $A$ is invertible. Finally, Hurwitz
stability of $S$ gives $V=0$ from the third line.
Thus $\mathscr A(0)$ has trivial kernel.
It maps the finite-dimensional space $\mathcal X$ into
itself and is therefore invertible.

Its determinant remains nonzero on a neighborhood of zero.
Because its entries are polynomial in $\rho$, the inverse
$\mathscr A(\rho)^{-1}$ is analytic on that neighborhood.
Consequently, the solution is analytic and is given by
\[
 (U(\rho),R(\rho),V(\rho))
 =-\mathscr A(\rho)^{-1}\mathscr F(\rho).
\]

For sufficiently small positive $\rho$, assemble these
matrices as
\[
 \widetilde\Sigma_\rho=
 \begin{pmatrix}
 U(\rho)&\rho R(\rho)\\
 \rho R(\rho)^\top&\rho V(\rho)
 \end{pmatrix}.
\]
The equivalence of the rescaled equations and
\eqref{eq:sup:report-normal-lyapunov} shows that this matrix
solves the transformed Lyapunov equation.
Proposition~\ref{prop:semigroup} ensures that $L_\rho$ is
Hurwitz for sufficiently small positive $\rho$, so that
equation has a unique solution.
Hence the assembled matrix is precisely
$T\Sigma_\rho T^\top$.
This establishes the expansion for the leading covariance
coefficient.

\paragraph{Calculation of the Coefficients}
To obtain the assembled covariance through order $\rho^2$,
we need $U$ through order $\rho^2$ and $R,V$ through order
$\rho$, since the latter two blocks already carry a factor
$\rho$. Write
\[
 \begin{aligned}
 U(\rho)&=U_0+\rho U_1+\rho^2U_2+O(\rho^3),\\
 R(\rho)&=R_0+\rho R_1+O(\rho^2),\\
 V(\rho)&=V_0+\rho V_1+O(\rho^2).
 \end{aligned}
\]
These are power-series coefficients; in particular,
$U_2=U''(0)/2$.
First compute $U_0,R_0,V_0$ from the successive equations
\eqref{eq:sup:report-U0-equation}--%
\eqref{eq:sup:report-V0-equation}.

Next, match the coefficients of $\rho$ in the three
rescaled equations. This gives, in the order
$U_1,R_1,V_1$,
\begin{align}
 AU_1+U_1A^\top
 &=\Lambda C U_0+U_0C^\top\Lambda^\top
   +\Lambda Q_{\theta x}+Q_{x\theta}\Lambda^\top,
 \label{eq:explicit-U1}\\
 AR_1
 &=\Lambda C R_0+\Lambda S V_0-U_1C^\top-R_0S^\top
   +\Lambda Q_{\theta\theta},
 \label{eq:explicit-R1}\\
 SV_1+V_1S^\top
 &=-CR_1-R_1^\top C^\top.
 \label{eq:explicit-V1}
\end{align}
Matching the coefficient of $\rho^2$ in
\eqref{eq:sup:report-regular-U} then gives
\begin{equation}
 \begin{aligned}
 AU_2+U_2A^\top
 ={}&\Lambda C U_1+U_1C^\top\Lambda^\top\\
 &+\Lambda SR_0^\top+R_0S^\top\Lambda^\top
   -\Lambda Q_{\theta\theta}\Lambda^\top.
 \end{aligned}
 \label{eq:explicit-U2}
\end{equation}
Each right-hand side contains only known matrices at the
time it is used. The equations for $U_1,V_1,U_2$ have unique
solutions because the corresponding Lyapunov operators
are invertible; the equation for $R_1$ is uniquely solved
by multiplication by $A^{-1}$.

\paragraph{Return to the Original Coordinates.}

Define
\[
 \widetilde\Sigma_0=\begin{pmatrix}U_0&0\\0&0\end{pmatrix},\qquad
 \widetilde\Sigma_1=\begin{pmatrix}U_1&R_0\\R_0^\top&V_0\end{pmatrix},\qquad
 \widetilde\Sigma_2=\begin{pmatrix}U_2&R_1\\R_1^\top&V_1\end{pmatrix}.
\]
The assembled expansion is
\[
 \widetilde\Sigma_\rho
 =\widetilde\Sigma_0+\rho\widetilde\Sigma_1
  +\rho^2\widetilde\Sigma_2+O(\rho^3).
\]
Using $T$ from \eqref{eq:sup:report-T-transform}, set
\[
 \Sigma_j=T^{-1}\widetilde\Sigma_j(T^{-1})^\top,
 \qquad j=0,1,2.
\]

Since $T$ is independent of $\rho$, transforming back yields
\[
 \Sigma_\rho
 =\Sigma_0+\rho\Sigma_1+\rho^2\Sigma_2+O(\rho^3),
 \qquad
 \Sigma_0=\operatorname{diag}(U_0,0).
\]
After choosing $\rho_0$ so that $[0,\rho_0]$ lies inside
the analyticity neighborhood, the derivatives needed for
these Taylor remainders are uniformly bounded.
Thus the final remainder is bounded by $C\rho^3$ for
$0\le\rho\le\rho_0$, where the value at zero denotes the
analytic extension.
Substitution into
\eqref{eq:report-b0}--\eqref{eq:report-b2}
then gives the three bias coefficients.

\subsection{Bias Coefficients}

The preceding subsection determines the covariance coefficients
$\Sigma_0,\Sigma_1,\Sigma_2$.
We now use them to compute the bias coefficients
$b_0,b_1,b_2$ and establish the stationary mean expansion
in the two step sizes.
The starting point is the formula from
Theorem~\ref{thm:fixed-ratio}:
\[
 b(\rho)=-J^{-1}\{\mathcal C(\Sigma_\rho)+r_\rho\}.
\]

First consider the direct Markov response $r_\rho$.
Let $(Y_0,Y_1)$ have the stationary one-step distribution
$\mathbb P(Y_0=y,Y_1=y')=\mu(y)P(y,y')$, and write
\[
 \xi_0=G(Y_0,z^\star)
 =\begin{pmatrix}\xi_0^x\\ \xi_0^\theta\end{pmatrix}.
\]
By \eqref{eq:sup:report-markov-response},
\[
 r_\rho
 =\mathbb E\!\left[
 D_z\mathcal U(Y_1,z^\star)D_\rho\xi_0
 \right].
\]
The matrix $D_\rho$ leaves the fast component unchanged
and multiplies the slow component by $\rho$.
Since the other quantities in this expectation do not
depend on $\rho$, we obtain the exact decomposition
\begin{equation}
 \begin{aligned}
 r_\rho&=r_0+\rho r_1,\\
 r_0&=\mathbb E[D_x\mathcal U(Y_1,z^\star)\xi_0^x],\\
 r_1&=\mathbb E[D_\theta\mathcal U(Y_1,z^\star)
                          \xi_0^\theta].
 \end{aligned}
 \label{eq:sup:report-r-series}
\end{equation}
Here $D_x\mathcal U$ and $D_\theta\mathcal U$ are
derivatives of the full vector-valued corrector with
respect to the fast and slow variables, respectively.

Next, insert this decomposition and the covariance
expansion \eqref{eq:sup:report-Sigma-series} into
the formula for $b(\rho)$.
Because $\mathcal C$ is linear and $J^{-1}$ is fixed,
\begin{equation}
 b(\rho)=b_0+\rho b_1+\rho^2b_2+O(\rho^3),
 \label{eq:sup:report-b-series}
\end{equation}
where
\begin{align}
 b_0&=-J^{-1}\{\mathcal C(\Sigma_0)+r_0\},
 \label{eq:sup:report-b0}\\
 b_1&=-J^{-1}\{\mathcal C(\Sigma_1)+r_1\},
 \label{eq:sup:report-b1}\\
 b_2&=-J^{-1}\mathcal C(\Sigma_2).
 \label{eq:sup:report-b2}
\end{align}
These vectors are independent of both step sizes.
There is no quadratic term in $r_\rho$, so $b_2$ has
no direct Markov-response term.
It can nevertheless depend on temporal correlations
through the covariance coefficient $\Sigma_2$.

To pass from the coefficient expansion to the stationary
mean, we combine two uniform bounds.
Theorem~\ref{thm:fixed-ratio} gives
\[
 \|m_{\eta,\rho}-\eta b(\rho)\|
 \le C_1\eta^{3/2},
\]
while \eqref{eq:sup:report-b-series} gives
\[
 \|b(\rho)-b_0-\rho b_1-\rho^2b_2\|
 \le C_2\rho^3.
\]
The constants are independent of sufficiently small
positive $\eta,\rho$.
Multiplying the second bound by $\eta$ and adding the
first yields
\[
 \bigl\|m_{\eta,\rho}
 -\eta b_0-\eta\rho b_1-\eta\rho^2b_2\bigr\|
 \le C_1\eta^{3/2}+C_2\eta\rho^3.
\]
Using
$\eta\rho=\varepsilon$,
$\eta\rho^2=\varepsilon^2/\eta$, and
$\eta\rho^3=\varepsilon^3/\eta^2$
proves Theorem~\ref{thm:singular}, restated below.

\begin{appendixsingularrestatement}[Singular expansion; restatement]
Under
\Cref{as:sup:smooth,as:sup:hurwitz,as:sup:markov,as:sup:localization,as:sup:taylor},
there exist $\eta_0,\rho_0>0$ such that, for
$0<\eta\le\eta_0$ and $0<\rho\le\rho_0$, with
$\varepsilon=\eta\rho$,
\begin{equation}
 \begin{aligned}
 m_{\eta,\rho}
 &=\eta b_0+\varepsilon b_1
   +\frac{\varepsilon^2}{\eta}b_2\\
 &\quad
   +O\!\left(\eta^{3/2}+\frac{\varepsilon^3}{\eta^2}\right),
 \end{aligned}
 \label{eq:sup:report-mixed-expansion}
\end{equation}
where the implied constant is independent of $\eta$
and $\rho$.
\end{appendixsingularrestatement}

The factor $\varepsilon^2/\eta$ comes from the quadratic
term $\eta\rho^2$ in the ratio expansion.
Although it contains a negative power of $\eta$ when
written in the two step sizes, it satisfies
\[
 0<\frac{\varepsilon^2}{\eta}
 =\eta\rho^2\le\eta\rho_0^2.
\]
Thus the mixed contribution tends to zero uniformly
as $\eta\downarrow0$ in the stated parameter range.
Along paths with $\rho\downarrow0$,
\[
 \frac{\varepsilon^2/\eta}{\varepsilon}=\rho\longrightarrow0,
\]
so its step-size factor is $o(\varepsilon)$.
Whether this contribution can be distinguished from
the remainder is considered in the next subsection.

\subsection{Visibility along \texorpdfstring{\(\varepsilon=\eta^p\)}{eps = eta\^p}}

The expansion in the preceding subsection displays three
bias terms, but some may be smaller than the approximation
error. We now compare their sizes along a prescribed
relation between the two step sizes.
This identifies which terms can be distinguished from
the proved remainder and determines the powers relevant
to extrapolation in the next subsection.

Fix $p>1$ and choose the step-size path
\begin{equation}
 \varepsilon=\eta^p.
 \label{eq:sup:report-path}
\end{equation}
Then $\rho=\varepsilon/\eta=\eta^{p-1}\downarrow0$
as $\eta\downarrow0$.
Define $m(\eta):=m_{\eta,\eta^{p-1}}$.
Substitution into \eqref{eq:sup:report-mixed-expansion}
gives
\begin{equation}
 \begin{aligned}
 m(\eta)
 &=\eta b_0+\eta^p b_1+\eta^{2p-1}b_2\\
 &\quad+O(\eta^{3/2}+\eta^{3p-2}).
 \end{aligned}
 \label{eq:sup:report-path-expansion}
\end{equation}
Both contributions to the remainder must be retained:
$O(\eta^{3/2})$ comes from the stationary mean approximation,
and $O(\eta^{3p-2})$ comes from truncating the ratio expansion.

A nonzero bias term is resolved by this expansion when,
after subtracting the preceding terms, the remaining error
is of smaller order than that term.
For positive powers of $\eta$, a smaller exponent means
slower decay as $\eta\downarrow0$.
Since
\[
 1<p<2p-1<3p-2,
\]
the error $O(\eta^{3p-2})$ from omitting higher-order
terms in the expansion of $b(\rho)$ is of smaller order
than every displayed bias term.
The comparison with $\eta^{3/2}$
therefore determines which terms are resolved.
For nonzero corresponding coefficients, $\eta b_0$ is
resolved for every $p>1$, while
\[
 \begin{aligned}
 \eta^p b_1 &: \qquad 1<p<\frac32,\\
 \eta^{2p-1}b_2 &: \qquad 1<p<\frac54.
 \end{aligned}
\]

For example, subtracting the first two contributions and
dividing by the mixed-term scale gives
\[
 \frac{m(\eta)-\eta b_0-\eta^p b_1}{\eta^{2p-1}}
 =b_2+O\!\left(\eta^{5/2-2p}+\eta^{p-1}\right).
\]
Both error terms tend to zero when $1<p<5/4$.
At $p=5/4$, the first error bound is only $O(1)$,
so the stated expansion does not guarantee convergence
of this normalized expression to $b_2$.

Under the additional assumptions of Theorem~\ref{thm:iid},
the stationary mean remainder improves to $O(\eta^2)$.
The path expansion then has remainder
$O(\eta^2+\eta^{3p-2})$.
The same comparison extends the range for a nonzero
$b_1$ term to $1<p<2$ and for a nonzero $b_2$ term to
$1<p<3/2$.
These endpoints are strict: when a term and the allowed
remainder have the same order, the bound alone does not
separate them.

More generally, a term $\rho^j$ in the analytic expansion
of $b(\rho)$ contributes
$\eta\rho^j=\eta^{1+j(p-1)}$ to $\eta b(\rho)$.
Thus the ratio expansion generates the exponents
\begin{equation}
 \gamma_j=1+j(p-1),\qquad j=0,1,2,\ldots:
 \qquad 1,\ p,\ 2p-1,\ 3p-2,\ldots.
 \label{eq:sup:report-path-ladder}
\end{equation}
This sequence describes the powers in $\eta b(\rho)$;
expanding it further does not reduce the stationary
approximation error.
If a higher-order stationary expansion establishes
additional terms, such as a term proportional to $\eta^2$,
their exponents must also be included when choosing
extrapolation weights. Contributions with the same exponent
are combined into a single coefficient along the path.

\subsection{Richardson--Romberg Extrapolation along Step-Size Paths}

The preceding subsection identifies the bias powers along
a chosen step-size path. We now combine stationary means
at several step sizes so that selected bias terms cancel
while the equilibrium $z^\star$ is preserved.
We begin with a fixed step-size ratio, then consider the
path $\varepsilon=h^p$, and finally describe the weights
for cancelling an arbitrary finite set of bias powers.

First fix $\rho$ and let $M(h)$ denote the stationary mean
at fast step $h$ and slow step $\rho h$.
Theorem~\ref{thm:fixed-ratio} gives
\(M(h)=z^\star+h b(\rho)+O(h^{3/2})\).
Halving both step sizes preserves $\rho$, so the same
coefficient $b(\rho)$ appears at the finer level.
The weights $-1$ and $2$ preserve the equilibrium and
cancel the term linear in $h$, giving
\begin{equation}
 2M(h/2)-M(h)=z^\star+O(h^{3/2}).
 \label{eq:sup:report-fixed-ratio-RR}
\end{equation}

Now fix $p>1$ and use $M(h)$ for the stationary mean
at fast step $h$ and slow step $h^p$.
Refinement must preserve this relation, so the step-size
pairs are \((h_j,\varepsilon_j)=(2^{-j}h,h_j^p)\).
The preceding two-level rule cancels the term proportional
to $h$ but generally leaves the term proportional to $h^p$.
To cancel both powers while preserving $z^\star$, combine
three levels with weights satisfying
\begin{equation}
 \sum_{j=0}^2w_j=1,\qquad
 \sum_{j=0}^2w_j2^{-j}=0,\qquad
 \sum_{j=0}^2w_j2^{-jp}=0.
 \label{eq:sup:report-three-weight-system}
\end{equation}
The first constraint preserves the common equilibrium.
The second and third cancel the $h$ and $h^p$ terms,
whose coefficients at level $j$ are multiplied by
$2^{-j}$ and $2^{-jp}$, respectively.
Put \(q_1=2^{-1}\) and \(q_p=2^{-p}\).
Solving these three linear equations gives
\begin{equation}
 \boxed{
 w_0=\frac{q_1q_p}{(1-q_1)(1-q_p)},\qquad
 w_1=-\frac{q_1+q_p}{(1-q_1)(1-q_p)},\qquad
 w_2=\frac{1}{(1-q_1)(1-q_p)}.}
 \label{eq:sup:report-three-weights}
\end{equation}

Applying these weights to
\eqref{eq:sup:report-path-expansion} removes its first
two bias terms. The remaining stationary bias is
\(O(h^{\min\{2p-1,\,3/2\}})\).
The exponent reflects the surviving mixed term and the
stationary approximation error.
The additional error of order $h^{3p-2}$ is absorbed
into this bound because $3p-2>2p-1$ for $p>1$.
When the sharper uniform weak remainder is available,
as under the additional assumptions of
Theorem~\ref{thm:iid}, replace \(3/2\) by \(2\).

More generally, suppose we wish to cancel the distinct
positive bias powers \(\gamma_1,\ldots,\gamma_M\).
Using $M+1$ levels of the same geometric refinement,
the weights must satisfy
\begin{equation}
 \sum_{j=0}^Mw_j=1,
 \qquad
 \sum_{j=0}^Mw_j2^{-j\gamma_\ell}=0,
 \quad \ell=1,\ldots,M.
 \label{eq:sup:report-general-RR-system}
\end{equation}
As before, the first equation preserves the equilibrium,
and each remaining equation cancels one selected power.
These conditions can equivalently be expressed through
the polynomial
\[
 W(t)=\sum_{j=0}^Mw_jt^j
 =\prod_{\ell=1}^M
 \frac{t-2^{-\gamma_\ell}}{1-2^{-\gamma_\ell}}.
\]
Each numerator factor gives
$W(2^{-\gamma_\ell})=0$, while the denominator normalizes
the product so that $W(1)=1$.
Expanding this polynomial in powers of $t$ gives the
weights $w_0,\ldots,w_M$.

For example, adding \(\gamma_3=2p-1\) to the powers
$1$ and $p$ gives a four-level rule that also cancels
the term containing $b_2$.
The resulting bias order remains limited by the remainder
established for the stationary expansion; the weight
equations alone do not justify a sharper remainder.
For finite runs, initialization error must also be
controlled at the retained remainder scale to preserve
the stated bias order.

\section{Additive Independent Noise}\label{app:iid}

For additive i.i.d.\ noise, we can sharpen the covariance and mean
remainders to $O(\eta^2)$. The key step is an $O(\eta^2)$ bound on
centered third moments, obtained from cancellations in their
stationary balance equation. We use this bound first to refine the
covariance and then to obtain the mean expansion. The theorem below
states the required smoothness and moment assumptions explicitly,
including localization around the selected equilibrium $z^\star$.

\begin{theorem}[\texorpdfstring{$O(\eta^2)$}{O(eta squared)} remainder under additive i.i.d.\ noise]
\label{thm:iid}
Consider
\begin{equation}
 z_{k+1}=z_k+\eta D_\rho\{\bar G(z_k)+\zeta_{k+1}\},
 \label{eq:assumption-iid-recursion}
\end{equation}
where the innovations are identically distributed, independent of one
another and of the past, centered, and have a finite sixth moment.
Assume that $\bar G(z^\star)=0$, that the matrices $A$ and $S$ in
\eqref{eq:J} are Hurwitz, and that $\bar G\in C^4(\mathbb R^{d_x+d_\theta})$
has bounded derivatives of orders one through four. For every
sufficiently small positive $\eta,\rho$, suppose that a stationary law
exists and satisfies
\begin{equation}
 \mathbb E\|z-z^\star\|^{2j}\le C_j\eta^j,
 \qquad j=1,2,3,
 \label{eq:iid-raw-localization}
\end{equation}
with constants independent of the two step sizes.
Let $Q_\zeta=\mathbb E[\zeta_1\zeta_1^\top]$, and let $\Sigma_\rho$
solve \eqref{eq:lyap} with $Q=Q_\zeta$. Then, uniformly for
$0<\rho\le\rho_0$,
\begin{align}
 \operatorname{Cov}(z-z^\star)
 &=\eta\Sigma_\rho+O(\eta^2),\label{eq:iid-covariance-result}\\
 m:=\mathbb E(z-z^\star)
 &=-\eta J^{-1}\mathcal C(\Sigma_\rho)+O(\eta^2).
 \label{eq:assumption-iid-refined-mean}
\end{align}
\end{theorem}

\begin{proof}
We first bound the mean and center the recursion. We then establish
an inverse estimate that is uniform as $\rho\downarrow0$ and use it
to control the third moments and the covariance, in that order.
The mean expansion follows from a final Taylor expansion of the
stationary drift equation. All constants below are independent of
sufficiently small positive $\eta$ and $\rho$.

\paragraph{Step 1: Mean Control and Centering.}
Stationary increment balance, together with the centered innovations
and the invertibility of $D_\rho$ for $\rho>0$, gives
$\mathbb E\bar G(z)=0$. Taylor expansion at $z^\star$ and
\eqref{eq:iid-raw-localization} give
\[
 0=Jm+O(\mathbb E\|z-z^\star\|^2)=Jm+O(\eta).
\]
Since $J$ is fixed and invertible, $m=O(\eta)$ uniformly in $\rho$.
This preliminary bound does not use the bias expansion to be proved.
Put $q=T(z-z^\star-m)$, where $T$ is defined in
\eqref{eq:tangent-coordinates}, and set $B_\rho=TD_\rho$.
Define the nonlinear remainder by
\[
 r_m(q)=B_\rho\bar G(z^\star+m+T^{-1}q)-L_\rho q.
\]
The centered recursion is then exactly
\begin{equation}
 q_{k+1}=q_k+\eta\{L_\rho q_k+r_m(q_k)+B_\rho\zeta_{k+1}\},
 \label{eq:assumption-iid-centered}
\end{equation}
with $\mathbb E q=0$ and $\mathbb E r_m(q)=0$. The localization
bounds and $m=O(\eta)$ imply the centered moment bounds below;
Taylor expansion at $z^\star+m$ gives the remainder estimates:
\begin{align}
 \mathbb E\|q\|^{2j}&\le C_j\eta^j,\qquad j=1,2,3,
 \label{eq:assumption-iid-even}\\
 \|r_m^u(q)\|&\le C(\|q\|^2+\eta\|q\|+\eta),\nonumber\\
 \|r_m^v(q)\|&\le C\rho(\|q\|^2+\eta\|q\|+\eta).
 \label{eq:assumption-iid-remainder}
\end{align}
Indeed, the linearization at $z^\star+m$ differs from $J$ by
$O(\eta)$. Subtracting the expectation of the quadratic remainder
introduces a constant of order $\eta$. The factor $\rho$ in the
second estimate comes from the lower block row of $B_\rho$.
We will also use a linear-growth bound for the full centered drift
$d(q)=L_\rho q+r_m(q)$. Global Lipschitz continuity gives
\[
 \|d^u(q)\|\le C(\|q\|+\sqrt\eta),\qquad
 \|d^v(q)\|\le C\rho(\|q\|+\sqrt\eta),\qquad
 \mathbb E d(q)=0.
\]

\paragraph{Step 2: A Uniform Inverse for the Moment Equations.}
Both the covariance and third-moment equations involve the following
operator, with $s=2$ and $s=3$, respectively:
\[
 \mathscr K_{s,\rho}
 =\sum_{a=1}^s I^{\otimes(a-1)}\otimes L_\rho
                         \otimes I^{\otimes(s-a)}.
\]
Its inverse need not be uniformly bounded on arbitrary tensors.
The estimate we need retains the factor $\rho$ in the block whose
indices are all slow coordinates. Specifically, if a tensor $R$
satisfies
\[
 \|R\|\le a_\eta,\qquad
 \|R_{v\cdots v}\|\le\rho a_\eta,
\]
then
\begin{equation}
 \|\mathscr K_{s,\rho}^{-1}R\|\le C_s a_\eta.
 \label{eq:iid-tensor-inverse}
\end{equation}
Here $R_{v\cdots v}$ denotes the block in which every tensor index
is a slow coordinate. To prove the estimate, use the change of basis
from the proof of Proposition~\ref{prop:sup:semigroup}:
\[
 Q_\rho^{-1}L_\rho Q_\rho
 =\operatorname{diag}(A_f(\rho),\rho S_s(\rho)),
 \qquad Q_\rho=I+O(\rho),
\]
where $A_f(\rho)$ and $S_s(\rho)$ are uniformly Hurwitz.
The all-slow block of the transformed forcing
$(Q_\rho^{-1})^{\otimes s}R$ is still $O(\rho a_\eta)$.
Indeed, any contribution from a block containing a fast index
includes at least one fast-to-slow entry of $Q_\rho^{-1}$,
which is $O(\rho)$.

In the transformed operator, a tensor block with at least one fast
index has a uniformly bounded inverse: the integral of its product
semigroup contains a uniformly decaying fast factor. The all-slow
inverse is $O(1/\rho)$, but its forcing is $O(\rho a_\eta)$.
Thus every solution block is $O(a_\eta)$. Transforming back proves
\eqref{eq:iid-tensor-inverse}.

\paragraph{Step 3: The Centered Third-Moment Bound.}
We now derive the bound on signed third moments that sharpens the
covariance remainder. Expanding the stationary third-moment identity
using \eqref{eq:assumption-iid-centered} and dividing by $\eta$ gives
an equation of the form
\[
 0=\mathscr K_{3,\rho}\mathbb E(q^{\otimes3})+R_3.
\]
We identify $R_3$ explicitly to show where independence and centering
are used. Write $q=q_k$, $q'=q_{k+1}$, $d=d(q)$,
$n=B_\rho\zeta_{k+1}$, and $h=d+n$, so that $q'=q+\eta h$.
For fixed indices $i,j,\ell$, let $\sum_{\mathrm{cyc}}$ denote
the sum over $(i,j,\ell)$, $(j,\ell,i)$, and $(\ell,i,j)$,
retaining all three placements even when indices coincide.
Expanding the product and using stationarity gives the exact identity
\[
 \begin{aligned}
 0={}&\sum_{\mathrm{cyc}}\mathbb E[h_iq_jq_\ell]
 +\eta\sum_{\mathrm{cyc}}\mathbb E[h_ih_jq_\ell]
 +\eta^2\mathbb E[h_ih_jh_\ell].
 \end{aligned}
\]
The innovation $n$ is centered and independent of $q$, and hence
of $d(q)$, so every term containing exactly one innovation vanishes.
Writing $V_{ij}=\mathbb E[n_in_j]$, the terms containing exactly
two innovations also vanish:
\[
 \mathbb E[n_in_jq_\ell]=V_{ij}\mathbb E q_\ell=0,
 \qquad
 \mathbb E[n_in_jd_\ell]=V_{ij}\mathbb E d_\ell=0.
\]
These cancellations use both $\mathbb E q=0$ and
$\mathbb E d=0$. The term containing three innovations remains;
no symmetry of the innovation distribution is required.

Set $M_{ij\ell}=\mathbb E[q_iq_jq_\ell]$ and $r=r_m(q)$.
Substituting $d=L_\rho q+r$ in the first-order term identifies
$\mathscr K_{3,\rho}M$, leaving
\[
 \begin{aligned}
 (R_3)_{ij\ell}={}&
 \sum_{\mathrm{cyc}}\mathbb E[r_iq_jq_\ell]
 +\eta\sum_{\mathrm{cyc}}\mathbb E[d_id_jq_\ell]\\
 &+\eta^2\mathbb E[d_id_jd_\ell]
 +\eta^2\mathbb E[n_in_jn_\ell].
 \end{aligned}
\]
The remainder estimate in \eqref{eq:assumption-iid-remainder}
bounds the first group by
\[
 C\{\mathbb E\|q\|^4
       +\eta\mathbb E\|q\|^3
       +\eta\mathbb E\|q\|^2\}=O(\eta^2).
\]
The linear-growth bound on $d$ and the centered moment bounds give
orders $\eta^{5/2}$ and $\eta^{7/2}$ for the second and third
groups, respectively. The finite third moment of the innovations
gives $O(\eta^2)$ for the fourth group.

To apply Step~2, we also need the all-slow block. Each term in the
first group then contains a factor $\rho$ from $r^v$.
The remaining three groups contain at least $\rho^2$, $\rho^3$,
and $\rho^3$, respectively, from their slow drift or innovation
coordinates. Thus $\|R_3\|\le C\eta^2$ and
$\|(R_3)_{vvv}\|\le C\rho\eta^2$. By
\eqref{eq:iid-tensor-inverse},
\begin{equation}
 \max_{i,j,\ell}|\mathbb E(q_iq_jq_\ell)|\le C\eta^2.
 \label{eq:assumption-iid-signed-third}
\end{equation}

\paragraph{Step 4: Covariance Refinement.}
We next use the signed third-moment bound in the stationary
covariance equation. Let $P_c=\mathbb E(qq^\top)$.
Independence and centering of the innovations give the exact balance
\begin{align}
 0={}&L_\rho P_c+P_cL_\rho^\top
       +\eta B_\rho Q_\zeta B_\rho^\top\nonumber\\
 &+\mathbb E\{q r_m(q)^\top+r_m(q)q^\top\}
       +\eta\mathbb E\{d(q)d(q)^\top\}.
 \label{eq:assumption-iid-covariance-balance}
\end{align}
Expand $r_m$ through its quadratic term at the stationary mean
$z^\star+m$. In the cross terms with $q$, the change in the linear
coefficient contributes $O(\eta)\mathbb E\|q\|^2=O(\eta^2)$.
The quadratic term contracts with the signed third moments in
\eqref{eq:assumption-iid-signed-third} and is therefore
$O(\eta^2)$. The cubic remainder contributes
$O(\mathbb E\|q\|^4)=O(\eta^2)$. All constants subtracted when
centering these terms disappear because $\mathbb E q=0$.
Finally, the discretization term satisfies
$\eta\mathbb E\|d(q)\|^2=O(\eta^2)$.

Let $R_2$ denote the sum of the nonlinear cross terms and the
discretization term in this balance. Its slow--slow block retains
a factor $\rho$ in each nonlinear term and $\rho^2$ in the
discretization term. Consequently,
\[
 \|R_2\|\le C\eta^2,\qquad
 \|(R_2)_{vv}\|\le C\rho\eta^2.
\]
Subtract the Lyapunov equation for
$\eta\widetilde\Sigma_\rho=\eta T\Sigma_\rho T^\top$ and apply
\eqref{eq:iid-tensor-inverse} with $s=2$. This yields
\begin{equation}
 P_c=\eta\widetilde\Sigma_\rho+O(\eta^2).
 \label{eq:assumption-iid-refined-covariance}
\end{equation}
Transforming back with the fixed matrix $T$ proves
\eqref{eq:iid-covariance-result}. Since $m=O(\eta)$, adding
$mm^\top$ also gives the raw second moment needed for the mean
expansion:
\begin{equation}
 \mathbb E\{(z-z^\star)(z-z^\star)^\top\}
 =T^{-1}P_cT^{-\top}+mm^\top
 =\eta\Sigma_\rho+O(\eta^2).
 \label{eq:assumption-iid-raw-second}
\end{equation}

\paragraph{Step 5: Mean Expansion.}
It remains to substitute the moment estimates into
$\mathbb E\bar G(z)=0$. Expand $\bar G$ through degree three
at $z^\star$. Every raw signed third moment is $O(\eta^2)$:
using $z-z^\star=T^{-1}q+m$, it is a linear combination of the
centered third moments in \eqref{eq:assumption-iid-signed-third},
products of $m$ with entries of $P_c$, and products of three
entries of $m$.
The fourth-order remainder is bounded by
$C\mathbb E\|z-z^\star\|^4=O(\eta^2)$, since the fourth
derivatives are bounded. Using \eqref{eq:assumption-iid-raw-second}
for the quadratic term therefore gives
\[
 0=Jm+\eta\mathcal C(\Sigma_\rho)+O(\eta^2).
\]
Multiplication by the fixed matrix $J^{-1}$ proves
\eqref{eq:assumption-iid-refined-mean}, with the required
uniform remainder.
\end{proof}

\section{Proof of localization under uniform block contraction}
\label{app:contractive-localization}

\begin{proposition}[A sufficient localization condition]\label{prop:contractive-localization}
Suppose Assumption~\ref{as:markov} holds, each $G(y,\cdot)$ is globally
$C^2$ with \eqref{eq:cov-regularity}, and there exist $a,d>0$,
$b,c\ge0$, $bc<ad$, such that, uniformly in the noise state $y$ and
all arguments,
\begin{align}
 \langle x-x',H(y,x,\theta)-H(y,x',\theta)\rangle
 &\le-a\|x-x'\|^2,\nonumber\\
 \langle\theta-\theta',F(y,x,\theta)-F(y,x,\theta')\rangle
 &\le-d\|\theta-\theta'\|^2,\nonumber\\
 \|H(y,x,\theta)-H(y,x,\theta')\|
 &\le b\|\theta-\theta'\|,\nonumber\\
 \|F(y,x,\theta)-F(y,x',\theta)\|
 &\le c\|x-x'\|.\label{eq:loc-block-conditions}
\end{align}
Then $\bar G$ has a unique equilibrium and a globally defined $C^2$
fast equilibrium map. For sufficiently small positive $\eta,\rho$,
the joint recursion has a unique invariant law satisfying
Assumptions~\ref{as:localization} and \ref{as:block-localization}, with
$\|\E z-z^\star\|\le C\eta$ uniformly.
\end{proposition}
These conditions verify localization; the remaining assumptions of the
bias theorems remain in force.

\begin{proof}[Proof of Proposition~\ref{prop:contractive-localization}]
All coordinate norms are Euclidean, and matrix norms are the
corresponding operator norms. Constants may depend on the updates,
dimensions, and fixed Markov kernel, but not on sufficiently small
positive $\eta,\rho$. We write $\varepsilon=\rho\eta$ and take
$\eta_0,\rho_0\le1$.

\paragraph{1. Averaged Equilibrium and Deterministic Comparison.}
Averaging preserves \eqref{eq:loc-block-conditions}. For each fixed
$\theta$, global Lipschitz continuity and strong dissipativity make
$x\mapsto x+h\bar H(x,\theta)$ a contraction for sufficiently small
$h>0$. Its unique fixed point defines $\lambda(\theta)$.
For $x=\lambda(\theta)$, $x'=\lambda(\theta')$,
$\delta x=x-x'$, and $\delta\theta=\theta-\theta'$, the equilibrium
identities and the averaged block bounds give
\[
 \begin{aligned}
 0
 &=\langle\delta x,\bar H(x,\theta)-\bar H(x',\theta')\rangle\\
 &\le-a\|\delta x\|^2
      +b\|\delta x\|\,\|\delta\theta\|.
 \end{aligned}
\]
Consequently,
\[
 \|\lambda(\theta)-\lambda(\theta')\|
 \le\frac ba\|\theta-\theta'\|.
\]
Strong dissipativity also makes $D_x\bar H$ everywhere nonsingular.
The implicit function theorem and uniqueness of the fast equilibrium
therefore give a globally defined $C^2$ map $\lambda$.

For the reduced drift $g(\theta)=\bar F(\lambda(\theta),\theta)$,
the same block bounds yield
\[
 \begin{aligned}
 \langle\delta\theta,g(\theta)-g(\theta')\rangle
 &\le-d\|\delta\theta\|^2
       +c\|\delta\theta\|\,
          \|\lambda(\theta)-\lambda(\theta')\|\\
 &\le-\left(d-\frac{cb}{a}\right)\|\delta\theta\|^2.
 \end{aligned}
\]
Thus $g$ is strongly dissipative with constant
$\kappa_g=d-cb/a>0$. If $L_\theta$ bounds
$\|D_\theta\bar F\|$, then
\[
 \|g(\theta)-g(\theta')\|
 \le\left(L_\theta+\frac{cb}{a}\right)\|\theta-\theta'\|.
\]
Writing $L_g=L_\theta+cb/a$, we obtain
\[
 \|\theta+h g(\theta)-\theta'-h g(\theta')\|^2
 \le(1-2\kappa_g h+L_g^2h^2)\|\delta\theta\|^2.
\]
For sufficiently small $h>0$, this is a contraction. Banach's
fixed-point theorem gives a unique zero $\theta^\star$ of $g$,
and $z^\star=(\lambda(\theta^\star),\theta^\star)$ is the unique
equilibrium of $\bar G$.

We next compare two trajectories driven by the same input.
Choose $a_0\in(0,a)$ and $d_0\in(0,d)$ with $bc<a_0d_0$.
If necessary, enlarge $b,c$ slightly so that they are positive
and still satisfy this inequality. Let $L$ be a uniform bound
for the derivatives of both update fields, and write
$\Delta x=x-x'$ and $\Delta\theta=\theta-\theta'$.
For a common noise state $y$, set
$\Delta H_x=H(y,x,\theta)-H(y,x',\theta)$. Then
\[
 \|\Delta x+\eta\Delta H_x\|^2
 \le(1-2a\eta+L^2\eta^2)\|\Delta x\|^2
 \le(1-a_0\eta)^2\|\Delta x\|^2
\]
for sufficiently small $\eta$. Taking square roots and applying
the cross-Lipschitz bound gives
\[
 \|\Delta x_{k+1}\|
 \le(1-a_0\eta)\|\Delta x_k\|
       +b\eta\|\Delta\theta_k\|.
\]
The analogous calculation for the slow update gives
\[
 \|\Delta\theta_{k+1}\|
 \le c\varepsilon\|\Delta x_k\|
       +(1-d_0\varepsilon)\|\Delta\theta_k\|.
\]
Since $\varepsilon\le\eta$, a single sufficiently small $\eta_0$,
independent of $\rho$, ensures both bounds and nonnegative diagonal
coefficients. Hence
\begin{equation}
 \begin{pmatrix}\|\Delta x_{k+1}\|\\
                 \|\Delta\theta_{k+1}\|\end{pmatrix}
 \le M\begin{pmatrix}\|\Delta x_k\|\\
                     \|\Delta\theta_k\|\end{pmatrix},
 \qquad
 M=\begin{pmatrix}
 1-a_0\eta&b\eta\\
 c\varepsilon&1-d_0\varepsilon
 \end{pmatrix}.
 \label{eq:loc-comparison}
\end{equation}
The comparison is componentwise, and $M$ is nonnegative.
Choose $t$ with $c/d_0<t<a_0/b$ and define
\[
 \|(x,\theta)\|_t=\max\{\|x\|,\|\theta\|/t\}.
\]
The corresponding weighted row sums of $M$ are
$1-(a_0-bt)\eta$ and $1-(d_0-c/t)\varepsilon$.
Thus every update map has Lipschitz constant at most
$1-\gamma\varepsilon$, where
$\gamma=\min\{a_0-bt,d_0-c/t\}>0$ is independent of the steps.

\paragraph{2. A Common Invariant Set and the Stationary Law.}
Let
\[
 B_H=\max_\xi\|H(\xi,0,0)\|,
 \qquad B_F=\max_\xi\|F(\xi,0,0)\|,
\]
and choose $R$ large enough that
\[
 (a_0-bt)R\ge B_H,\qquad (d_0t-c)R\ge B_F.
\]
The comparison above shows that the convex compact set
\[
 K=\{(x,\theta):\|x\|\le R,\ \|\theta\|\le tR\}
\]
is invariant under every noise-state update, uniformly in both
step sizes. Starting at zero at time $-N$ along a two-sided
stationary noise sequence and letting $N\to\infty$ gives a limit
$Z_k=(X_k,\Theta_k)\in K$. Contraction ensures convergence;
the limit is measurable with respect to $\sigma(Y_i:i\le k-1)$
and is stationary jointly with $Y_k$.

Any invariant law must be supported on $K$. Comparison with a
trajectory started in $K$ gives
\[
 \operatorname{dist}_t(z_k,K)
 \le(1-\gamma\varepsilon)^k
     \operatorname{dist}_t(z_0,K).
\]
Under an invariant law, the left-hand side has the same distribution
for every $k$, while the right-hand side tends to zero almost surely.
The Doeblin condition allows two copies of the noise to couple in
finite time almost surely. Their iterates contract thereafter,
which proves uniqueness of the joint invariant law.

\paragraph{3. Influence of a Past Input.}
To obtain separate fast and slow moment bounds, we estimate the
effect of changing one past input. Write
\[
 M=I-\eta K_\rho,
 \qquad K_\rho=
 \begin{pmatrix}a_0&-b\\-\rho c&\rho d_0\end{pmatrix}.
\]
Its eigenvalues are
\[
 \ell_\pm=
 \frac{a_0+\rho d_0
 \pm\sqrt{(a_0-\rho d_0)^2+4\rho bc}}{2}.
\]
For sufficiently small $\rho_0$, uniformly over $0<\rho\le\rho_0$,
\[
 \ell_+\asymp1,\quad \ell_-\asymp\rho,\quad
 \ell_+-\ell_-\ge a_0/2,\quad
 \ell_+-a_0=O(\rho).
\]
Decrease $\eta_0$ so that $1-\eta\ell_\pm\in[0,1]$.
The spectral formula
\[
 M^k=
 (1-\eta\ell_+)^k
 \frac{K_\rho-\ell_-I}{\ell_+-\ell_-}
 +(1-\eta\ell_-)^k
 \frac{\ell_+I-K_\rho}{\ell_+-\ell_-}
\]
then gives the entrywise bound
\begin{equation}
 0\le M^k\le C
 \begin{pmatrix}
 e^{-\kappa\eta k}+\rho e^{-\kappa\varepsilon k}
   &e^{-\kappa\varepsilon k}\\
 \rho e^{-\kappa\varepsilon k}
   &e^{-\kappa\varepsilon k}
 \end{pmatrix}
 \label{eq:loc-matrix-powers}
\end{equation}
for some fixed $\kappa>0$.

Since the update fields are bounded on $K$, changing one noise
input changes the next iterate by at most $C(\eta,\varepsilon)$
componentwise. Applying \eqref{eq:loc-matrix-powers} for the
subsequent $k$ updates bounds its effects on $X_0$ and $\Theta_0$
by, respectively,
\begin{equation}
 A_k=C\bigl(\eta e^{-\kappa\eta k}
          +\varepsilon e^{-\kappa\varepsilon k}\bigr),
 \qquad
 B_k=C\varepsilon e^{-\kappa\varepsilon k}.
 \label{eq:loc-influence}
\end{equation}
These bounds hold for arbitrary input sequences because every
update preserves $K$. Telescoping therefore bounds the effect
of several input changes by the sum of their individual bounds.

\paragraph{4. Centered Moments under Markov Dependence.}
We combine the influence bounds with a coupling of the Markov
inputs. The Doeblin condition allows chains started in any two
states to be coupled so that
\begin{equation}
 \mathbb P(Y_\ell\ne Y'_\ell)\le C_Pq^\ell,
 \qquad \ell\ge0,
 \label{eq:loc-noise-coupling}
\end{equation}
for fixed $q\in(0,1)$ and $C_P<\infty$. Couple the endpoints of
successive blocks of length $n_0$ using the common minorizing
measure, and use identical transitions after coupling.
Sampling intermediate paths conditionally on their endpoints
gives \eqref{eq:loc-noise-coupling} at all integer times.

Set $\mathcal F_j=\sigma(Y_i:i\le j)$, $j\le-1$, and define
\[
 D_j^X=\mathbb E[X_0\mid\mathcal F_j]
       -\mathbb E[X_0\mid\mathcal F_{j-1}],
 \qquad
 D_j^\Theta=\mathbb E[\Theta_0\mid\mathcal F_j]
       -\mathbb E[\Theta_0\mid\mathcal F_{j-1}].
\]
Conditionally on $\mathcal F_{j-1}$, $Z_j$ is fixed. Comparing
two possible values of $Y_j$ by coupling their future chains
and telescoping the resulting input changes gives, for $j=-k-1$,
\[
 \|D_{-k-1}^X\|
 \le C_P\sum_{\ell=0}^kq^\ell A_{k-\ell},\qquad
 \|D_{-k-1}^\Theta\|
 \le C_P\sum_{\ell=0}^kq^\ell B_{k-\ell}.
\]
Decrease $\eta_0$ so that $qe^{\kappa\eta_0}<1$.
The convolutions are then bounded by constant multiples of
$A_k$ and $B_k$, uniformly in both steps. Enlarging the constant
in \eqref{eq:loc-influence}, we obtain
\[
 \|D_{-k-1}^X\|\le A_k,\qquad
 \|D_{-k-1}^\Theta\|\le B_k.
\]
Geometric summation gives
\begin{equation}
 \sum_{k\ge0}A_k^2\le C\eta,\qquad
 \sum_{k\ge0}B_k^2\le C\varepsilon,\qquad
 \sum_{k\ge0}A_kB_k\le C\varepsilon.
 \label{eq:loc-influence-sums}
\end{equation}

The remote past tail of the stationary Doeblin chain is trivial.
Reverse martingale convergence and the measurability of $Z_0$
with respect to $\mathcal F_{-1}$ therefore give the $L^2$ expansions
\[
 X_0-\mathbb EX_0=\sum_{j=-\infty}^{-1}D_j^X,
 \qquad
 \Theta_0-\mathbb E\Theta_0
 =\sum_{j=-\infty}^{-1}D_j^\Theta.
\]
For each scalar coordinate, conditional Hoeffding's inequality
applied to finite partial sums gives a sub-Gaussian moment
generating function with variance proxy bounded by $\sum A_k^2$
or $\sum B_k^2$. Passing to the limit and integrating the tail
bound yields, for each fixed integer $j\ge1$,
\begin{equation}
 \mathbb E\|X_0-\mathbb EX_0\|^{2j}\le C_j\eta^j,
 \qquad
 \mathbb E\|\Theta_0-\mathbb E\Theta_0\|^{2j}
 \le C_j\varepsilon^j.
 \label{eq:loc-centered-moments}
\end{equation}
The passage to Euclidean norms changes only dimension-dependent
constants. Orthogonality of martingale increments at distinct times
also gives
\begin{equation}
 \|\operatorname{Cov}(X_0,\Theta_0)\|
 \le\sum_{k\ge0}
 \mathbb E\bigl[\|D_{-k-1}^X\|
                 \|D_{-k-1}^\Theta\|\bigr]
 \le C\varepsilon.
 \label{eq:loc-cross-covariance}
\end{equation}

\paragraph{5. Location of the Stationary Mean.}
We now locate the stationary mean relative to $z^\star$, using
stationarity and the Poisson equation. Let $\bar z=\mathbb EZ_0$ and
$\widetilde G(\xi,z)=G(\xi,z)-\bar G(z)$. The equation
\[
 (I-P)\mathcal V(\cdot,z)=\widetilde G(\cdot,z),
 \qquad \sum_\xi\mu(\xi)\mathcal V(\xi,z)=0
\]
has a solution with bounded first $z$-derivative on $K$, since
the inverse of $I-P$ on the $\mu$-centered subspace is a fixed
finite-dimensional operator. Exogeneity and joint stationarity give
\begin{align*}
 \mathbb E\widetilde G(Y_0,Z_0)
 &=\mathbb E\{\mathcal V(Y_0,Z_0)
                    -\mathcal V(Y_1,Z_0)\}\\
 &=\mathbb E\{\mathcal V(Y_1,Z_1)
                    -\mathcal V(Y_1,Z_0)\}
 =O(\eta).
\end{align*}
The last bound uses convexity of $K$ and
$\|Z_1-Z_0\|\le C\eta$ on $K$.
Stationary increment balance gives $\mathbb EG(Y_0,Z_0)=0$.
Taylor expansion at $\bar z$, bounded second derivatives, and
\eqref{eq:loc-centered-moments} yield
\[
 \|\mathbb E\bar G(Z_0)-\bar G(\bar z)\|
 \le C\mathbb E\|Z_0-\bar z\|^2\le C\eta.
\]
Consequently, $\|\bar G(\bar z)\|\le C\eta$.

For any $w=(x,\theta)$, set $r=\|x-x^\star\|$ and
$s=\|\theta-\theta^\star\|$. The averaged block bounds give
\[
 ar\le bs+\|\bar H(w)\|,\qquad
 ds\le cr+\|\bar F(w)\|.
\]
Since $ad-bc>0$, these inequalities imply
\[
 \|w-z^\star\|\le C\|\bar G(w)\|.
\]
Applying this with $w=\bar z$ proves
\begin{equation}
 \|\mathbb EZ_0-z^\star\|\le C\eta.
 \label{eq:loc-mean-bound}
\end{equation}

\paragraph{6. Localization in the Exact and Tangent Coordinates.}
Combining \eqref{eq:loc-centered-moments} and
\eqref{eq:loc-mean-bound}, for $j=1,2,3$,
\begin{align*}
 \mathbb E\|X_0-x^\star\|^{2j}&\le C_j\eta^j,\\
 \mathbb E\|\Theta_0-\theta^\star\|^{2j}
 &\le C_j(\varepsilon^j+\eta^{2j}).
\end{align*}
For $\widehat u=X_0-\lambda(\Theta_0)$ and
$v=\Theta_0-\theta^\star$, the Lipschitz bound on $\lambda$ gives
\[
 \|\widehat u\|
 \le\|X_0-x^\star\|+(b/a)\|\Theta_0-\theta^\star\|.
\]
Thus $\mathbb E\|\widehat u\|^{2j}\le C_j\eta^j$.
The inequality
\[
 \mathbb E\|W-\mathbb EW\|^p\le2^p\mathbb E\|W\|^p,
 \qquad p\ge1,
\]
gives the centered fast tracking bounds. The centered slow bounds
are already contained in \eqref{eq:loc-centered-moments}.

In tangent coordinates, $u=X_0-x^\star-\Lambda v$, with
$\Lambda=D\lambda(\theta^\star)$. Hence
\begin{align*}
 \operatorname{Cov}(u,v)
 &=\operatorname{Cov}(X_0,\Theta_0)
   -\Lambda\operatorname{Cov}(\Theta_0,\Theta_0),\\
 \|\operatorname{Cov}(u,v)\|&\le C\varepsilon,
\end{align*}
by \eqref{eq:loc-cross-covariance} and the centered slow
second-moment bound. The raw fourth moments and
$\varepsilon\le\eta\le1$ also give
\[
 \mathbb E\|Z_0-z^\star\|^4\le C\eta^2,
\]
which verifies Assumption~\ref{as:localization}.

To control the coordinate remainder, differentiate
$\bar H(\lambda(\theta),\theta)=0$ twice. For
$h_1,h_2\in\mathbb R^{d_\theta}$,
\[
\begin{aligned}
 D^2\lambda(\theta)[h_1,h_2]
 ={}&-[D_x\bar H(\lambda(\theta),\theta)]^{-1}\\
 &\quad\times D^2\bar H(\lambda(\theta),\theta)
 [(D\lambda(\theta)h_1,h_1),
  (D\lambda(\theta)h_2,h_2)].
\end{aligned}
\]
Strong dissipativity gives
$\|[D_x\bar H(\lambda(\theta),\theta)]^{-1}\|\le a^{-1}$,
and $\|D\lambda(\theta)\|\le b/a$. Together with
\eqref{eq:cov-regularity}, these bounds make $D^2\lambda$
globally bounded. Therefore
$\|r_\lambda(v)\|\le C\|v\|^2$, and the raw fourth-moment
bound yields
\[
 \E\|r_\lambda(v)\|\le C\eta,
 \qquad
 \E\|r_\lambda(v)\|^2\le C\eta^2.
\]
The global $C^2$ map supplies the required extension, completing
the verification of Assumption~\ref{as:block-localization}.
\end{proof}

\section{A nonlinear Markov example with an exact bias formula}
\label{sec:sup:exact-markov-example}

A shared two-state Markov input drives fast and slow linear filters,
and a nonlinear readout converts their variance and cross-covariance
into stationary bias. The exact mean contains a nonzero
$\varepsilon^2/\eta$ term and allows a direct comparison of the
extrapolation rules in Section~\ref{sec:rr}.

\subsection{Model and Markov Input}

Let \(\xi_k\in\{-1,+1\}\) be a stationary symmetric Markov chain with
transition matrix, in the state order \((-1,+1)\),
\begin{equation}
 P_q=\begin{pmatrix}
 (1+q)/2&(1-q)/2\\
 (1-q)/2&(1+q)/2
 \end{pmatrix},
 \qquad -1<q<1.
 \label{eq:sup:report-example-transition}
\end{equation}
The invariant distribution is uniform, and
$\mathbb E(\xi_{k+1}\mid\xi_k)=q\xi_k$. Consequently,
\begin{equation}
 \mathbb E\xi_k=0,\qquad
 \mathbb E(\xi_k\xi_{k+h})=q^{|h|}.
 \label{eq:sup:report-example-correlation}
\end{equation}
The case $q=0$ gives independent symmetric signs. The scalar
long-run covariance is
\begin{equation}
 c_q=\sum_{h\in\mathbb Z}\mathbb E(\xi_0\xi_h)
 =1+2\sum_{h=1}^\infty q^h
 =\frac{1+q}{1-q}.
 \label{eq:sup:report-example-cq}
\end{equation}

Consider
\begin{align}
 X_{k+1}&=(1-\eta)X_k+\eta\xi_{k+1},
 \label{eq:sup:report-example-X}\\
 Y_{k+1}&=(1-\varepsilon)Y_k+\varepsilon\xi_{k+1},
 \label{eq:sup:report-example-Y}\\
 Z_{k+1}&=Z_k+\varepsilon\{-Z_k+\phi(X_k,Y_k)\}.
 \label{eq:sup:report-example-Z}
\end{align}
The fast coordinate is $X$, and the slow coordinates are $(Y,Z)$.
The filters $X$ and $Y$ share the same input. We take
$0<\varepsilon\le\eta\le1$ and choose a smooth cutoff
$\chi:\mathbb R\to[0,1]$ equal to one on $[-1,1]$ and zero
outside $(-2,2)$. Define
\begin{equation}
 \phi(x,y)=\chi(x)\chi(y)(x^2+xy).
 \label{eq:sup:report-example-cutoff}
\end{equation}
Thus $\phi$ and all its derivatives are globally bounded.
The stationary filters lie in $[-1,1]^2$, as shown below, so the
cutoff leaves their readout equal to $X^2+XY$.

The averaged drift is
\[
 \bar G(x,y,z)=(-x,-y,-z+\phi(x,y))^\top.
\]
Its unique equilibrium is the origin. The fast equilibrium map is
$\lambda\equiv0$, and the reduced slow drift is $(-y,-z)^\top$.

\subsection{Stationary Representations}

Set
\begin{equation}
 a=1-\eta,\qquad d=1-\varepsilon.
 \label{eq:sup:report-example-ad}
\end{equation}
Backward iteration gives the unique stationary causal filters
\begin{equation}
 X_k=\eta\sum_{i=0}^\infty a^i\xi_{k-i},
 \qquad
 Y_k=\varepsilon\sum_{j=0}^\infty d^j\xi_{k-j}.
 \label{eq:sup:report-example-causal-filters}
\end{equation}
Their weights are nonnegative and sum to one, so
$|X_k|,|Y_k|\le1$ almost surely. Their means are zero by
\eqref{eq:sup:report-example-correlation}. The stationary slow
readout is
\[
 Z_k=\varepsilon\sum_{j\ge0}d^j
                  \phi(X_{k-1-j},Y_{k-1-j}).
\]
Since $|\phi(x,y)|\le2$ on $[-1,1]^2$, this also gives
$|Z_k|\le2$.

\subsection{Exact Second Moments}

Define the correlations with the current driver by
\begin{equation}
 r_X=\mathbb E(X_k\xi_k),
 \qquad
 r_Y=\mathbb E(Y_k\xi_k).
 \label{eq:sup:report-example-rXY}
\end{equation}
The Markov property gives
$\mathbb E(X_{k-1}\xi_k)=qr_X$ and
$\mathbb E(Y_{k-1}\xi_k)=qr_Y$. Hence
$r_X=aqr_X+\eta$ and $r_Y=dqr_Y+\varepsilon$, so
\begin{equation}
 r_X=\frac{\eta}{1-aq},
 \qquad
 r_Y=\frac{\varepsilon}{1-dq}.
 \label{eq:sup:report-example-rXY-solved}
\end{equation}
Squaring the fast recursion and taking stationary expectations yields
\[
 (1-a^2)\mathbb EX^2=2a\eta q r_X+\eta^2.
\]
Substituting \eqref{eq:sup:report-example-rXY-solved} gives
\begin{equation}
 \boxed{
 \mathbb EX^2
 =\frac{\eta^2}{1-a^2}\frac{1+aq}{1-aq}.}
 \label{eq:sup:report-example-var-X}
\end{equation}
The same calculation for the slow filter gives
\begin{equation}
 \boxed{
 \mathbb EY^2
 =\frac{\varepsilon^2}{1-d^2}\frac{1+dq}{1-dq}.}
 \label{eq:sup:report-example-var-Y}
\end{equation}
Multiplying the two filter recursions and taking expectations gives
\begin{align*}
 (1-ad)\mathbb E(XY)
 &=a\varepsilon q r_X+d\eta q r_Y+\eta\varepsilon\\
 &=\eta\varepsilon
 \frac{1-adq^2}{(1-aq)(1-dq)}.
\end{align*}
Therefore
\begin{equation}
 \boxed{
 \mathbb E(XY)
 =\frac{\eta\varepsilon(1-adq^2)}
 {(1-ad)(1-aq)(1-dq)}.}
 \label{eq:sup:report-example-cross}
\end{equation}
For independent observations, $q=0$, this reduces to
\[
 \mathbb E(XY)
 =\frac{\eta\varepsilon}
 {\eta+\varepsilon-\eta\varepsilon}.
\]

\subsection{Stationary Mean}

Stationarity of \eqref{eq:sup:report-example-Z} gives
$\mathbb EZ=\mathbb E\phi(X,Y)$. Since the cutoff is inactive on
the stationary support,
\begin{equation}
 \boxed{
 \mathbb EZ=\mathbb EX^2+\mathbb E(XY).}
 \label{eq:sup:report-example-mean-Z}
\end{equation}
The full mean displacement is $(0,0,\mathbb EZ)$.
The two quadratic terms contribute the fast variance and the
fast--slow cross-covariance, respectively.

\subsection{Fixed-Ratio Bias}

Set $\varepsilon=\rho\eta$ and first let $\eta\to0$ at fixed
$\rho$. The exact fast variance is
\begin{equation}
 \mathbb EX^2
 =\frac{\eta}{2-\eta}
   \frac{1+q-q\eta}{1-q+q\eta}.
 \label{eq:sup:report-example-var-rational}
\end{equation}
Taylor expansion at $\eta=0$ gives
\begin{equation}
 \mathbb EX^2=\frac{c_q}{2}\eta+O(\eta^2).
 \label{eq:sup:report-example-var-expansion}
\end{equation}
Similarly, substituting $a=1-\eta$ and $d=1-\rho\eta$ into
\eqref{eq:sup:report-example-cross} gives
\begin{equation}
 \mathbb E(XY)
 =\rho\eta\,
 \frac{1-q^2+q^2(1+\rho)\eta-q^2\rho\eta^2}
 {(1+\rho-\rho\eta)(1-q+q\eta)(1-q+q\rho\eta)}.
 \label{eq:sup:report-example-cross-rational}
\end{equation}
For fixed $q\in(-1,1)$, the rational factor is smooth near
$(\eta,\rho)=(0,0)$, so
\[
 \mathbb E(XY)
 =c_q\eta\frac{\rho}{1+\rho}+O(\eta^2\rho).
\]
Combining the two moments yields
\begin{equation}
 \mathbb EZ
 =c_q\eta\left\{\frac12+\frac{\rho}{1+\rho}\right\}
 +O(\eta^2),
 \label{eq:sup:report-example-fixed-ratio}
\end{equation}
uniformly for sufficiently small positive $\rho$. Thus the leading
fixed-ratio coefficient is
\begin{equation}
 b_Z(\rho)=c_q\left\{\frac12+\frac{\rho}{1+\rho}\right\}.
 \label{eq:sup:report-example-b-rho}
\end{equation}

\subsection{Singular Expansion and Coefficient Recovery}

Using
\begin{equation}
 \frac{\rho}{1+\rho}
 =\rho-\rho^2+\frac{\rho^3}{1+\rho}
 \label{eq:sup:report-example-rho-identity}
\end{equation}
in \eqref{eq:sup:report-example-fixed-ratio} gives
\begin{equation}
 \mathbb EZ
 =\frac{c_q}{2}\eta+c_q\eta\rho-c_q\eta\rho^2
 +O(\eta\rho^3+\eta^2).
 \label{eq:sup:report-example-rho-series}
\end{equation}
Returning to the two step sizes,
\begin{equation}
 \boxed{
 \mathbb EZ
 =\frac{c_q}{2}\eta+c_q\varepsilon
 -c_q\frac{\varepsilon^2}{\eta}
 +O\!\left(\eta^2+\frac{\varepsilon^3}{\eta^2}\right).}
 \label{eq:sup:report-example-singular-expansion}
\end{equation}
The vector coefficients are therefore
\begin{equation}
 b_0=\begin{pmatrix}0\\0\\c_q/2\end{pmatrix},\qquad
 b_1=\begin{pmatrix}0\\0\\c_q\end{pmatrix},\qquad
 b_2=\begin{pmatrix}0\\0\\-c_q\end{pmatrix}.
 \label{eq:sup:report-example-vector-coefficients}
\end{equation}
In particular, the mixed coefficient is nonzero for every fixed
$q\in(-1,1)$.

Write $b_i^Z$ for the third component of $b_i$.
The coefficients can also be recovered algebraically through
\begin{align}
 \widehat b_0&=\frac{\mathbb EZ}{\eta}
 =b_0^{Z}+O(\rho+\eta),\nonumber\\
 \widehat b_1&=\frac{\mathbb EZ-\eta b_0^{Z}}{\varepsilon}
 =b_1^{Z}+O\!\left(\rho+\frac{\eta}{\rho}\right),\nonumber\\
 \widehat b_2&=
 \frac{\mathbb EZ-\eta b_0^{Z}-\varepsilon b_1^{Z}}
 {\varepsilon^2/\eta}
 =b_2^{Z}+O\!\left(\rho+\frac{\eta}{\rho^2}\right).
 \label{eq:sup:report-example-sequential-recovery}
\end{align}
Along $\rho=\eta^{p-1}$, these bounds guarantee convergence of
all three quantities for $1<p<3/2$. The upper threshold is sharp
when the regular $\eta^2$ coefficient is nonzero.

\subsection{Extrapolation along Power-law Paths}

For fixed $p>1$, setting $\varepsilon=\eta^p$ in
\eqref{eq:sup:report-example-singular-expansion}
gives
\begin{equation}
 \mathbb EZ
 =\frac{c_q}{2}\eta+c_q\eta^p-c_q\eta^{2p-1}
 +O(\eta^{3p-2}+\eta^2).
 \label{eq:sup:report-example-path-expansion}
\end{equation}
For $1<p<3/2$, the mixed term precedes the regular quadratic
remainder. More generally, analyticity of the exact formula in
$(\eta,\rho)$ gives path powers $k+j(p-1)$, with $k\ge1$ and
$j\ge0$. Contributions with the same exponent must be combined.

Applying the extrapolation weights from
Appendix~\ref{sec:sup:singular-path-rr} to the expansion above
gives the following cancellation properties.
For $1<p<3/2$, cancelling the integer powers $\{1,2\}$ leaves
the $\eta^p$ term, whereas cancelling $\{1,p\}$ leaves bias of
order $\eta^{2p-1}$. Cancelling $\{1,p,2p-1\}$ gives the bound
$O(\eta^{\min\{3p-2,2\}})$. This last exponent is attained when
the corresponding combined coefficient is nonzero; additional
cancellation can give a faster rate.

\subsection{Agreement with the General Bias Formula}
\label{subsec:sup:example-consistency}

To apply \eqref{eq:sup:report-fixed-ratio-bias}, reindex the driver
as $s_k=\xi_{k+1}$. The update then has the form $G(s_k,z_k)$,
where, for $z=(x,y,z_3)$,
\[
 G(s,z)=\begin{pmatrix}-x+s\\-y+s\\-z_3+\phi(x,y)\end{pmatrix},
 \qquad
 \bar G(z)=\begin{pmatrix}-x\\-y\\-z_3+\phi(x,y)\end{pmatrix}.
\]
At the origin,
\begin{equation}
 J=-I_3,\qquad
 D_\rho=\operatorname{diag}(1,\rho,\rho).
 \label{eq:sup:report-example-J}
\end{equation}
Let $\ell=(1,1,0)^\top$. The centered update is $\ell s$, and
$P_qs=qs$, so the centered Poisson solution is
\begin{equation}
 \mathcal U(s,z)=\frac{\ell s}{1-q}.
 \label{eq:sup:report-example-Poisson}
\end{equation}
It is independent of $z$, hence
\begin{equation}
 r_\rho=0.
 \label{eq:sup:report-example-r-zero}
\end{equation}
The temporal dependence enters through the long-run covariance
\begin{equation}
 Q=c_q\ell\ell^\top
 =c_q\begin{pmatrix}1&1&0\\1&1&0\\0&0&0\end{pmatrix}.
 \label{eq:sup:report-example-Q}
\end{equation}
The nonzero entries of the Lyapunov solution satisfy
\begin{align*}
 -2(\Sigma_\rho)_{11}+c_q&=0,\\
 -(1+\rho)(\Sigma_\rho)_{12}+\rho c_q&=0,\\
 -2\rho(\Sigma_\rho)_{22}+\rho^2c_q&=0.
\end{align*}
Thus
\begin{equation}
 \Sigma_\rho=c_q
 \begin{pmatrix}
 1/2&\rho/(1+\rho)&0\\
 \rho/(1+\rho)&\rho/2&0\\
 0&0&0
 \end{pmatrix}.
 \label{eq:sup:report-example-Sigma}
\end{equation}
The only nonzero Hessian of the averaged drift at the origin is
\[
 D^2\bar G_3(0)=
 \begin{pmatrix}2&1&0\\1&0&0\\0&0&0\end{pmatrix}.
\]
Its contraction with the covariance gives
\begin{equation}
 \mathcal C(\Sigma_\rho)
 =c_q\begin{pmatrix}0\\0\\
 \dfrac12+\dfrac{\rho}{1+\rho}
 \end{pmatrix}.
 \label{eq:sup:report-example-curvature}
\end{equation}
Since $-J^{-1}=I_3$ and $r_\rho=0$, the general formula reproduces
\eqref{eq:sup:report-example-b-rho}. If the filters instead use
independent centered drivers, their cross-covariance vanishes and
the leading bias comes only from the fast variance.

The mixing and coupling estimates are uniform when $q$ ranges over
a compact subset of $(-1,1)$. Mixing slows as $q$ approaches either
endpoint: $q\uparrow1$ also makes $c_q$ diverge, while
$q\downarrow-1$ approaches periodic alternation. The exact rational
mean formulas remain regular at $q=-1$, although the uniform mixing
conditions used in the general theory no longer hold there.

\subsection{Stationary Localization and Coupling Verification}
\label{app:example-verification}

For this model, $\widehat u=X$, $v=(Y,Z)$, and $\lambda\equiv0$.
We verify the fourth-moment condition in Assumption~\ref{as:localization}
and the additional block conditions in Assumption~\ref{as:block-localization}:
\begin{align}
 \E_{\pi_{\eta,\rho}}\norm{\widehat u}^{2j}&\le C\eta^j,
 &\E_{\pi_{\eta,\rho}}\norm{v}^{2j}&\le C\varepsilon^j+C\eta^{2j},
 \nonumber\\
 \E_{\pi_{\eta,\rho}}\norm{\widehat u-\E\widehat u}^{2j}
 &\le C\eta^j,
 &\E_{\pi_{\eta,\rho}}\norm{v-\E v}^{2j}
 &\le C\varepsilon^j,
 \qquad j=1,2,3,
 \label{eq:sup:assumption-block-moments}
\end{align}
Fix $-1<q<1$ and $0<\varepsilon\le\eta\le\eta_0\le1$.
Constants may depend on $q$, but not on the step sizes.

\paragraph{Independent-Input Representation and Sensitivity.}

Put $r=|q|$ and choose $\sigma\in\{-1,1\}$ with $q=r\sigma$,
taking $\sigma=1$ when $q=0$. Let $W_k=(B_k,\omega_k)$, $k\in\mathbb Z$,
be independent pairs, with $B_k\sim\operatorname{Bernoulli}(r)$
and $\omega_k$ an independent symmetric sign. The recursion
\[
 \xi_k=B_k\sigma\xi_{k-1}+(1-B_k)\omega_k
\]
defines a stationary causal chain: almost surely, a reset $B_k=0$
occurs at a finite time in the past. Its invariant distribution is
uniform and its conditional mean is $q\xi_{k-1}$, so its transition
matrix is $P_q$. The stationary representations above give
$|X_k|,|Y_k|\le1$ and $|Z_k|\le2$.

Replace $W_0$ by an independent copy, keep all other inputs fixed,
and denote the resulting processes by primes. The drivers agree
after their first common reset at a positive time. For
$p\in\{2,4,6\}$, set $g_0=1$ and $g_j=r^{j/p}$ for $j\ge1$.
Then $\|\xi_j-\xi'_j\|_{L^p}\le2g_j$. Write
$\delta^A_{j,p}=\|A_j-A'_j\|_{L^p}$ and let $*$ denote convolution
on the nonnegative integers. With
\[
 a_j=\eta(1-\eta)^j,\qquad d_j=\varepsilon(1-\varepsilon)^j,
\]
the filter representations and Minkowski's inequality give
\[
 \delta^X_{\cdot,p}\le2a*g,\qquad
 \delta^Y_{\cdot,p}\le2d*g.
\]
On the invariant square, $\phi(x,y)=x^2+xy$ and
\[
 |\phi(x,y)-\phi(x',y')|\le3|x-x'|+|y-y'|.
\]
Thus $\delta^\phi\le3\delta^X+\delta^Y$ and
$\delta^Z\le\mathcal S(d*\delta^\phi)$, where
$(\mathcal S h)_0=0$ and $(\mathcal S h)_j=h_{j-1}$ for $j\ge1$.
This shift accounts for the use of $(X_k,Y_k)$ in the update of
$Z_{k+1}$ and preserves every sequence norm used below. Since
\[
 \|a\|_{\ell^1}=\|d\|_{\ell^1}=1,\quad
 \|a\|_{\ell^2}^2=\frac{\eta}{2-\eta},\quad
 \|d\|_{\ell^2}^2=\frac{\varepsilon}{2-\varepsilon},\quad
 \|g\|_{\ell^1}=\frac1{1-r^{1/p}},
\]
Young's convolution inequality yields
\[
 \|\delta^X_{\cdot,p}\|_{\ell^2}\le C_p\sqrt\eta,\qquad
 \|\delta^Y_{\cdot,p}\|_{\ell^2}\le C_p\sqrt\varepsilon,\qquad
 \|\delta^\phi_{\cdot,p}\|_{\ell^1}\le C_p,\qquad
 \|\delta^Z_{\cdot,p}\|_{\ell^2}\le C_p\sqrt\varepsilon.
\]

\paragraph{Stationary Moments and Cross-Covariance.}
Let $\mathcal H_i=\sigma(W_j:j\le i)$ and
$\mathcal P_i A_0=\E(A_0\mid\mathcal H_i)-\E(A_0\mid\mathcal H_{i-1})$.
Each stationary coordinate is bounded and causal, and the independent-input
tail sigma-field is trivial. Consequently,
\[
 A_0-\E A_0=\sum_{i\le0}\mathcal P_i A_0
 \quad\text{in }L^p.
\]
Replacing $W_i$ by an independent copy and using conditional Jensen's
inequality gives $\|\mathcal P_{-j}A_0\|_{L^p}\le\delta^A_{j,p}$.
Orthogonality for $p=2$, and the martingale square-function inequality
followed by Minkowski's inequality for $p=4,6$, imply
\[
 \|A_0-\E A_0\|_{L^p}
 \le C_p\left(\sum_{j\ge0}\|\mathcal P_{-j}A_0\|_{L^p}^2\right)^{1/2}
 \le C_p\|\delta^A_{\cdot,p}\|_{\ell^2}.
\]
The centered $p$th moments are therefore $O(\eta^{p/2})$ for $X$
and $O(\varepsilon^{p/2})$ for $Y,Z$. Moreover, $\E X=\E Y=0$ and
\[
 \E Z=\E X^2+\E(XY)=O(\eta),
\]
by stationarity, Cauchy--Schwarz, and $\varepsilon\le\eta$.
Combining these estimates with
$\E|A|^p\le2^{p-1}\{\E|A-\E A|^p+|\E A|^p\}$ proves
\eqref{eq:sup:assumption-block-moments}.

The sensitivity bounds also give
\[
 \|\delta^X_{\cdot,2}\|_{\ell^1}\le C,\qquad
 \|\delta^Y_{\cdot,2}\|_{\ell^\infty}
 \le2\|d\|_{\ell^\infty}\|g\|_{\ell^1}\le C\varepsilon,\qquad
 \|\delta^Z_{\cdot,2}\|_{\ell^\infty}
 \le\|d\|_{\ell^\infty}\|\delta^\phi_{\cdot,2}\|_{\ell^1}
 \le C\varepsilon.
\]
Orthogonality of projections at distinct times therefore yields, for $A=Y,Z$,
\[
 |\Cov(X,A)|
 \le\sum_{j\ge0}\delta^X_{j,2}\delta^A_{j,2}
 \le\|\delta^X_{\cdot,2}\|_{\ell^1}
      \|\delta^A_{\cdot,2}\|_{\ell^\infty}
 \le C\varepsilon.
\]
Since $\lambda=0$ globally, $u=\widehat u=X$ and $r_\lambda=0$.
This verifies the remaining covariance, extension, and coordinate-remainder
conditions in Assumption~\ref{as:block-localization}. The raw fourth moments give
\[
 \E\|(X,Y,Z)\|^4=O(\eta^2),
\]
verifying Assumption~\ref{as:localization}. The cutoff makes all derivatives
of $\phi$ bounded, and the other drift terms are linear. Hence the updates
satisfy the global derivative bounds, and
Proposition~\ref{prop:covariance} applies to this example.

\paragraph{Coupling from a Prescribed Initial State.}
The cutoff makes $\phi$ globally Lipschitz; let $L_\phi$ be a Lipschitz
constant. Couple a trajectory from any deterministic $(X_0,Y_0,Z_0)$
and initial observed sign $s_0=\xi_1$ to the stationary trajectory.
Use common reset variables $(B_i,\omega_i)$ from time $i=2$ onward.
The initial observed signs may differ, but the drivers agree after
their first common reset. Thus
\[
 \|\xi_i-\xi_i^{\rm st}\|_{L^2}\le2r^{(i-1)/2},\qquad i\ge1,
\]
with the bound at $i=1$ equal to $2$, including when $r=0$.
Put $D_{X,k}=\|X_k-X_k^{\rm st}\|_{L^2}$ and define
$D_{Y,k},D_{Z,k}$ similarly. The filter equations give
\begin{align*}
 D_{X,k}&\le(1-\eta)^kD_{X,0}
       +2\eta\sum_{i=1}^k(1-\eta)^{k-i}r^{(i-1)/2},\\
 D_{Y,k}&\le(1-\varepsilon)^kD_{Y,0}
       +2\varepsilon\sum_{i=1}^k(1-\varepsilon)^{k-i}r^{(i-1)/2}.
\end{align*}
Choose $\eta_0$ so that $1-\eta_0>\sqrt r$. Geometric summation gives
\[
 D_{X,k}\le Ce^{-\eta k}(D_{X,0}+\eta),\qquad
 D_{Y,k}\le Ce^{-\varepsilon k}(D_{Y,0}+\varepsilon).
\]
The Lipschitz bound also gives
\[
 D_{Z,k}\le(1-\varepsilon)^kD_{Z,0}
 +\varepsilon L_\phi\sum_{j=0}^{k-1}(1-\varepsilon)^{k-1-j}
                   (D_{X,j}+D_{Y,j}).
\]
For $0<\rho\le\rho_0<1$, the fast-to-slow convolution is bounded by
$C\rho(e^{-c\eta k}+e^{-c\varepsilon k})$. The slow-to-slow convolution
is bounded by $Ce^{-c\varepsilon k}$ after decreasing $c>0$ to absorb
the factor $k\varepsilon$. Writing $F_k=e^{-c\eta k}$ and
$S_k=e^{-c\varepsilon k}$, we obtain
\begin{align*}
 D_{X,k}&\le CF_kA_u,\\
 D_{Y,k}+D_{Z,k}&\le C\{\rho(F_k+S_k)A_u+S_kA_v\},\\
 A_u&=D_{X,0}+\eta,\qquad
 A_v=D_{Y,0}+D_{Z,0}+\varepsilon.
\end{align*}
The bounded stationary coordinates ensure
$A_u^2+A_v^2\le C(1+\|(X_0,Y_0,Z_0)\|^2)$.
Because $\lambda=0$ globally, these imply the exact-coordinate block
$L^2$ coupling bounds with $V_2(z)=\|z\|^2$.
Cauchy--Schwarz gives Assumption~\ref{as:coupling} with $V(z)=\|z\|$;
both size functions have uniformly bounded stationary expectations.

\section{From the stationary calculation to finite time}
\label{sec:sup:finite-time}

The stationary block scales and mean-square bounds below use
Assumptions~\ref{as:smooth}--\ref{as:taylor} together with the additional
localization condition stated next. Constants are independent of
$0<\eta\le\eta_0$ and $0<\rho\le\rho_0\le1$, with
$\varepsilon=\eta\rho$. Stationary expectations are taken under the
selected joint invariant law $\pi_{\eta,\rho}$.

\begin{assumption}[Additional block localization]
\label{as:block-localization}
For the selected invariant laws in Assumption~\ref{as:localization},
the local map $\lambda$ has a $C^2$ extension to a neighborhood
containing the stationary slow states. In the resulting exact
coordinates $\widehat u=x-\lambda(\theta)$ and
$v=\theta-\theta^\star$, suppose that, for $j=1,2,3$,
\begin{align}
 \E\|\widehat u\|^{2j}&\le C\eta^j,
 &\E\|v\|^{2j}&\le C(\varepsilon^j+\eta^{2j}),\nonumber\\
 \E\|\widehat u-\E\widehat u\|^{2j}&\le C\eta^j,
 &\E\|v-\E v\|^{2j}&\le C\varepsilon^j.
 \label{eq:moments}
\end{align}
In tangent coordinates, $\|\Cov(u,v)\|\le C\varepsilon$.
The coordinate remainder
$r_\lambda(v)=\lambda(\theta^\star+v)-x^\star-\Lambda v$
satisfies
\begin{equation}
 \E\|r_\lambda(v)\|\le C\eta,
 \qquad
 \E\|r_\lambda(v)\|^2\le C\eta^2.
 \label{eq:lambda-remainder-bounds}
\end{equation}
All constants are independent of sufficiently small positive
$\eta,\rho$.
\end{assumption}

Appendices~\ref{app:contractive-localization} and
\ref{app:example-verification} verify this assumption under their
respective conditions. The main covariance and bias expansions do not
require it.

\subsection{Stationary Block Scales}

The exact and tangent tracking errors satisfy
$u=\widehat u+r_\lambda(v)$. By \eqref{eq:moments} and
\eqref{eq:lambda-remainder-bounds},
\[
 \E\|u-\E u\|^2
 \le2\E\|\widehat u-\E\widehat u\|^2
      +2\E\|r_\lambda-\E r_\lambda\|^2
 =O(\eta).
\]
The reverse comparison follows from $\widehat u=u-r_\lambda(v)$,
and the slow coordinate is unchanged. Thus
\begin{align}
 \E\norm{u-\E u}^2&=O(\eta),
 &\E\norm{v-\E v}^2&=O(\varepsilon),
 \label{eq:sup:finite-centered-scales}\\
 \norm{\E(u,v)}&=O(\eta).
 \label{eq:sup:finite-mean-scale}
\end{align}
For the mean estimate, Lemma~\ref{lem:mean-O-eta} gives
$\E(z-z^\star)=O(\eta)$, and $(u,v)=T(z-z^\star)$ with fixed $T$.
The expected-remainder bound also gives $\E\widehat u=O(\eta)$.
If $D^2\lambda$ is bounded on the segments from $\theta^\star$ to
stationary slow states, Taylor's formula gives
$\|r_\lambda(v)\|\le C\|v\|^2$; the second and fourth moments in
\eqref{eq:moments} then imply \eqref{eq:lambda-remainder-bounds}.

Using $\E\|W\|^2=\E\|W-\E W\|^2+\|\E W\|^2$, we obtain
\begin{align}
 \E\norm{u}^2&=O(\eta),
 &\E\norm{v}^2&=O(\varepsilon+\eta^2).
 \label{eq:sup:finite-raw-floors}
\end{align}
The same fast bound holds for $\widehat u$. The slow $\eta^2$ term
comes from squared bias and may dominate the variance when
$\varepsilon\ll\eta^2$.

\subsection{Initialization and Blockwise Mean-Square Bounds}

Let $z_k^{\mathrm{st}}$ be a stationary version coupled to a trajectory
from a prescribed state $z_0$. Assumption~\ref{as:coupling} requires
\begin{equation}
 \E\norm{z_k-z_k^{\mathrm{st}}}
 \le C\{1+V(z_0)\}
       \left(e^{-c\eta k}+e^{-c\varepsilon k}\right),
 \label{eq:sup:finite-two-clock-coupling}
\end{equation}
where $V\ge0$ and
$\sup_{\eta,\rho}\E_{\pi_{\eta,\rho}}V(z)<\infty$.
Expectations include the stationary starting state and any coupling
of differently initialized Markov inputs. Theorem~\ref{thm:tail}
gives the corresponding tail-average initialization bound.

For blockwise mean-square control, assume that the chosen extension of
$\lambda$ is defined at the slow states of both coupled trajectories.
Require, in addition, the following exact-coordinate $L^2$ bounds,
with $F_k=e^{-c\eta k}$ and $S_k=e^{-c\varepsilon k}$:
\begin{align*}
 \|\widehat u_k-\widehat u_k^{\mathrm{st}}\|_{L^2}
 &\le C\{(F_k+\rho S_k)A_u+\rho(F_k+S_k)A_v\},\\
 \|v_k-v_k^{\mathrm{st}}\|_{L^2}
 &\le C\{\rho(F_k+S_k)A_u+(S_k+\rho F_k)A_v\},\\
 A_u^2+A_v^2&\le C\{1+V_2(z_0)\}.
\end{align*}
Here $A_u,A_v\ge0$ bound the corresponding $L^2$ initial
discrepancies, $V_2\ge0$, and
$\sup_{\eta,\rho}\E_{\pi_{\eta,\rho}}V_2(z)<\infty$.
These bounds must include any contribution from different initial input
states; they do not follow from the $L^1$ condition
\eqref{eq:sup:finite-two-clock-coupling}.
Appendix~\ref{app:example-verification} verifies them for the nonlinear
Markov example.

Squaring the block bounds, using $(a+b)^2\le2a^2+2b^2$, and
adjusting the positive decay constant give
\begin{align*}
 \E\|\widehat u_k-\widehat u_k^{\mathrm{st}}\|^2
 &\le C\{1+V_2(z_0)\}
          (e^{-c\eta k}+\rho^2e^{-c\varepsilon k}),\\
 \E\|v_k-v_k^{\mathrm{st}}\|^2
 &\le C\{1+V_2(z_0)\}
          (\rho^2e^{-c\eta k}+e^{-c\varepsilon k}).
\end{align*}
Combining these estimates with \eqref{eq:sup:finite-raw-floors}
and $\|a+b\|^2\le2\|a\|^2+2\|b\|^2$ yields
\begin{align}
 \E\norm{\widehat u_k}^2
 &\le C\eta+C\{1+V_2(z_0)\}
   \left(e^{-c\eta k}+\rho^2e^{-c\varepsilon k}\right),
 \label{eq:sup:finite-mse-fast}\\
 \E\norm{v_k}^2
 &\le C(\varepsilon+\eta^2)+C\{1+V_2(z_0)\}
   \left(\rho^2e^{-c\eta k}+e^{-c\varepsilon k}\right).
 \label{eq:sup:finite-mse-slow}
\end{align}

\section{Application details and additional simulations}
\label{app:experiments}

\subsection{Application: Linear Temporal-Difference Learning with Correction}
\label{subsec:tdc-application}

We use temporal-difference learning with correction (TDC) to approximate
the discounted value function of a fixed Markov reward process. The
state space is $\{0,1,2,3,4\}$. From an interior state, the process
stays put with probability $1/2$ and moves to either neighbor with
probability $1/4$; at a boundary, the missing outward move is replaced
by a self-loop. The invariant distribution is uniform. Set
\[
 \gamma=0.2,\qquad \phi(s)=1+0.05s,\qquad
 R_{k+1}=\mathbf 1\{S_{k+1}=4\},\qquad
 V_\theta(s)=\phi(s)\theta.
\]
For the temporal-difference error
\[
 \delta_k=R_{k+1}+\gamma\phi(S_{k+1})\theta_k
                         -\phi(S_k)\theta_k,
\]
the fast auxiliary variable $w$ and the slow value parameter $\theta$
follow
\begin{equation}
 \begin{aligned}
 w_{k+1}&=w_k+\eta\,[\delta_k-\phi(S_k)w_k]\phi(S_k),\\
 \theta_{k+1}&=\theta_k+\varepsilon\,[\delta_k\phi(S_k)
                  -\gamma\phi(S_{k+1})\phi(S_k)w_k],
 \qquad \varepsilon=\eta\rho.
 \end{aligned}
 \label{eq:tdc-recursion}
\end{equation}
Both updates use the old parameters $(w_k,\theta_k)$ and the same
observed transition $Y_k=(S_k,S_{k+1})$.

\paragraph{Averaged Objective and Equilibrium.}
For a transition $y=(s,s')$, define
\[
 C_y=\phi(s)^2,\quad
 A_y=\phi(s)\{\phi(s)-\gamma\phi(s')\},\quad
 D_y=\gamma\phi(s)\phi(s'),\quad
 b_y=\phi(s)\mathbf 1\{s'=4\}.
\]
Thus $A_y=C_y-D_y$ and, in the coordinate order $z=(w,\theta)$,
\begin{equation}
 G(y,z)=J_yz+b_y\binom{1}{1},\qquad
 J_y=\begin{pmatrix}-C_y&-A_y\\-D_y&-A_y\end{pmatrix}.
 \label{eq:tdc-affine-field}
\end{equation}
Let $\nu(s,s')=P(s,s')/5$ on the 13 admissible transitions, and use
bars for $\nu$-averages. Direct summation gives
\[
 \bar C=1.215,\qquad \bar A=0.9721,\qquad \bar b=0.2375.
\]
The fast equilibrium is
$\lambda(\theta)=(\bar b-\bar A\theta)/\bar C$.
Substituting it into the averaged slow update gives
\[
 g(\theta)=\frac{\bar A}{\bar C}(\bar b-\bar A\theta)
           =-\mathcal L'(\theta),\qquad
 \mathcal L(\theta)=\frac{(\bar b-\bar A\theta)^2}{2\bar C}.
\]
Here $\mathcal L$ is one half of the mean-squared projected Bellman
error for this scalar feature. Its unique minimizer and the joint
equilibrium are
\begin{equation}
 \theta^\star=\frac{\bar b}{\bar A}
       \approx0.2443164284,\qquad z^\star=(0,\theta^\star).
 \label{eq:tdc-equilibrium}
\end{equation}
The fast and reduced slow Jacobians are $-\bar C<0$ and
$S=-\bar A^2/\bar C<0$.

\paragraph{Verification of the Assumptions.}
The transition-pair driver is exogenous, with kernel
\[
 K((i,j),(u,v))=\mathbf 1\{u=j\}P(j,v).
\]
Although this kernel has zeros, it satisfies a finite-step Doeblin
condition. Indeed, $P^4(j,u)\ge1/256$ for every pair of states, so
\[
 K^5((i,j),(u,v))=P^4(j,u)P(u,v)
          \ge\frac{5}{256}\nu(u,v).
\]
The updates are affine with globally bounded derivatives. Since
$1\le\phi(s)\le1.2$, the global block-contraction constants in
Proposition~\ref{prop:contractive-localization} can be chosen as
\[
 a=1,\qquad b=1.2,\qquad c=0.288,\qquad d=0.76,
 \qquad bc=0.3456<ad=0.76.
\]
These follow from $C_y\ge1$, $0<A_y\le1.2$, $D_y\le0.288$,
and $A_y\ge0.76$. Together with the global affine fast equilibrium
and the two negative averaged Jacobians,
Proposition~\ref{prop:contractive-localization}
verifies Assumptions~\ref{as:smooth}--\ref{as:taylor}, including
existence, uniqueness, and uniform fourth-moment localization of the
stationary law for sufficiently small steps.

For the IID control, transitions are drawn independently from the
same distribution $\nu$. This preserves the averaged objective and
equilibrium, while removing serial dependence; these IID pairs are
not required to form a contiguous state trajectory.

\paragraph{Theoretical Bias Coefficients.}
The Poisson equation for the finite-state kernel $K$ determines the
centered corrector $\mathcal U(y,z)$. Because $G$ is affine,
$D^2\bar G=0$, and the leading bias contains only the direct Markov
response:
\[
 r_\rho=\sum_{y,y'}\nu(y)K(y,y')
       D_z\mathcal U(y',z^\star)D_\rho G(y,z^\star),\qquad
 b(\rho)=-J^{-1}r_\rho=b_0+\rho b_1.
\]
For this instance,
$b_0^\theta\approx0.00624921$ and
$b_1^\theta\approx-0.40166550$.

Consequently $b_2=0$ for this application. This example tests the
direct contribution of Markov dependence to stationary bias.
The nonlinear curvature contribution is illustrated by the exact
model in Appendix~\ref{sec:sup:exact-markov-example}.
In the IID control,
stationarity and independence give $\bar G(\mathbb E z)=0$;
therefore the stationary mean equals $z^\star$ exactly throughout
the verified small-step regime.

\paragraph{Optional Analytic Characterization of Stationary Bias.}
The affine structure also permits a separate analytic
characterization of the stationary mean.
This characterization is not needed to compute the empirical
bias in~\eqref{eq:tdc-empirical-bias}.

Write $z_k=(w_k,\theta_k)^\top$ and define
\[
 \xi_y=G(y,z^\star),\qquad
 M_y=I_2+\eta D_\rho J_y,\qquad
 D_\rho=\operatorname{diag}(1,\rho).
\]
The error recursion is
\[
 z_{k+1}-z^\star
 =M_{Y_k}(z_k-z^\star)+\eta D_\rho\xi_{Y_k}.
\]
For each of the 13 admissible transitions $y$, let
\[
 h_y=\mathbb E_{\pi_{\eta,\rho}}
     \bigl[(z_k-z^\star)\mathbf 1\{Y_k=y\}\bigr].
\]
Stationarity and the transition-pair kernel $K$ give
\begin{equation}
 h_{y'}=\sum_y K(y,y')
    \left\{M_yh_y+\eta\nu(y)D_\rho\xi_y\right\}.
 \label{eq:tdc-exact-mean-system}
\end{equation}
These are 26 scalar linear equations.
Their solution determines the stationary bias through
\[
 B^\theta_{\eta,\rho}
 =\begin{pmatrix}0&1\end{pmatrix}\sum_y h_y.
\]
This identity retains the full dependence on the step sizes.
It provides an additional theoretical characterization for
the affine example; the empirical estimates are obtained
from simulated trajectories.

\subsection{Simulation Protocol}
\label{subsec:mc-protocol}

We simulate the TDC model in~\eqref{eq:tdc-recursion} and the
nonlinear model in Appendix~\ref{sec:sup:exact-markov-example}.
All data are synthetic.
The TDC target is $\theta^\star$ in~\eqref{eq:tdc-equilibrium};
the nonlinear model has target $Z^\star=0$.
These targets are specified by the averaged problems.

Each setting uses $R=256$ independent replications.
Within a replication, extrapolation levels share the same
input sequence.
For a given method, let $\widehat\vartheta^{(r)}$ denote its
estimate in replication $r$: either a post-burn time average
or a weighted combination of such averages.
With target $\vartheta^\star$, define
\[
 \overline{\widehat\vartheta}
 =\frac1R\sum_{r=1}^{R}\widehat\vartheta^{(r)}.
\]
We compute the empirical bias, sample standard deviation,
and RMSE as
\begin{align*}
 \widehat b
   &=\overline{\widehat\vartheta}-\vartheta^\star,\\
 \widehat s^2
   &=\frac1{R-1}\sum_{r=1}^{R}
     \bigl(\widehat\vartheta^{(r)}
           -\overline{\widehat\vartheta}\bigr)^2,\\
 \widehat{\mathrm{RMSE}}
   &=\left\{\frac1R\sum_{r=1}^{R}
     \bigl(\widehat\vartheta^{(r)}
           -\vartheta^\star\bigr)^2\right\}^{1/2}.
\end{align*}
All three quantities are computed from the simulated
replication-level estimates.

For TDC, an approximate pointwise 95\% confidence interval
for the bias of the finite-run estimator is
\[
 \widehat b
 \ \pm\ t_{R-1,0.975}\frac{\widehat s}{\sqrt R},
\]
where $t_{R-1,0.975}$ is the corresponding Student-$t$
quantile.
For normalized bias, both the estimate and the interval
endpoints are divided by $\eta$.

Standard-deviation and RMSE intervals, and the nonlinear
model's bias intervals, use the percentile bootstrap with
2,000 resamples of entire replications.
RMSE ratios use paired resampling of methods.
All intervals are pointwise 95\% intervals.
Dependent iterations within a trajectory are not treated
as independent observations.

For a finite burn-in, empirical bias includes any remaining
initialization effect.
The confidence intervals quantify Monte Carlo uncertainty;
they do not account for this systematic transient error.
Stationary-bias interpretation therefore also requires
adequate burn-in.

\paragraph{Extrapolation Rules and Update Budgets.}
The base fast step is $h$, and level $j$ uses
$\eta_j=h2^{-j}$ and $\varepsilon_j=\eta_j^p$, with $p=6/5$.
The methods are ordinary averaging and extrapolation targeting
$\mathcal G=\{1\}$, $\{1,2\}$, or $\{1,p\}$.
For $m=1+|\mathcal G|$ levels, the weights solve
\[
 \sum_{j=0}^{m-1}w_j=1,\qquad
 \sum_{j=0}^{m-1}w_j2^{-j\gamma}=0
 \quad(\gamma\in\mathcal G).
\]
We apply the weights to post-burn tail averages, rather than to
individual terminal iterates. A recursion update advances one
complete state at one level. With total budget $B$, each of $m$
levels receives $N=\lfloor B/m\rfloor$ updates, including its
burn-in; at most $m-1$ updates are unused. Shared random inputs do
not reduce this accounting. Equal update counts are not a claim
of equal wall-clock cost.

\paragraph{Exact-model Runs.}
We use $q\in\{0,0.5\}$ and
$h\in\{0.1,0.05,0.025,0.0125\}$, starting the iterates at zero
and the sign chain in its uniform invariant distribution.
Level $j$ discards $K_j=\lceil20/\varepsilon_j\rceil$ updates.
The fixed-ratio study sets $\rho=0.25$ and uses 240,000 updates
per replication. The path-bias study uses 2,400,000 updates per
level, whereas the separate equal-budget comparison fixes
$B=240,000$ total updates per method. The longer path-bias runs
were chosen after a pilot showed that the smallest corrected bias
was close to Monte Carlo uncertainty; the step sizes, coefficients,
and seeds were unchanged. These long runs are not used as the
equal-budget comparison.

\paragraph{TDC Runs.}
We initialize $w_0=\theta_0=0$ and draw $S_0$ uniformly.
The fixed-ratio study uses $\rho\in\{0.1,0.25\}$ and
$\eta\in\{0.1,0.05,0.025,0.0125\}$, with 256 independent
replications for each Markov and IID transition-pair setting.
Each replication retains $2^{20}$ updates after discarding
$K=\lceil80/\varepsilon\rceil$ updates.
The IID control draws independent pairs from $\nu$; it does not
replace the Markov trajectory by overlapping pairs of IID states.

The TDC budget study uses $\eta_j=h2^{-j}$ and
$\varepsilon_j=\eta_j^p$, with $p=6/5$, $h=0.05$, and total
update budgets $B\in\{2^{17},2^{19},2^{21}\}$. We compare
ordinary averaging with the extrapolation rules targeting
$\mathcal G=\{1\}$, $\{1,2\}$, and $\{1,p\}$ defined above.
For $m$ levels, every level runs for $N=\lfloor B/m\rfloor$
updates and discards the common burn-in
$K=\lceil80/\min_j\varepsilon_j\rceil$.
All levels share a Markov trajectory and average the same
post-burn iteration indices. The total cost is $mN$, including
all $mK$ discarded updates.
The error target is $\theta^\star$ in~\eqref{eq:tdc-equilibrium}.
Bias intervals use the replication-level Student-$t$ construction;
RMSE intervals and ratios use 2,000 bootstrap resamples, paired
across methods for ratios.

As a simulation-based burn-in diagnostic, we compare
$K=64{,}000$ and $K=128{,}000$ at
$\rho=0.1$ and $\eta=0.0125$, keeping the retained length
and replication seeds fixed.
We compute the change in the estimated mean and a paired
confidence interval across replications.
Changing the burn-in also changes the retained stochastic
window, so this comparison is a sensitivity diagnostic
rather than an exact measurement of initialization error.

\subsection{Results}
\label{subsec:empirical-results}

All reported bias estimates are computed from simulated
post-burn averages across 256 independent replications,
subtracting the corresponding equilibrium target.
The simulations do not use a stationary linear-system
solution or an exact stationary-bias formula.
The protocol and uncertainty calculations are described in
Appendix~\ref{subsec:mc-protocol}.

\paragraph{Temporal-Difference Learning.}
Figure~\ref{fig:tdc-diagnostics} reports the empirical
normalized bias under Markov and independent-pair sampling,
together with the bias and standard deviation of the
equal-budget estimators.
Figure~\ref{fig:tdc-rmse} reports the corresponding RMSE
relative to $\theta^\star$.
This comparison accounts jointly for residual bias and
sampling variability when the total number of updates,
including burn-in across all extrapolation levels, is fixed.

\begin{figure}[t]
 \centering
 \includegraphics[width=.72\linewidth]{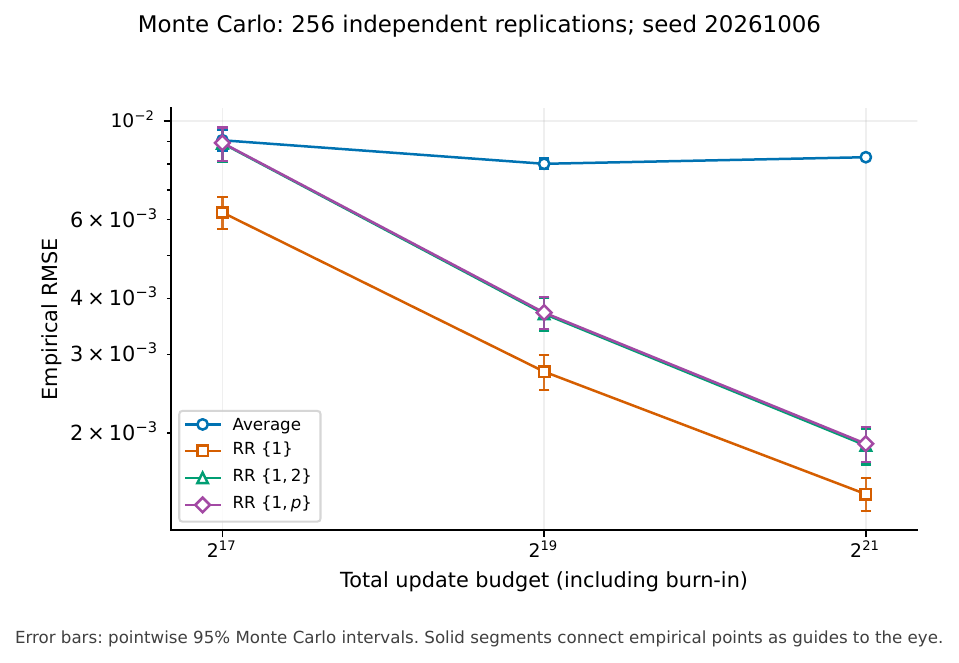}
 \caption{TDC estimation error at equal total update budgets.
 Points show empirical RMSE relative to $\theta^\star$ for
 ordinary averaging and three extrapolation rules.
 The step sizes are $\eta_j=h2^{-j}$ and
 $\varepsilon_j=\eta_j^p$, with $h=0.05$ and $p=6/5$.
 Each budget includes burn-in and all updates across levels.
 Error bars are pointwise 95\% bootstrap intervals based on
 256 independent replications and 2,000 bootstrap resamples.
 Solid segments connect empirical points.
 RR labels specify the cancelled powers.}
 \label{fig:tdc-rmse}
\end{figure}

\paragraph{Nonlinear Toy Example.}
We simulate the model of
Appendix~\ref{sec:sup:exact-markov-example}, whose equilibrium
target is $Z^\star=0$.
Figure~\ref{fig:exact-mc-bias} compares empirical normalized
bias with the theoretical leading coefficient at a fixed
step-size ratio.
Figure~\ref{fig:exact-mc-rr} separates the long-run bias study
from the comparison of RMSE at equal total update budgets.
The latter comparison includes the sampling variability
introduced by combining estimates from several levels.

\begin{figure}[t]
 \centering
 \includegraphics[width=.96\linewidth]{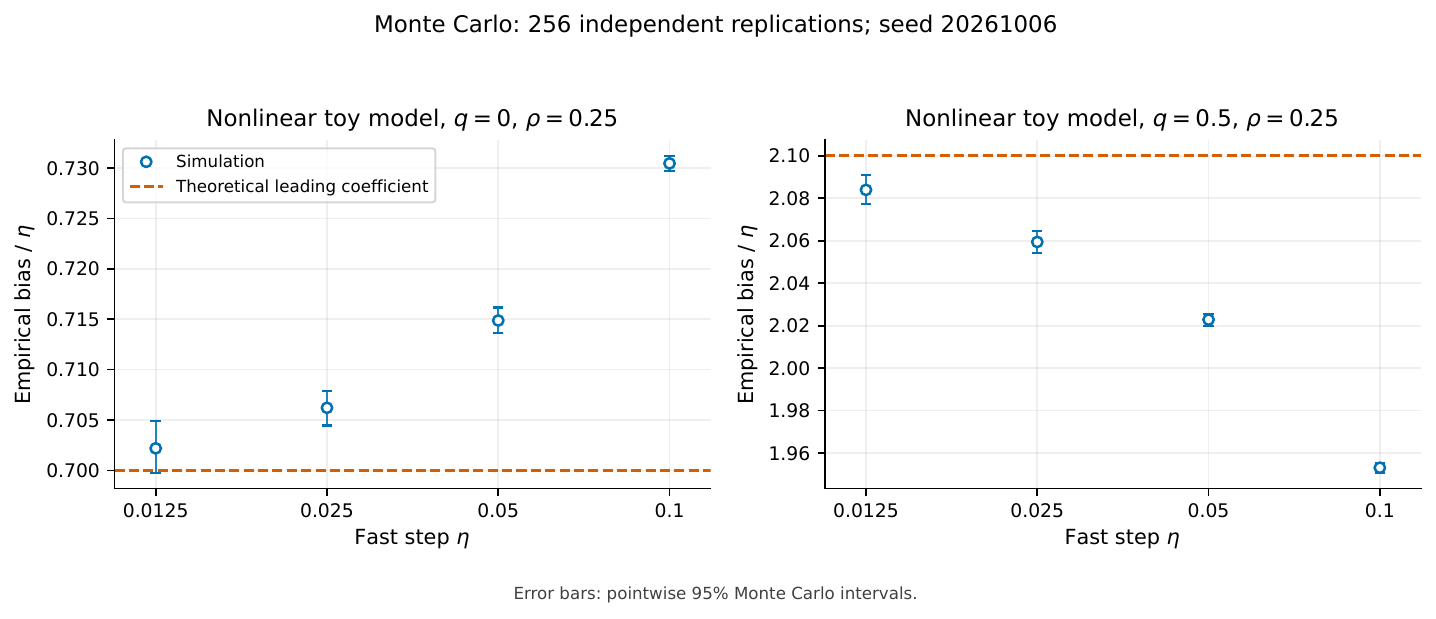}
 \caption{Empirical bias of the nonlinear toy model at
 $\rho=\varepsilon/\eta=0.25$, for $q=0$ and $q=0.5$.
 Points show estimated bias relative to $Z^\star=0$,
 divided by $\eta$.
 Dashed lines show the theoretical leading coefficient
 $c_q\bigl(1/2+\rho/(1+\rho)\bigr)$, where
 $c_q=(1+q)/(1-q)$.
 Error bars are pointwise 95\% bootstrap intervals based on
 256 independent replications and 2,000 bootstrap resamples.
 Each replication uses 240,000 updates, including burn-in.}
 \label{fig:exact-mc-bias}
\end{figure}

\begin{figure}[p]
 \centering
 \includegraphics[width=.90\linewidth]{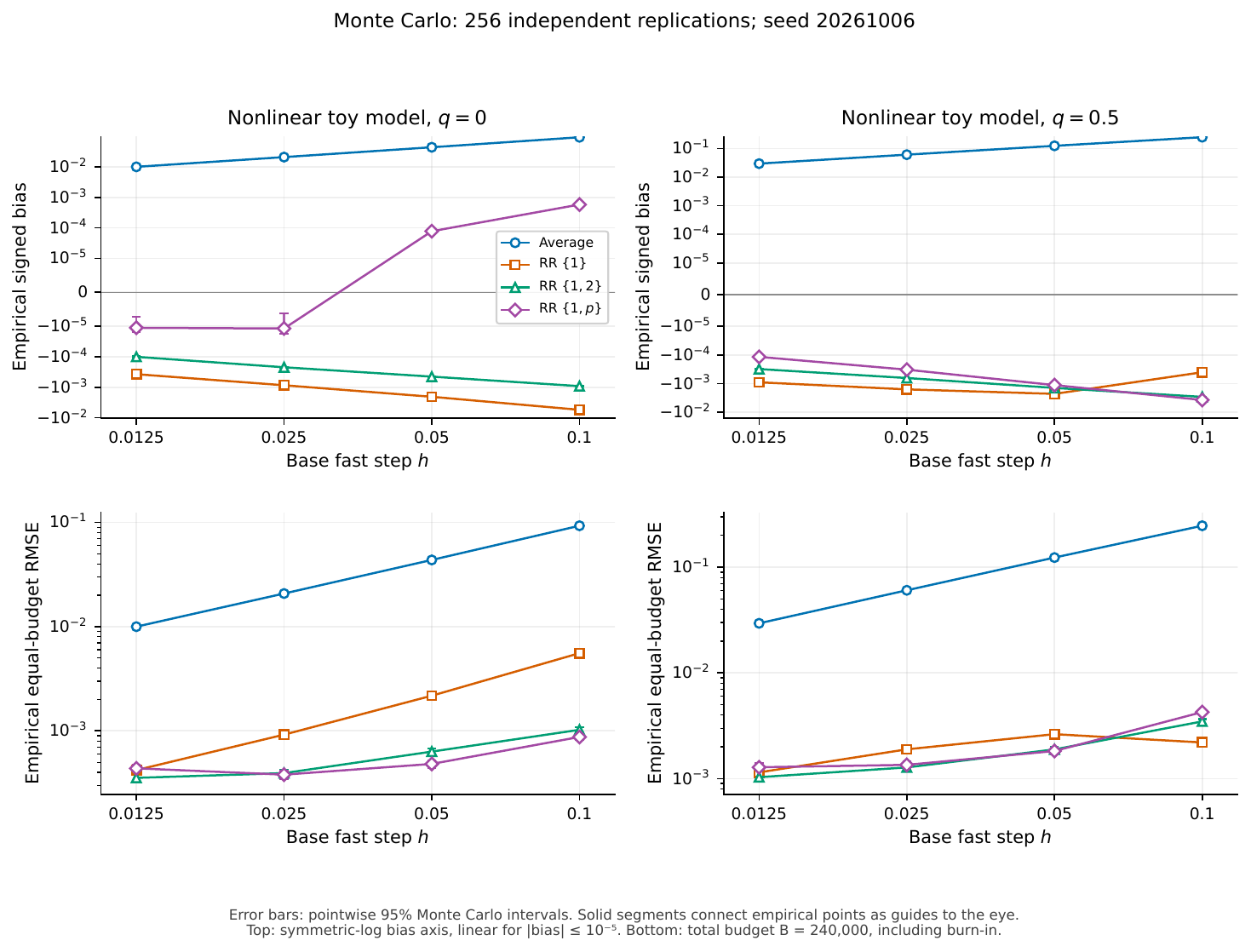}
 \caption{Empirical extrapolation results for the nonlinear
 toy model along $\varepsilon=\eta^{6/5}$.
 Columns correspond to $q=0$ and $q=0.5$.
 Top: estimated signed bias from runs using 2,400,000 updates
 per level, including burn-in. The vertical scale is
 logarithmic away from zero and linear for
 $|\mathrm{bias}|\le 10^{-5}$.
 Bottom: empirical RMSE at a total budget of 240,000 updates
 per method, including burn-in across all levels.
 The top and bottom rows therefore represent distinct
 experiments with different averaging lengths.
 All error bars are pointwise 95\% bootstrap intervals based
 on 256 independent replications and 2,000 bootstrap resamples.
 Solid segments connect empirical points.
 RR labels specify the cancelled powers, with $p=6/5$.}
 \label{fig:exact-mc-rr}
\end{figure}

\section{Additional related work}
\label{sec:sup:literature}
Classical two-timescale stochastic approximation with decreasing steps uses
singular ODEs, averaging, and weak convergence
\citep{borkar1997twotime,mokkadem2006convergence}.  Finite-sample and limit
results cover linear Markov recursions, nonlinear convergence rates, lock-in,
and controlled-Markov central limit theory
\citep{dalal2018finite,kaledin2020finite,doan2021nonlinear,hu2024clt}.  Their
primary target is convergence as the steps vanish over time, rather than the
stationary mean displacement produced by two fixed steps. \citet{kaledin2020finite} give finite-time mean-square bounds for
linear Markov recursions with constant, piecewise-constant, and
diminishing step sizes. Our distinction is the nonlinear stationary
mean-bias coefficient expansion, rather than constant-step analysis alone.

The closest constant-step comparison is the linear TTSA analysis of
\citet{kwon2025constant}, which treats stationarity, finite-time behavior,
mean and covariance, and joint extrapolation.  Linear tracking error can
already contain a scale such as $\varepsilon^2/\eta$.  We address the
nonlinear question of how stationary covariance moves the invariant mean.  Expanding the covariance response in
$\rho$ produces the $\varepsilon^2/\eta$ correction, and curvature converts it
into mean bias.
Fast--slow cross-covariance is one source, but not the only possible source.

Linear recursions with centered additive noise often have invariant mean
exactly at the ODE equilibrium.   A quadratic readout provides a simple mechanism for converting singular covariance into visible mean bias.

\end{document}

%% file: macros.tex
\newcommand{\R}{\mathbb{R}}
\newcommand{\E}{\mathbb{E}}
\newcommand{\Pp}{\mathbb{P}}

\newcommand{\diag}{\operatorname{diag}}
\newcommand{\Tr}{\operatorname{Tr}}
\newcommand{\Cov}{\operatorname{Cov}}

\newcommand{\norm}[1]{\left\lVert#1\right\rVert}